\documentclass[twoside]{article}

\usepackage[preprint]{aistats2027}

\usepackage[round]{natbib}
\usepackage{times}
\usepackage[T1]{fontenc}
\usepackage[utf8]{inputenc}
\usepackage{microtype}
\usepackage{amsmath,amssymb,amsfonts,mathtools,bm,amsthm}
\usepackage{booktabs,multirow,makecell}
\usepackage{algorithm}
\usepackage{algpseudocode}
\usepackage{enumitem}
\usepackage[table]{xcolor}
\usepackage{pifont}
\usepackage[colorlinks = True, linkcolor = orange, citecolor = violet!70!white]{hyperref}
\usepackage{url}
\usepackage{graphicx}
\usepackage{tikz}
\usetikzlibrary{positioning,fit,backgrounds,calc,arrows.meta}

\usepackage{tabularx}
\usepackage{array}
\usepackage{longtable}
\usepackage{etoc}
\usepackage{wrapfig}
\definecolor{theoryhead}{RGB}{229,236,246}
\definecolor{theorybandA}{RGB}{240,247,255}
\definecolor{theorybandB}{RGB}{243,249,240}
\definecolor{theorybandC}{RGB}{248,242,252}
\definecolor{theoryref}{RGB}{98,72,168}

\definecolor{mygreen}{RGB}{34,139,34}
\definecolor{myred}{RGB}{180,35,24}

\newcommand{\cmark}{\textcolor{mygreen}{\ding{51}}}
\newcommand{\comark}{\textcolor{orange}{\ding{51}}}
\newcommand{\xmark}{\textcolor{myred}{\ding{55}}}

\DeclareMathOperator{\diag}{diag}
\DeclareMathOperator{\Var}{Var}
\DeclareMathOperator{\Cov}{Cov}
\DeclareMathOperator{\KL}{KL}

\DeclareMathOperator{\tr}{tr}
\DeclareMathOperator{\blkdiag}{blkdiag}

\newcommand{\R}{\mathbb{R}}
\newcommand{\E}{\mathbb{E}}
\newcommand{\Normal}{\mathcal{N}}
\newcommand{\GP}{\mathcal{GP}}
\newcommand{\Id}{\boldsymbol{I}}
\newcommand{\vx}{\bm{x}}
\newcommand{\vy}{\bm{y}}
\newcommand{\vz}{\bm{z}}
\newcommand{\vf}{\bm{f}}
\newcommand{\vg}{\bm{g}}
\newcommand{\vw}{\bm{w}}
\newcommand{\vu}{\bm{u}}
\newcommand{\vs}{\bm{s}}
\newcommand{\vb}{\bm{b}}
\newcommand{\vm}{\bm{m}}
\newcommand{\vS}{\bm{S}}

\newcommand{\vc}{\bm{c}}
\newcommand{\vH}{\bm{H}}
\newcommand{\vA}{\bm{A}}

\newcommand{\vQ}{\bm{Q}}
\newcommand{\vLambda}{\bm{\Lambda}}
\newcommand{\vSigma}{\bm{\Sigma}}
\newcommand{\veta}{\bm{\eta}}
\newcommand{\Hx}{\boldsymbol{H}_x}
\newcommand{\Hxdot}{\boldsymbol{H}_{\dot x}}
\newcommand{\method}{\textsc{MEGPODE}}
\newcommand{\methodm}{\textsc{MEGPODE-M}}
 \newcommand{\independent}{\perp\!\!\!\!\perp}

\newtheorem{proposition}{Proposition}

\setlist[itemize]{leftmargin=*,topsep=2pt,itemsep=2pt,parsep=0pt}
\setlist[enumerate]{leftmargin=*,topsep=2pt,itemsep=2pt,parsep=0pt}

\begin{document}
\runningauthor{Martinelli et al.}

\twocolumn[
\aistatstitle{Bayesian Nonparametric Mixed-Effect ODEs with Gaussian Processes}
\vspace{-18pt}
\aistatsauthor{%
\begin{tabular}[t]{c}
\begin{tabular}{ccc}
\bfseries Julien Martinelli\textsuperscript{*} &
\bfseries Maksim Sinelnikov &
\bfseries Harri Lähdesmäki
\end{tabular}
\\[3pt]
\normalfont Aalto University
\\[1.4ex]
\begin{tabular}{cc}
\bfseries Quentin Clairon\textsuperscript{$\dagger$} &
\bfseries Mélanie Prague\textsuperscript{$\dagger$}
\end{tabular}
\\[3pt]
\normalfont Univ. Bordeaux, INSERM BPH, U1219, Inria SISTM team, VRI, France
\end{tabular}%
}
\vspace{12pt}
]
\begingroup
\renewcommand{\thefootnote}{\fnsymbol{footnote}}
\footnotetext[1]{Work initiated while at the Inria SISTM team.}
\footnotetext[2]{Equal contribution in supervision.}
\endgroup

\begin{abstract}
Dynamical modelling in pharmacometrics, systems biology, and epidemiology often involves subjects with related but distinct dynamics, observed through temporally sparse and irregular measurements. Classical nonlinear mixed-effects ODE models represent this heterogeneity but rely on parametric vector fields, making them vulnerable to structural misspecification. Existing nonparametric ODE methods typically assume shared dynamics rather than an explicit mixed-effect hierarchy. We propose \method{}, a Bayesian nonparametric mixed-effect ODE model that decomposes each subject's vector field into a shared population component and a subject-specific deviation, both endowed with Gaussian process priors. To avoid repeated ODE solves during training, we represent latent trajectories with state-space GP priors and penalize discrepancies between their derivatives and the vector fields on a finite time grid. A single variational objective enables approximate joint inference over trajectories and vector fields, accounting for state uncertainty and dependence between population and subject-specific dynamics. Across heterogeneous oscillatory and biomedical ODE benchmarks, \method{} improves population-field recovery and subject-level trajectory prediction relative to strong baselines.
\end{abstract}

\section{Introduction}
\label{sec:intro}

Ordinary differential equation (ODE) models are a standard language for mechanistic modeling in systems biology, medicine, vaccinology, and pharmacometrics because they encode interpretable dynamics and support extrapolation beyond the observed window~\citep{machado2011modeling}.
In longitudinal population studies, however, the scientific object is rarely a single trajectory. Patients and cells have related but distinct dynamics, motivating models that pool information across subjects while retaining individual differences~\citep{clairon2023antibody, alexandre2023prediction}.

The classical framework for this setting is nonlinear mixed-effects (NLME) ODE modelling, where a shared mechanistic system is coupled with subject-specific random effects~\citep{lavielle2014mixedeffects}.
This is appealing because heterogeneity is represented at the level of the \emph{dynamics} rather than only at the \emph{trajectory} level. %
But it is also restrictive: one must commit to a parametric vector field and to a specific low-dimensional random-effect parameterization.
This can be problematic in the very regimes where these models are most often used, for instance in biology, medicine, and pharmacometrics~\citep{martinelli2025position}. There, mechanisms are often only partially known, individual differences may not be well captured by a few parameter perturbations, and observations are typically noisy, sparse, and irregularly sampled.
Parametric NLME-ODEs can thus be misspecified and difficult to fit~\citep{li2026vae}.

These limitations point toward nonparametric dynamics modelling. Flexible continuous-time models based on Gaussian processes (GPs) or neural ODEs can learn unknown dynamics without committing to a detailed mechanistic family \citep{heinonen2018npode,yang2021magi,hegde2022vmsgp,long2022autoip,hamelijnck2024physsgp,chen2018neuralode,kidger2020neuralcde}. But most of this literature is built for a single shared dynamical system rather than a heterogeneous population. Conversely, recent population-level continuous-time models do capture subject variability, yet typically through finite-dimensional parametric or neural effects rather than a Bayesian nonparametric hierarchy over subject-specific vector fields \citep{nazarovs2022menode,arruda2024amortized,bram2025monolix,li2026vae,ojedamarin2025aicmet}. This leaves a clear methodological gap: we still lack a nonparametric mixed-effect model for continuous-time dynamics that operates directly at the ODE level and provides uncertainty quantification at both the population and subject levels.

\textbf{Contributions.}
We introduce \method{}, a Bayesian nonparametric mixed-effect ODE model with a shared population field and regularized subject-specific deviations. Sharing information across subjects supports inference from sparse measurements and adaptation to new subjects from a short observation prefix. Figure~\ref{fig:intro-gist-and-positioning} illustrates this decomposition and compares \method{} with related model classes. Our contributions are:
\begin{itemize}
    \item \textbf{Conceptual:} we formulate a GP hierarchy over shared and subject-specific ODE vector fields in the modeled state coordinates, extending classical parameter-based mixed effects to random functions.
    \item \textbf{Methodological:} we develop a unified variational criterion combining state-space trajectory priors and soft ODE constraints, with exact joint Gaussian field updates and trajectory updates that account for field uncertainty.
    \item \textbf{Empirical:} we evaluate MEGPODE on heterogeneous ODE benchmarks, assessing field recovery, forecasting, and uncertainty quantification. We also introduce a semi-mechanistic variant and compare it with an established nonlinear mixed-effects approach under structural misspecification.
\end{itemize}

\begin{figure*}[!t]
\centering
    \vspace*{-5mm}
    \includegraphics[width=\linewidth,trim=0 0 0 9bp,clip]{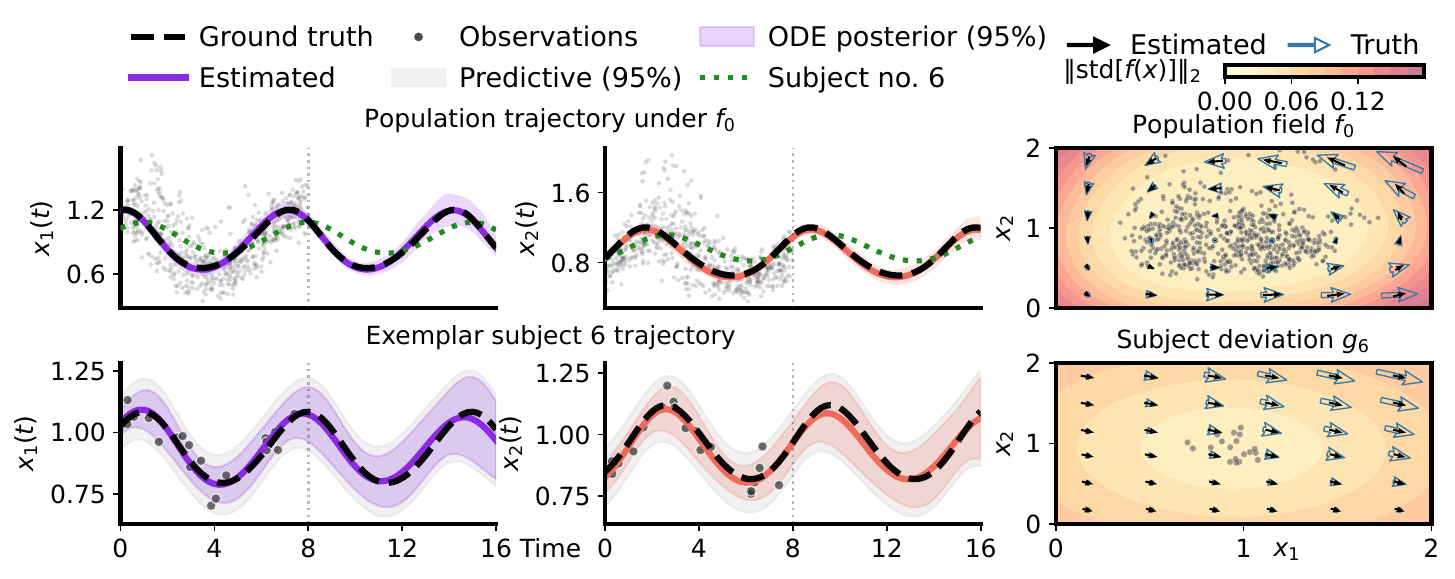}
    
    \vspace{-.10cm}
    \centering
    \scriptsize
    \renewcommand{\arraystretch}{.75}
    \setlength{\aboverulesep}{.2ex}
    \setlength{\belowrulesep}{.3ex}
    \setlength{\tabcolsep}{4pt}
\resizebox{\linewidth}{!}{%
\begin{tabular}{lcccc}
\toprule
& \makecell{Personalized}
& \makecell{Flexible ODE vector fields}
& \makecell{Mixed-effect}
& \makecell{Probabilistic} \\
\midrule
Parametric NLME-ODEs \citep{wang2014mixedode}
    & \cmark & \xmark & \cmark & \cmark \\
Mixed-Effect GPs \citep{leroy2022magma,kimme,yoon2022doubly}
    & \cmark & \xmark & \cmark & \cmark \\
GP-ODEs~\citep{heinonen2018npode,long2022autoip,hamelijnck2024physsgp}
    & \xmark & \cmark & \xmark & \cmark \\
Mixed-effects Neural ODEs \citep{nazarovs2022menode}
    & \cmark & \comark~\textcolor{orange}{(\textbf{latent})} & \comark~\textcolor{orange}{(\textbf{latent weights})} & \cmark \\
Contextual neural dynamics~\citep{kirchmeyer2022generalizing,mouli2024metaphysica,iflgui}
    & \cmark & \cmark & \xmark & \xmark \\
\method{} (This work)
    & \cmark & \cmark & \cmark & \cmark \\
\bottomrule
\end{tabular}
}
\setlength{\abovecaptionskip}{-3pt}
\setlength{\belowcaptionskip}{-4pt}
\caption{
\textbf{Population and subject dynamics in Lotka--Volterra.}
\method{} learns a population vector field \(f_0\) and subject-specific deviation fields \(g_i\), whose sum governs each subject's dynamics. Left: population (top) and subject (bottom) trajectories. Black shows ground truth; green shows the true trajectory of the subject fitted below. Colored bands show approximate 95\% posterior uncertainty, and gray bands include observation noise. Crucially, top-row dots pool heterogeneous subjects, but the bands concern the population process without subject-specific deviations.
Right: population and subject-specific deviation fields. Warm shading shows uncertainty.
The table compares model classes.
}
\label{fig:intro-gist-and-positioning}
\end{figure*}

\section{Background}
\label{sec:background}

\subsection{Problem setup}
\label{sec:problem}
We observe $M$ subjects. Subject $i$ is measured at times $\mathcal{T}_i=\{t_{i,1}<\dots<t_{i,N_i}\}$, with observations $\vy_{i,n}\in\R^{D_y}$. Let $\vx_i(t)\in\R^D$ denote the latent dynamical state. We assume a linear-Gaussian observation model
\begin{equation}
    \vy_{i,n}\mid\vx_i(t_{i,n}) \sim \Normal\!\bigl(\vH\vx_i(t_{i,n}),\boldsymbol{\Sigma}_y\bigr),
    \vH\in\R^{D_y\times D},
    \label{eq:obs}
\end{equation}
and a subject-specific ODE
\begin{equation}
    \dot{\vx}_i(t)=\vf_i\bigl(\vx_i(t)\bigr),
    \qquad \vx_i(t_{i,0})=\vx_{i,0}.
    \label{eq:ode}
\end{equation}
The learning problem is to recover a collection of related subject-level vector fields from irregular and partially observed longitudinal data. In the spirit of mixed-effects modeling, we seek a decomposition of subject dynamics into a shared population component and a subject-specific deviation. Section~\ref{sec:fieldprior} introduces this decomposition. We next review the GP and state-space ingredients that make this construction computationally tractable.

\subsection{GPs and state-space trajectory priors}
\label{sec:ssgpbackground}

\begingroup
\fussy\tolerance=3000
\spaceskip=2.2pt plus 1.1pt minus 0.8pt
A convenient probabilistic starting point for continuous-time trajectories is the GP \mbox{\citep{rasmussen2006gp}}:
\begin{equation}
f \sim \GP(m,k),
\label{eq:gp}
\end{equation}
\endgroup
means that any finite collection of function values is jointly Gaussian, with mean function $m$ and covariance kernel $k$. For differentiable kernels, derivatives are jointly Gaussian with function values as well, since differentiation is a linear operator. This makes GPs a natural tool when both a trajectory and its derivative matter.

\textbf{State-space priors.} For temporal modelling, an important subclass of GP priors admits an equivalent state-space representation. In particular, Mat\'ern kernels and integrated Wiener-process (IWP) priors can be written as linear stochastic differential equations whose discretization on subject $i$'s irregular grid $\tau_{i,0:K_i}$ yields a Gaussian Markov chain \citep{sarkka2019sde,zhang2024model}:
\begin{equation}
    \begin{aligned}
    \vs_{i,k+1} &= \vA(h_{i,k})\vs_{i,k}+\bm{q}_{i,k},\\
    \bm{q}_{i,k} &\sim \Normal\bigl(\mathbf{0},\vQ(h_{i,k})\bigr),\\
    h_{i,k} &= \tau_{i,k+1}-\tau_{i,k},
    \end{aligned}
    \label{eq:ssgp}
\end{equation}
for a suitable latent state $\vs_{i,k}$, where $\vA(h)$ is the state-transition matrix and $\vQ(h)$ is the process-noise covariance. Under linear-Gaussian observations, posterior inference then scales linearly in the grid size $K_i$ via Kalman filtering and smoothing \citep{sarkka2013bfs}.

\textbf{IWP trajectory prior.} To keep both the state and its derivative explicit, define the latent state at time $\tau_{i,k}$ as
\begin{equation}
    \vs_{i,k}\coloneqq
    \begin{bmatrix}
        \vx_i(\tau_{i,k})\\
        \dot{\vx}_i(\tau_{i,k})
    \end{bmatrix}
    \in\R^{2D},
    \quad
    \begin{aligned}
        \vx_i(\tau_{i,k})&=\boldsymbol{H}_x\vs_{i,k},\\
        \dot{\vx}_i(\tau_{i,k})&=\boldsymbol{H}_{\dot{x}}\vs_{i,k}.
    \end{aligned}
    \label{eq:augstate}
\end{equation}
with $\boldsymbol{H}_x=[\boldsymbol{I}_D\;\boldsymbol{0}_{D\times D}]\in\mathbb{R}^{D\times 2D}$ and $\boldsymbol{H}_{\dot{x}}=[\boldsymbol{0}_{D\times D}\;\boldsymbol{I}_D]\in\mathbb{R}^{D\times 2D}$.
For example, an integrated Wiener-process prior of order two yields $\vA(h)=\bigl[\begin{smallmatrix}1 & h\\ 0 & 1\end{smallmatrix}\bigr] \otimes \Id_D \in  \R^{2D\times 2D}$ and $\vQ(h)=\sigma_{\mathrm{IWP}}^2\bigl[\begin{smallmatrix}h^3/3 & h^2/2\\ h^2/2 & h\end{smallmatrix}\bigr] \otimes \Id_D \in \R^{2D\times 2D}$.
We write $\vx_{i,k}$ and $\dot{\vx}_{i,k}$ for the state and derivative at $\tau_{i,k}$. Stacking the augmented grid states gives the subject trajectory vector $\vs_i=(\vs_{i,0}^\top,\ldots,\vs_{i,K_i}^\top)^\top$. With mean $\bm m_{\mathrm{init}}$ and covariance $\boldsymbol P_{\mathrm{init}}$ for the initial augmented state $\vs_{i,0}$, the IWP trajectory density is
\begin{equation}
\begin{aligned}
&p_{T,i}(\vs_i)=\Normal(\vs_{i,0}\mid\bm m_{\mathrm{init}},\boldsymbol P_{\mathrm{init}})\\
&\quad\times\prod_{k=0}^{K_i-1}\Normal\bigl(\vs_{i,k+1}\mid\vA(h_{i,k})\vs_{i,k},\vQ(h_{i,k})\bigr),
\end{aligned}
\label{eq:iwp-path-prior}
\end{equation}

\subsection{Soft ODE constraints via virtual observations}
\label{sec:collocationbackground}
To connect the trajectory prior with the ODE, the dynamical relation is imposed as a soft constraint rather than through repeated numerical integration. Given a known vector field $\vf$, the state derivative agrees with the ODE at \emph{collocation times} $\tau_{i,k}$, up to a small collocation error:
\begin{equation}
    \mathbf{0} = \Hxdot\vs_{i,k} - \vf\bigl(\Hx\vs_{i,k}\bigr)
        + \boldsymbol{\epsilon}^{(\mathrm{coloc})}_{i,k},
    \label{eq:collocbg}
\end{equation}
where $\boldsymbol{\epsilon}^{(\mathrm{coloc})}_{i,k}\sim\Normal\bigl(\mathbf{0},\boldsymbol{\Sigma}_{\mathrm{coloc}}\bigr)$.
This construction was proposed by \citet{long2022autoip} and \citet{hamelijnck2024physsgp}. For nonlinear fields, these virtual observations generally yield a non-Gaussian trajectory posterior, but a Gaussian Markov approximation permits efficient smoothing computations. However, this remains an auxiliary inference construction and not the exact ODE-induced trajectory distribution. Section~\ref{sec:trajupdate} develops a variational formulation of soft ODE constraints for unknown hierarchical vector fields.

\section{Mixed-effect GP ODEs}
\label{sec:method}
\vspace{-3pt}

Sections~\ref{sec:fieldprior} and~\ref{sec:scalable} introduce the hierarchical GP-ODE model and its finite-feature approximation. Section~\ref{sec:trajupdate} defines the relaxed inference target and variational criterion, with soft ODE constraints replacing exact ODE conditioning. Sections~\ref{sec:fieldupdate} and~\ref{sec:trajectory-optimization} describe the joint field and trajectory updates. Figure~\ref{fig:method-workflow} summarizes the pipeline.

\begin{figure*}[!t]
    \vspace*{-4mm}
    \setlength{\abovecaptionskip}{-2pt}
\setlength{\belowcaptionskip}{-4pt}
    \centering
    \includegraphics[width=\textwidth]{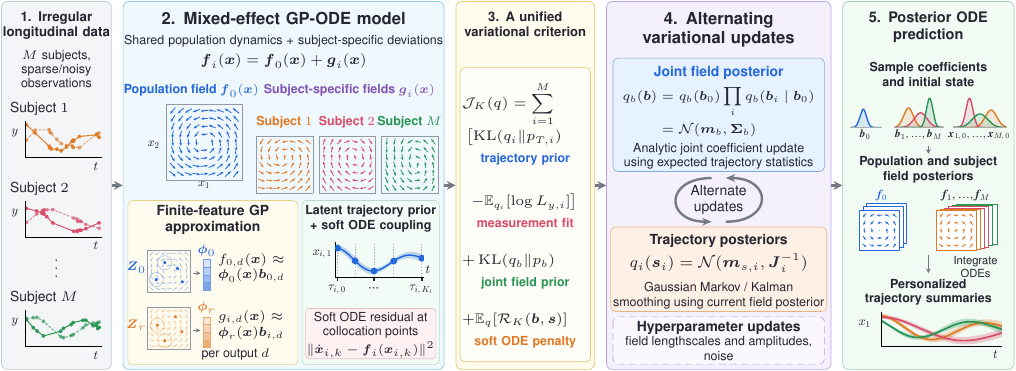}
    \caption{\textbf{\method{} inference workflow.}}
    \label{fig:method-workflow}
\end{figure*}

\vspace{-3pt}
\subsection{Hierarchical GP-ODE model}
\label{sec:fieldprior}

\begingroup
\setlength{\abovedisplayskip}{2pt}
\setlength{\belowdisplayskip}{2pt}
\setlength{\abovedisplayshortskip}{0pt}
\setlength{\belowdisplayshortskip}{2pt}
\setlength{\jot}{1pt}
We model each subject-specific field as the sum of a population field and an individual deviation,
\begingroup\setlength{\abovedisplayskip}{2pt}\setlength{\belowdisplayskip}{2pt}\setlength{\abovedisplayshortskip}{0pt}\setlength{\belowdisplayshortskip}{2pt}
\begin{equation}
    \vf_i(\vx)=\vf_0(\vx)+\vg_i(\vx),
    \qquad i=1,\dots,M.
    \label{eq:decomp}
\end{equation}
\endgroup
Here, a subject may be a patient or a biological replicate, and $\vf_0,\vg_i:\R^D\to\R^D$. The random effect $\vg_i$ is an entire subject-specific vector field.
We use independent-output GP priors across state dimensions $d=1,\dots,D$:
\begin{align}
    \begin{split}
    f_{0,d}(\cdot)&\sim\GP(0,k_0),\\
    f_{0,d}&\independent f_{0,d'}\quad\text{for }d\neq d',
    \end{split}
    \label{eq:sharedgp}\\
    \begin{split}
    g_{i,d}(\cdot)&\sim\GP\bigl(0,\lambda_r^{-1}k_r\bigr),\\
    g_{i,d}&\independent g_{j,d'}\quad\text{for }(i,d)\neq(j,d').
    \end{split}
    \label{eq:localgp}
\end{align}
The shared and local GP families are mutually independent. The shrinkage parameter $\lambda_r>0$ controls the prior magnitude of subject-specific deviations. For fixed $k_r$, increasing $\lambda_r$ reduces their prior variance.
Then, for any inputs $\vx,\vx'$,
\begin{equation}
    \Cov[\vf_i(\vx),\vf_j(\vx')]
    =[k_0(\vx,\vx')+\tfrac{\delta_{ij}}{\lambda_r}k_r(\vx,\vx')]\Id_D.
    \label{eq:fieldcov}
\end{equation}
Here $\delta_{ab}$ denotes the Kronecker delta.
Thus all subjects borrow strength through the common field, which is exactly the classical NLME philosophy translated from parameter space to function space. MAGMA provides a closely related common-mean construction for trajectory regression \citep{leroy2022magma}; here the common component acts directly on the vector field.
Independent GP outputs allow coupled dynamics: each output depends on the full state $\vx$.

\textbf{Initial conditions.} We use fixed Gaussian priors:
\begin{equation}
\vx_{i,0}\sim\Normal(\bm m_{x,0},\boldsymbol P_{x,0}).
\label{eq:initial-state-prior}
\end{equation}
Initial states are independent across subjects and of the GP fields a priori. Eqs.~\eqref{eq:obs}--\eqref{eq:ode} complete the generative model.
\endgroup

\begingroup\setlength{\abovedisplayskip}{2pt}\setlength{\belowdisplayskip}{2pt}\setlength{\abovedisplayshortskip}{0pt}\setlength{\belowdisplayshortskip}{2pt}\setlength{\jot}{1pt}
\subsection{Scalable shared-local vector field representation}
\label{sec:scalable}

Direct GP inference over all field evaluations is expensive. To obtain a scalable approximation, the shared and subject-specific fields are represented via Nystr\"om inducing-feature expansions.
Let $\boldsymbol{Z}_0\in\R^{l_0\times D}$ and $\boldsymbol{Z}_r\in\R^{l_r\times D}$ denote the inducing locations for the population and residual fields, with $l_0$ and $l_r$ inducing points, respectively.
Here $r$ labels the residual component and $i$ indexes subjects. Residual fields share $\boldsymbol{Z}_r$ and features, with subject-specific coefficients.
For $a\in\{0,r\}$, define the feature maps as
\begin{equation}
    \boldsymbol{\phi}_a(\vx)=k_a(\vx,\boldsymbol{Z}_a)
      k_a(\boldsymbol{Z}_a,\boldsymbol{Z}_a)^{-1/2}\in\R^{1\times l_a}.
    \label{eq:features}
\end{equation}
For each output dimension $d$, we approximate
\begin{align}
    f_{0,d}(\vx)&\approx\boldsymbol{\phi}_0(\vx)\vb_{0,d},\quad
    g_{i,d}(\vx)\approx\boldsymbol{\phi}_r(\vx)\vb_{i,d},
    \label{eq:weightspace}\\
    \vb_{0,d}&\sim\Normal(\mathbf{0},\Id_{l_0}),\quad
    \vb_{i,d}\sim\Normal(\mathbf{0},\lambda_r^{-1}\Id_{l_r}).
    \label{eq:weightprior}
\end{align}
The inverse square root in Eq.~\eqref{eq:features} whitens the features, encoding kernel correlations in the maps and allowing independent Gaussian coefficient priors for $\vb_{0,d}$ and $\vb_{i,d}$.

Stacking each field's $D$ output coefficient vectors gives $\vb_0\in\R^{Dl_0}$ and $\vb_i\in\R^{Dl_r}$. Collecting the shared and $M$ local coefficient vectors as $\vb=(\vb_0^\top,\vb_1^\top,\ldots,\vb_M^\top)^\top$, each subject's field is
\begingroup\setlength{\abovedisplayskip}{2pt}\setlength{\abovedisplayshortskip}{0pt}
\begin{equation}
\vf_i(\vx)=\boldsymbol F_i(\vx)\vb.
\label{eq:global-field}
\end{equation}
\endgroup
The matrix $\boldsymbol F_i$ applies the shared and local feature maps to $\vb_0$ and $\vb_i$, respectively; coefficients belonging to other subjects do not contribute. The independent coefficient priors give $p_b(\vb)=\Normal(\mathbf0,\boldsymbol P)$ with block-diagonal covariance $\boldsymbol P$.

We approximate the coefficients jointly by $q_b(\vb)=\Normal(\vm_b,\vSigma_b)$. The resulting field mean and covariance are
\begin{equation}
\bm\mu_i^f(\vx)=\boldsymbol F_i(\vx)\vm_b,\quad
\bm V_i^f(\vx)=\boldsymbol F_i(\vx)\vSigma_b\boldsymbol F_i(\vx)^\top.
\label{eq:fieldmoments}
\end{equation}
The covariance includes shared--local cross terms. Shared and local coefficients are independent under the prior but generally correlated in the joint posterior.

\vspace{-8pt}
\subsection{Target distribution and variational inference}
\label{sec:trajupdate}
We now define a target distribution over vector fields and latent trajectories. Based on the IWP prior in Eq.~\eqref{eq:iwp-path-prior}, it favors trajectories that explain the measurements and have small ODE residuals. Following physics-informed GP formulations~\citep{long2022autoip,hamelijnck2024physsgp}, we penalize these residuals with a quadratic loss instead of imposing the ODE as a hard equality constraint. This provides a relaxed inference distribution for the finite-feature model of Section~\ref{sec:scalable}. Appendix~\ref{app:target} gives its full-GP counterpart.

\textbf{Collocation grid and likelihood.} We infer the shared field $\vf_0$ and local fields $\{\vg_i\}_{i=1}^M$ from ODE residuals of $\vf_i\!=\!\vf_0+\vg_i$, avoiding repeated ODE solves. Subject $i$\textquotesingle s grid $\mathcal C_i\!=\!\{\tau_{i,k}\}_{k=0}^{K_i}$ combines measurement times, additional collocation points, and both endpoints. With $T_i\!=\!\tau_{i,K_i}-\tau_{i,0}>0$, quadrature weights $\omega_{i,k}$ approximate the time integral of the squared residual normalized by $T_i$ (Appendix~\ref{app:target}). For constraint strength $\kappa>0$, the residual penalty and measurement likelihood from Eq.~\eqref{eq:obs} are
\begingroup\setlength{\belowdisplayskip}{2pt}\setlength{\belowdisplayshortskip}{2pt}
\begin{align}
\mathcal R_K(\vb,\vs)&=\frac\kappa2\sum_{i,k}\omega_{i,k}
\|\dot{\vx}_{i,k}-\boldsymbol F_i(\vx_{i,k})\vb\|^2,
\label{eq:residual-loss}\\
L_y(\vs)&=\prod_{i,n}\Normal\bigl(\vy_{i,n}\mid\vH\vx_i(t_{i,n}),\boldsymbol\Sigma_y\bigr).
\label{eq:physical-likelihood}
\end{align}
\endgroup
Hence, each residual factor is proportional to the density in Eq.~\eqref{eq:collocbg}, with $\boldsymbol\Sigma_{\mathrm{coloc},i,k}=(\kappa\omega_{i,k})^{-1}\Id_D$.

\vspace{-2pt}
\textbf{Target distribution.} Let $\vs=(\vs_1^\top,\ldots,\vs_M^\top)^\top$ stack the subject trajectories. We use independent IWP priors $p_{T,i}$ from Eq.~\eqref{eq:iwp-path-prior}, with the initial-state marginal in Eq.~\eqref{eq:initial-state-prior} and fixed augmented-state parameters (Appendix~\ref{app:init-conditions}). Combining these priors with $p_b$, $L_y$, and $\mathcal R_K$ gives the target:
\begin{equation}
\pi_K(\vb,\vs\mid\vy)=\frac{p_b(\vb)}{Z_K}\prod_i p_{T,i}(\vs_i)L_y(\vs)e^{-\mathcal R_K(\vb,\vs)}.
\label{eq:relaxed-target}
\end{equation}
Here $Z_K$ normalizes the target and $K$ indicates the grid discretization. For $\boldsymbol\Sigma_y\succ0$ and finite ODE residuals, $0<Z_K<\infty$. Thus $\pi_K$ is a Gibbs probability distribution, an exponential tilt of the joint prior by the measurement loss $-\log L_y$ and ODE penalty $\mathcal R_K$ (Appendix~\ref{app:target}).

\textbf{Variational criterion.} We approximate $\pi_K$ with a joint Gaussian coefficient factor $q_b(\vb)=\Normal(\vm_b,\vSigma_b)$ and Gaussian Markov trajectory factors:
\begin{equation}
q(\vb,\vs)=q_b(\vb)\prod_i q_i(\vs_i),\enspace
q_i=\Normal(\vm_{s,i},\boldsymbol J_i^{-1}),
\label{eq:variational}
\end{equation}
\endgroup
where each $\boldsymbol J_i$ is positive definite and block tridiagonal, preserving the IWP prior's temporal dependence and enabling Gaussian smoothing in linear time in $K_i$. Under $q$, coefficients are independent of trajectories, and trajectories are independent across subjects. The full coefficient covariance $\vSigma_b$ captures shared--local and cross-subject correlations. We minimize $\KL(q\|\pi_K)$, equivalently:
\begingroup
\setlength{\abovedisplayskip}{4pt plus1pt minus1pt}
\setlength{\belowdisplayskip}{4pt plus1pt minus1pt}
\setlength{\abovedisplayshortskip}{0pt plus1pt}
\setlength{\belowdisplayshortskip}{3pt plus1pt minus1pt}
\setlength{\jot}{1pt}
\begin{equation}
\begin{aligned}
\mathcal J_K(q)={}&\sum\nolimits_i\bigl[\KL(q_i\|p_{T,i})-\E_{q_i}[\log L_{y,i}]\bigr]\\
&+\KL(q_b\|p_b)+\E_{q_b\prod_i q_i}[\mathcal R_K].
\end{aligned}
\label{eq:objective}
\end{equation}
Here $L_{y,i}$ is subject $i$\textquotesingle s likelihood factor in Eq.~\eqref{eq:physical-likelihood}. The next two sections optimize this objective over the field and trajectory factors.

\vspace{-4pt}
\subsection{Joint Gaussian field update at uncertain inputs}
\label{sec:fieldupdate}
\setlength{\abovedisplayskip}{2.5pt plus.5pt minus.5pt}
\setlength{\belowdisplayskip}{2.5pt plus.5pt minus.5pt}
\setlength{\belowdisplayshortskip}{2pt plus1pt minus1pt}
We minimize Eq.~\eqref{eq:objective} over the joint coefficient factor $q_b$, holding the trajectory factors $q_i$ fixed. Only $\KL(q_b\|p_b)$ and $\E_q[\mathcal R_K]$ depend on $q_b$. Averaging Eq.~\eqref{eq:residual-loss} over the fixed trajectory factors leaves a quadratic function of $\vb$. Combined with the Gaussian coefficient prior, this gives an exact Gaussian optimum. With precision $\vLambda_b=\vSigma_b^{-1}$ and precision-weighted mean $\veta_b=\vLambda_b\vm_b$, the update is:
\begin{equation}
\begin{aligned}
\vLambda_b&=\boldsymbol P^{-1}+\sum\nolimits_{i,k}\kappa\omega_{i,k}
\E_{q_i}[\boldsymbol F_i(\vx)^\top\boldsymbol F_i(\vx)],\\
\veta_b&=\sum\nolimits_{i,k}\kappa\omega_{i,k}\E_{q_i}[\boldsymbol F_i(\vx)^\top\dot{\vx}],\\
\vSigma_b&=\vLambda_b^{-1},\qquad \vm_b=\vSigma_b\veta_b.
\end{aligned}
\label{eq:coefficient-update}
\end{equation}
The resulting $\vm_b$ and $\vSigma_b$ update the field mean and covariance through Eq.~\eqref{eq:fieldmoments}. The expectations propagate state uncertainty through the nonlinear feature maps and retain state--derivative covariance.

\textbf{Shared--local structure.} Each subject couples shared coefficients with its own local coefficients. We solve for the shared marginal, then recover the local conditional distributions. At the coefficient optimum,
\begin{equation}
q_b(\vb)=q_b(\vb_0)\prod\nolimits_i q_b(\vb_i\mid\vb_0).
\label{eq:coefficient-factorization}
\end{equation}
This avoids a dense solve over all subjects while preserving shared--local and cross-subject correlations. Appendix~\ref{app:fieldupdates} gives the block construction and a five-subject example. Analytically minimizing over $q_b$ also yields a reduced objective for trajectory optimization (Appendix~\ref{app:profiled}).

\vspace{-4pt}
\subsection{Gaussian Markov trajectory optimization}
\label{sec:trajectory-optimization}
Next, we optimize Eq.~\eqref{eq:objective} over the trajectory factors $q_i$, holding $q_b$ fixed. Averaging the residual penalty in Eq.~\eqref{eq:residual-loss} over $q_b$ leaves a function of the states and derivatives. At each collocation time, this gives:
\begin{align}
\ell_{i,k}(\vs_{i,k})&=\tfrac{\kappa\omega_{i,k}}2\E_{q_b}\|\dot{\vx}_{i,k}-\boldsymbol F_i(\vx_{i,k})\vb\|^2
\label{eq:trajectory-potential}\\
&=\tfrac{\kappa\omega_{i,k}}2\bigl[
\|\dot{\vx}_{i,k}-\bm\mu_i^f(\vx_{i,k})\|^2+
\tr(\bm V_i^f(\vx_{i,k}))\bigr].\notag
\end{align}
Here $\bm\mu_i^f$ and $\bm V_i^f$ are the field mean and covariance in Eq.~\eqref{eq:fieldmoments}, with $q_b$ fixed. Averaging over the trajectories gives the residual term in Eq.~\eqref{eq:objective}: $\E_q[\mathcal R_K]=\sum_{i,k}\E_{q_i}[\ell_{i,k}]$.
Together with the trajectory KL and measurement loss, this yields the unrestricted minimizer $q_i^\star\propto p_{T,i}L_{y,i}\exp[-\sum_k\ell_{i,k}]$. The nonlinear state dependence makes this density generally non-Gaussian; we optimize within the Gaussian Markov family of Eq.~\eqref{eq:variational}.

\textbf{Averaging over trajectories.} Iterated expectation gives $\E_{q_i}[\ell_{i,k}]=\E_{q_i(\vx_{i,k})}\bigl[\E_{q_i(\dot{\vx}_{i,k}\mid\vx_{i,k})}[\ell_{i,k}]\bigr]$. Conditioning on $\vx_{i,k}$ fixes the field mean and covariance, leaving a loss quadratic in $\dot{\vx}_{i,k}$. The inner expectation follows exactly from Gaussian conditional moments; the outer retains the nonlinear state dependence. This reduces integration from $2D$ to $D$ variables, retaining state--derivative covariance. Omitting subject/time indices, with marginal means $\vm_x,\vm_{\dot{x}}$ and $\boldsymbol P_{\dot{x}x}=\Cov_{q_i}(\dot{\vx},\vx)=\boldsymbol P_{x\dot{x}}^\top$, the conditional derivative moments are
\begin{equation}
\begin{aligned}
\overline{\dot{\vx}}(\vx)&=\vm_{\dot{x}}+\boldsymbol P_{\dot{x}x}\boldsymbol P_{xx}^{-1}(\vx-\vm_x),\\
\boldsymbol V_{\dot{x}\mid x}&=\boldsymbol P_{\dot{x}\dot{x}}-\boldsymbol P_{\dot{x}x}\boldsymbol P_{xx}^{-1}\boldsymbol P_{x\dot{x}},
\end{aligned}
\label{eq:conditional-velocity}
\end{equation}
with $\overline{\dot{\vx}}(\vx)=\E_{q_i}[\dot{\vx}\mid\vx]$; averaging over $q_i(\vx)$ gives
\begin{equation}
\E_{q_i}\ell_{i,k}=\tfrac{\kappa\omega_{i,k}}2\E_{q_i}\bigl[
\|\overline{\dot{\vx}}-\bm\mu_i^f\|^2+
\tr(\boldsymbol V_{\dot{x}\mid x}+\bm V_i^f)\bigr].
\label{eq:expected-residual}
\end{equation}
\endgroup
For RBF/Nystr\"om features, this nonlinear state expectation has an analytic expression (Appendix~\ref{app:colloc-derivation}).

We update the trajectory means $\vm_{s,i}$ and covariances $\boldsymbol J_i^{-1}$ within the Gaussian Markov family of Eq.~\eqref{eq:variational}, keeping $\boldsymbol J_i$ block tridiagonal to couple only adjacent collocation times. We accept updates that do not increase the objective after recomputing the optimal $q_b$; Appendix~\ref{app:kalman} gives the update construction. We initialize each $q_i$ by Kalman smoothing using only the observations under the IWP prior, then compute the first coefficient update using Eq.~\eqref{eq:coefficient-update}.

For ODE prediction, we use the fitted marginal $q_i(\vx_{i,0})=\Normal(\vm_{x,i,0},\boldsymbol P_{xx,i,0})$ as the starting-state distribution. Sampling this marginal propagates initial-state uncertainty through the rollouts (Appendix~\ref{app:init-conditions}).

\textbf{Hyperparameters.}
We learn shared and local kernel lengthscales and amplitudes, and diagonal observation-noise variances. Each proposal holds $q_i$ fixed and recomputes the optimal $q_b$ to evaluate the profiled criterion. Feature locations, the IWP diffusion scale, and $\kappa$ remain fixed at initialization (Appendix~\ref{subsec:hyperparameter_learning}).

Algorithm~\ref{alg:main} summarizes inference; Figure~\ref{fig:appendix-plate-gpode} shows the target factors. Appendix~\ref{app:theory} analyzes shared--local uncertainty and how adding subjects affects shared covariance.

\section{Related work}
\label{sec:related}
\vspace{-.3cm}

\textbf{Flexible continuous-time dynamics.}
GP-ODEs and neural ODE models form the closest continuous-time modelling lineage to our approach. GP-based ODE learning spans nonparametric vector-field inference \citep{heinonen2018npode}, gradient matching \citep{dondelinger2013ode,wenk2020odin,lorenzi}, manifold-constrained GP inference \citep{yang2021magi}, variational multiple shooting \citep{hegde2022vmsgp}, and collocation-based physics-informed models \citep{long2022autoip,hamelijnck2024physsgp}. Neural ODEs~\citep{chen2018neuralode}, latent ODEs~\citep{rubanova2019latentode}, and Neural CDEs~\citep{kidger2020neuralcde} provide flexible alternatives for irregular-time prediction and forecasting. These methods learn expressive dynamics, but most assume a shared dynamical system rather than subject-specific vector fields.
On another note, probabilistic ODE solvers use state-space GP priors and Bayesian filtering to quantify numerical solution uncertainty for specified ODEs~\citep{tronarp2019filtering}.

\textbf{Personalized neural dynamics.}
Recent approaches adapt continuous-time neural dynamics across environments or patients. Mixed-effect Neural ODEs~\citep{nazarovs2022menode} use Gaussian weights to define additive shared--local latent vector fields. CoDA~\citep{kirchmeyer2022generalizing}, MetaPhysiCa~\citep{mouli2024metaphysica}, and DIF~\citep{iflgui} personalize neural dynamics through context variables or meta-learned representations. \method{} places GP priors on shared and subject-specific fields in the modeled state coordinates and jointly infers their coefficients while accounting for trajectory uncertainty.

\textbf{Mixed-effect models and longitudinal GPs.}
Parametric NLME-ODEs remain the dominant framework for population dynamics in pharmacometrics and systems biology \citep{lindstrom1990nlme,lavielle2014mixedeffects,wang2014mixedode}. Population Kalman filtering \citep{collin2022popkf} and amortized or in-context estimators \citep{arruda2024amortized,li2026vae,ojedamarin2025aicmet} improve scalability, personalization, or inference speed within predefined parametric or neural model families.
Mixed-effect GP regression models, including MAGMA~\citep{leroy2022magma} and other formulations~\citep{kimme,yoon2022doubly,cheng2019longp,timonen2025mixeddomain}, share information across subjects but do not infer ODE vector fields.
In contrast, \method{} places the mixed-effect hierarchy directly over the vector field, combining personalization, flexible dynamics, and function-space uncertainty.

\vspace{-.3cm}
\section{Experiments}
\vspace{-.3cm}
\label{sec:experiments}

We evaluate \method{} on synthetic population-ODE benchmarks with known ground-truth dynamics. Experiments assess whether the mixed-effect field decomposition improves trajectory prediction under subject heterogeneity and whether function-space residuals can correct a misspecified mechanistic population model.

\begin{table*}[!t]
\centering
\vspace*{-4mm}
\includegraphics[width=\linewidth,trim=0 2bp 0 0,clip]{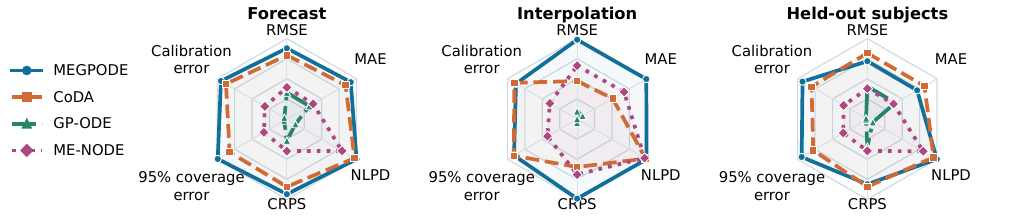}
\par\vspace{-3pt}
{\small\bfseries Absolute forecast performance}\par\vspace{-1pt}
\begingroup
\resizebox{\linewidth}{!}{%
\begin{minipage}{1.12\linewidth}
\fontsize{8.2}{9.6}\selectfont
\renewcommand{\arraystretch}{1}
\setlength{\tabcolsep}{1.2pt}
\setlength{\aboverulesep}{.2ex}
\setlength{\belowrulesep}{.35ex}
\dimen0=0pt
\setbox0=\hbox{\begin{tabular}{@{}lcccccc@{}}

\toprule
Method & FHN & VDP & PK & SIR & LV & LV-R \\
\midrule
GP-ODE & $.965{\pm}.33$ & $\underline{.577}{\pm}.27$ & $1.30{\pm}.69$ & $.323{\pm}.18$ & $.372{\pm}.16$ & $.424{\pm}.050$ \\
ME-NODE & $\underline{.680}{\pm}.050$ & $1.66{\pm}.044$ & $.291{\pm}.077$ & $.112{\pm}.018$ & $.333{\pm}.049$ & $.453{\pm}.062$ \\
CoDA & $.888{\pm}.13$ & $.689{\pm}.25$ & $\underline{.133}{\pm}.035$ & $\mathbf{.072}{\pm}.017$ & $\mathbf{.141}{\pm}.16$ & $\underline{.357}{\pm}.073$ \\
MEGPODE & $\mathbf{.404}{\pm}.069$ & $\mathbf{.467}{\pm}.084$ & $\mathbf{.125}{\pm}.027$ & $\underline{.106}{\pm}.017$ & $\underline{.143}{\pm}.019$ & $\mathbf{.346}{\pm}.031$ \\
\midrule
\rowcolor{gray!12}
\textcolor{black!60}{SAEMIX} & \textcolor{black!60}{$.071{\pm}.011$} & \textcolor{black!60}{$.015{\pm}.002$} & \textcolor{black!60}{$.126{\pm}.018$} & \textcolor{black!60}{$.0071{\pm}.0009$} & \textcolor{black!60}{$.018{\pm}.001$} & \textcolor{black!60}{$.433{\pm}.048$} \\
\bottomrule
\end{tabular}}
\ifdim\wd0>\dimen0\dimen0=\wd0\fi
\setbox0=\hbox{\begin{tabular}{@{}lcccccc@{}}

\toprule
Method & FHN & VDP & PK & SIR & LV & LV-R \\
\midrule
GP-ODE & $619{\pm}1100$ & $144{\pm}275$ & $(1.16{\pm}3.2)\times10^{4}$ & $(2.83{\pm}5.3)\times10^{4}$ & $(.940{\pm}1.4)\times10^{4}$ & $3300{\pm}4300$ \\
ME-NODE & $8.87{\pm}3.7$ & $135{\pm}41$ & $1.36{\pm}.54$ & $4.03{\pm}1.9$ & $21.7{\pm}14$ & $8.64{\pm}3.9$ \\
CoDA & $\underline{2.03}{\pm}.52$ & $\underline{14.0}{\pm}8.9$ & $\underline{.428}{\pm}1.2$ & $\mathbf{-1.04}{\pm}.51$ & $\underline{2.24}{\pm}9.5$ & $\underline{.996}{\pm}.67$ \\
MEGPODE & $\mathbf{.190}{\pm}.26$ & $\mathbf{.921}{\pm}.52$ & $\mathbf{-.732}{\pm}.21$ & $\underline{-.402}{\pm}.074$ & $\mathbf{-.631}{\pm}.16$ & $\mathbf{.347}{\pm}.12$ \\
\midrule
\rowcolor{gray!12}
\textcolor{black!60}{SAEMIX} & \textcolor{black!60}{$-2.13{\pm}.64$} & \textcolor{black!60}{$-3.17{\pm}.13$} & \textcolor{black!60}{$1.01{\pm}.29$} & \textcolor{black!60}{$-3.77{\pm}.12$} & \textcolor{black!60}{$-2.65{\pm}.075$} & \textcolor{black!60}{$33.7{\pm}23$} \\
\bottomrule
\end{tabular}}
\ifdim\wd0>\dimen0\dimen0=\wd0\fi
\setbox0=\hbox{\begin{tabular}{@{}lcccccc@{}}

\toprule
Method & FHN & VDP & PK & SIR & LV & LV-R \\
\midrule
GP-ODE & $.594{\pm}.32$ & $.425{\pm}.23$ & $.999{\pm}.56$ & $.250{\pm}.081$ & $.288{\pm}.12$ & $.323{\pm}.027$ \\
ME-NODE & $\underline{.363}{\pm}.034$ & $1.31{\pm}.052$ & $.156{\pm}.042$ & $\underline{.072}{\pm}.012$ & $.244{\pm}.043$ & $.292{\pm}.045$ \\
CoDA & $.455{\pm}.088$ & $\underline{.361}{\pm}.13$ & $\underline{.072}{\pm}.018$ & $\mathbf{.042}{\pm}.010$ & $\mathbf{.055}{\pm}.024$ & $\underline{.193}{\pm}.041$ \\
MEGPODE & $\mathbf{.196}{\pm}.042$ & $\mathbf{.258}{\pm}.051$ & $\mathbf{.064}{\pm}.011$ & $.078{\pm}.006$ & $\underline{.075}{\pm}.010$ & $\mathbf{.191}{\pm}.019$ \\
\midrule
\rowcolor{gray!12}
\textcolor{black!60}{SAEMIX} & \textcolor{black!60}{$.021{\pm}.004$} & \textcolor{black!60}{$.0071{\pm}.0009$} & \textcolor{black!60}{$.075{\pm}.012$} & \textcolor{black!60}{$.0036{\pm}.0004$} & \textcolor{black!60}{$.0098{\pm}.0006$} & \textcolor{black!60}{$.261{\pm}.022$} \\
\bottomrule
\end{tabular}}
\ifdim\wd0>\dimen0\dimen0=\wd0\fi
\setbox0=\hbox{\begin{tabular}{@{}lcccccc@{}}

\toprule
Method & FHN & VDP & PK & SIR & LV & LV-R \\
\midrule
GP-ODE & $66.7{\pm}23$ & $73.8{\pm}14$ & $82.9{\pm}13$ & $83.9{\pm}13$ & $91.2{\pm}1.9$ & $86.5{\pm}6.3$ \\
ME-NODE & $45.5{\pm}5.6$ & $82.0{\pm}1.5$ & $31.1{\pm}7.3$ & $42.5{\pm}4.5$ & $70.6{\pm}5.8$ & $47.0{\pm}8.1$ \\
CoDA & $\underline{12.6}{\pm}3.2$ & $\underline{39.6}{\pm}11$ & $\underline{18.1}{\pm}9.6$ & $\underline{10.4}{\pm}7.5$ & $\underline{8.57}{\pm}4.8$ & $\underline{13.9}{\pm}4.5$ \\
MEGPODE & $\mathbf{3.75}{\pm}1.4$ & $\mathbf{11.4}{\pm}8.2$ & $\mathbf{13.2}{\pm}5.2$ & $\mathbf{5.00}{\pm}.00$ & $\mathbf{2.45}{\pm}1.8$ & $\mathbf{4.26}{\pm}.49$ \\
\midrule
\rowcolor{gray!12}
\textcolor{black!60}{SAEMIX} & \textcolor{black!60}{$8.85{\pm}8.9$} & \textcolor{black!60}{$2.27{\pm}2.5$} & \textcolor{black!60}{$62.2{\pm}2.9$} & \textcolor{black!60}{$1.73{\pm}.96$} & \textcolor{black!60}{$2.29{\pm}1.9$} & \textcolor{black!60}{$56.2{\pm}2.2$} \\
\bottomrule
\end{tabular}}
\ifdim\wd0>\dimen0\dimen0=\wd0\fi
\dimen2=.49\linewidth
\ifdim\dimen0<\dimen2\dimen0=\dimen2\fi
\edef\resultstablewidth{\the\dimen0}
\noindent
\begin{minipage}[t]{.49\linewidth}
\centering
{\bfseries RMSE}\par\vspace{2pt}
\newcommand{\resultstablebody}{%
\toprule
Method & FHN & VDP & PK & SIR & LV & LV-R \\
\midrule
GP-ODE & $.965{\pm}.33$ & $\underline{.577}{\pm}.27$ & $1.30{\pm}.69$ & $.323{\pm}.18$ & $.372{\pm}.16$ & $.424{\pm}.050$ \\
ME-NODE & $\underline{.680}{\pm}.050$ & $1.66{\pm}.044$ & $.291{\pm}.077$ & $.112{\pm}.018$ & $.333{\pm}.049$ & $.453{\pm}.062$ \\
CoDA & $.888{\pm}.13$ & $.689{\pm}.25$ & $\underline{.133}{\pm}.035$ & $\mathbf{.072}{\pm}.017$ & $\mathbf{.141}{\pm}.16$ & $\underline{.357}{\pm}.073$ \\
MEGPODE & $\mathbf{.404}{\pm}.069$ & $\mathbf{.467}{\pm}.084$ & $\mathbf{.125}{\pm}.027$ & $\underline{.106}{\pm}.017$ & $\underline{.143}{\pm}.019$ & $\mathbf{.346}{\pm}.031$ \\
\midrule
\rowcolor{gray!12}
\textcolor{black!60}{SAEMIX} & \textcolor{black!60}{$.071{\pm}.011$} & \textcolor{black!60}{$.015{\pm}.002$} & \textcolor{black!60}{$.126{\pm}.018$} & \textcolor{black!60}{$.0071{\pm}.0009$} & \textcolor{black!60}{$.018{\pm}.001$} & \textcolor{black!60}{$.433{\pm}.048$} \\
\bottomrule
}
\resizebox{\linewidth}{!}{%
\begin{tabular*}{\resultstablewidth}{@{\extracolsep{\fill}}lcccccc@{}}
\resultstablebody
\end{tabular*}}
\end{minipage}
\hfill
\begin{minipage}[t]{.49\linewidth}
\centering
{\bfseries NLPD}\par\vspace{2pt}
\newcommand{\resultstablebody}{%
\toprule
Method & FHN & VDP & PK & SIR & LV & LV-R \\
\midrule
GP-ODE & $619{\pm}1100$ & $144{\pm}275$ & $(1.16{\pm}3.2)\times10^{4}$ & $(2.83{\pm}5.3)\times10^{4}$ & $(.940{\pm}1.4)\times10^{4}$ & $3300{\pm}4300$ \\
ME-NODE & $8.87{\pm}3.7$ & $135{\pm}41$ & $1.36{\pm}.54$ & $4.03{\pm}1.9$ & $21.7{\pm}14$ & $8.64{\pm}3.9$ \\
CoDA & $\underline{2.03}{\pm}.52$ & $\underline{14.0}{\pm}8.9$ & $\underline{.428}{\pm}1.2$ & $\mathbf{-1.04}{\pm}.51$ & $\underline{2.24}{\pm}9.5$ & $\underline{.996}{\pm}.67$ \\
MEGPODE & $\mathbf{.190}{\pm}.26$ & $\mathbf{.921}{\pm}.52$ & $\mathbf{-.732}{\pm}.21$ & $\underline{-.402}{\pm}.074$ & $\mathbf{-.631}{\pm}.16$ & $\mathbf{.347}{\pm}.12$ \\
\midrule
\rowcolor{gray!12}
\textcolor{black!60}{SAEMIX} & \textcolor{black!60}{$-2.13{\pm}.64$} & \textcolor{black!60}{$-3.17{\pm}.13$} & \textcolor{black!60}{$1.01{\pm}.29$} & \textcolor{black!60}{$-3.77{\pm}.12$} & \textcolor{black!60}{$-2.65{\pm}.075$} & \textcolor{black!60}{$33.7{\pm}23$} \\
\bottomrule
}
\resizebox{\linewidth}{!}{%
\begin{tabular*}{\resultstablewidth}{@{\extracolsep{\fill}}lcccccc@{}}
\resultstablebody
\end{tabular*}}
\end{minipage}
\par\vspace{5pt}
\noindent
\begin{minipage}[t]{.49\linewidth}
\centering
{\bfseries CRPS}\par\vspace{2pt}
\newcommand{\resultstablebody}{%
\toprule
Method & FHN & VDP & PK & SIR & LV & LV-R \\
\midrule
GP-ODE & $.594{\pm}.32$ & $.425{\pm}.23$ & $.999{\pm}.56$ & $.250{\pm}.081$ & $.288{\pm}.12$ & $.323{\pm}.027$ \\
ME-NODE & $\underline{.363}{\pm}.034$ & $1.31{\pm}.052$ & $.156{\pm}.042$ & $\underline{.072}{\pm}.012$ & $.244{\pm}.043$ & $.292{\pm}.045$ \\
CoDA & $.455{\pm}.088$ & $\underline{.361}{\pm}.13$ & $\underline{.072}{\pm}.018$ & $\mathbf{.042}{\pm}.010$ & $\mathbf{.055}{\pm}.024$ & $\underline{.193}{\pm}.041$ \\
MEGPODE & $\mathbf{.196}{\pm}.042$ & $\mathbf{.258}{\pm}.051$ & $\mathbf{.064}{\pm}.011$ & $.078{\pm}.006$ & $\underline{.075}{\pm}.010$ & $\mathbf{.191}{\pm}.019$ \\
\midrule
\rowcolor{gray!12}
\textcolor{black!60}{SAEMIX} & \textcolor{black!60}{$.021{\pm}.004$} & \textcolor{black!60}{$.0071{\pm}.0009$} & \textcolor{black!60}{$.075{\pm}.012$} & \textcolor{black!60}{$.0036{\pm}.0004$} & \textcolor{black!60}{$.0098{\pm}.0006$} & \textcolor{black!60}{$.261{\pm}.022$} \\
\bottomrule
}
\resizebox{\linewidth}{!}{%
\begin{tabular*}{\resultstablewidth}{@{\extracolsep{\fill}}lcccccc@{}}
\resultstablebody
\end{tabular*}}
\end{minipage}
\hfill
\begin{minipage}[t]{.49\linewidth}
\centering
{\bfseries 95\% coverage error (\%)}\par\vspace{2pt}
\newcommand{\resultstablebody}{%
\toprule
Method & FHN & VDP & PK & SIR & LV & LV-R \\
\midrule
GP-ODE & $66.7{\pm}23$ & $73.8{\pm}14$ & $82.9{\pm}13$ & $83.9{\pm}13$ & $91.2{\pm}1.9$ & $86.5{\pm}6.3$ \\
ME-NODE & $45.5{\pm}5.6$ & $82.0{\pm}1.5$ & $31.1{\pm}7.3$ & $42.5{\pm}4.5$ & $70.6{\pm}5.8$ & $47.0{\pm}8.1$ \\
CoDA & $\underline{12.6}{\pm}3.2$ & $\underline{39.6}{\pm}11$ & $\underline{18.1}{\pm}9.6$ & $\underline{10.4}{\pm}7.5$ & $\underline{8.57}{\pm}4.8$ & $\underline{13.9}{\pm}4.5$ \\
MEGPODE & $\mathbf{3.75}{\pm}1.4$ & $\mathbf{11.4}{\pm}8.2$ & $\mathbf{13.2}{\pm}5.2$ & $\mathbf{5.00}{\pm}.00$ & $\mathbf{2.45}{\pm}1.8$ & $\mathbf{4.26}{\pm}.49$ \\
\midrule
\rowcolor{gray!12}
\textcolor{black!60}{SAEMIX} & \textcolor{black!60}{$8.85{\pm}8.9$} & \textcolor{black!60}{$2.27{\pm}2.5$} & \textcolor{black!60}{$62.2{\pm}2.9$} & \textcolor{black!60}{$1.73{\pm}.96$} & \textcolor{black!60}{$2.29{\pm}1.9$} & \textcolor{black!60}{$56.2{\pm}2.2$} \\
\bottomrule
}
\resizebox{\linewidth}{!}{%
\begin{tabular*}{\resultstablewidth}{@{\extracolsep{\fill}}lcccccc@{}}
\resultstablebody
\end{tabular*}}
\end{minipage}
\par
\end{minipage}%
}
\endgroup

\setlength{\abovecaptionskip}{-2pt}
\setlength{\belowcaptionskip}{-4pt}
\caption{\textbf{Prediction across six systems.} Top: losses are min--max normalized per system and seed (1 best, 0 worst), averaged within systems, then equally across systems; held-out scores average subjects first. Bottom: forecast means $\pm$ SD across 10 seeds; coverage errors are in percentage points. Lower is better; bold and underline mark the best and second-best nonmechanistic means. Gray SAEMIX rows are mechanistic references supplied with the known ODE family and parameter random-effect structure, excluded from highlighting.}
\label{tab:prediction_results}
\end{table*}

\vspace{-.1cm}
\subsection{Benchmark systems and heterogeneity injection}
\label{sec:systems}
\vspace{-.1cm}

\textbf{Synthetic benchmark systems.}
We consider oscillatory and non-oscillatory population models. The comparison uses FitzHugh--Nagumo (FHN), Van der Pol (VDP),
pharmacokinetics (PK), susceptible--infectious--recovered (SIR), and Lotka--Volterra (LV) with random effects
on mechanistic parameters. A sixth system, LV-R, retains the LV population
field and adds smooth GP-draw subject residuals, testing function-space
heterogeneity beyond parametric random effects (Appendix~\ref{app:systems}; Figure~\ref{fig:examples}).

\textbf{Baselines.}
We compare \method{} against
GPODE~\citep{hegde2022vmsgp}, with a shared latent vector field and
subject-specific initial states; CoDA~\citep{kirchmeyer2022generalizing}, which
conditions a neural vector field on learned subject codes; and a mixed-effect
neural ODE (ME-NODE; \citealp{nazarovs2022menode}). A mechanistic NLME reference (SAEMIX, \citealt{comets2017parameter}) uses the known ODE family and random-effect structure, serving as a gold standard for well-specified systems (Appendix~\ref{app:concurrent_baselines}).

\textbf{Metrics.}
We evaluate predicted state trajectories against noise-free ground truth on dense time grids in interpolation, forecasting, and new-subject forecasting using RMSE, mean absolute error (MAE), negative log predictive density (NLPD), continuous ranked probability score (CRPS), 95\% coverage error, and calibration error. Appendix~\ref{app:evaluation-metrics} gives definitions and scoring details.

\vspace{-.1cm}
\subsection{Results}
\label{sec:results}
\vspace{-.2cm}

\textbf{Predictive performance.} Table~\ref{tab:prediction_results} summarizes forecasting, interpolation, and new-subject forecasting; Appendix Tables~\ref{tab:interpolation_results} and~\ref{tab:heldout_results} detail interpolation and new-subject scores. Across all three settings, \method{} offers the strongest overall balance of point accuracy, probabilistic scores, and calibration among nonmechanistic methods. CoDA is a strong competitor on RMSE and MAE, particularly for new-subject forecasting, but its deterministic fields do not provide epistemic uncertainty; its probabilistic scores rely on fitted residual variance. ME-NODE's stochastic mixed-effect ensemble achieves lower NLPD than GP-ODE, but remains less competitive than \method{}. GP-ODE's large NLPD values accompany severe undercoverage, which may reflect its restriction to a shared field with subject personalization through initial conditions. SAEMIX serves as a gold-standard mechanistic reference on the five systems where its ODE family and random effects match the generator, and generally performs best there. However, on LV-R, where random effects are smooth subject-specific residual fields beyond SAEMIX's mechanistic parameterization, \method{} performs better across the reported forecast metrics (Appendix~\ref{app:lv-saemix-comparison}). Lastly, Figures~\ref{fig:vdp_fits_forecast}, \ref{fig:vdp_fits_interpolation} and~\ref{fig:vdp_fits_heldout} display the fits achieved by baselines for a given seed on the VDP problem.

\textbf{Population dynamics.} We also assess recovery of the shared field and its zero-effect trajectories from a common starting state. Figure~\ref{fig:population-recovery-combined} reports field and trajectory errors, distinguishing the trajectory of the posterior-mean field $\vx(t;\E_q[\vf_0])$ from the mean of sampled-field trajectories $\E_q[\vx(t;\vf_0)]$. On PK and LV-R, \method{} achieves lower average population forecast error $R_{\mathrm{MT}}$ than SAEMIX while learning the shared field without a prescribed mechanistic ODE. In LV-R, this advantage holds even though SAEMIX uses the correct population ODE family, with only the random-effects model misspecified.

\textbf{Trajectory smoothing and ODE rollouts.} \method{} learns fields using an auxiliary trajectory smoother. Appendix~\ref{app:smoothing-rollouts} shows close agreement between smoothing and posterior ODE means during fitting on LV, despite different uncertainty bands, and compares two starting states.

\textbf{Shared and idiosyncratic contributions.} Finally, the hierarchy motivates examining how shared population dynamics and idiosyncratic subject effects contribute to trajectories and their uncertainty. Appendix~\ref{app:local-dynamics-contributions} illustrates these contributions on LV and PK and separates posterior trajectory uncertainty mediated by the shared field from local uncertainty conditional on it (Figure~\ref{fig:local-dynamics-contributions}).

\vspace{-.2cm}
\subsection{Correcting misspecified mechanistic dynamics}
\label{sec:vdp_residual_engineered}
\vspace{-.2cm}

\begin{figure}[t]
    \centering
    \vspace*{-13pt}
    \includegraphics[width=\columnwidth]{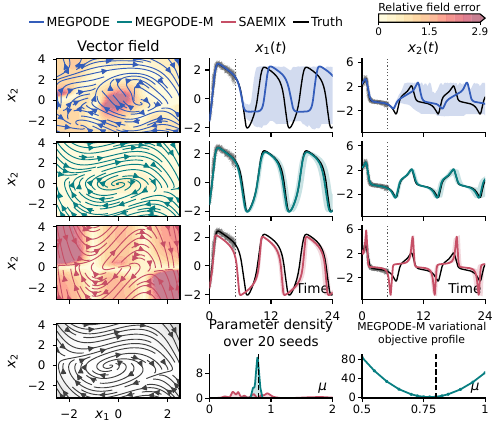}
    \setlength{\abovecaptionskip}{-14pt}
\setlength{\belowcaptionskip}{-4pt}
\caption{\textbf{Additive-residual Van der Pol experiment.} Rows 1--3 show population fields and state trajectories for \method{}, \methodm{}, and SAEMIX, respectively. Row 4 shows the true population field, SAEMIX and \methodm{} fitted densities of $\mu$ averaged over 20 seeds, and \methodm{}'s variational objective profile from refits at fixed $\mu$. Dashed vertical lines mark $\mu^\star=0.8$. Field shading shows relative error for fitted models. Trajectories show true (black) and fitted (colored) rollouts from the same fixed initial state; gray crosses are subject observations. Dotted lines mark the fitting cutoff $t=5$. Pointwise 95\% bands from 128 rollouts propagate SAEMIX's Gaussian parameter approximation or shared-field posteriors, including $q(\mu)$ for \methodm{}.}
    \label{fig:population_method_rows}
\end{figure}
Mechanistic models, such as those fitted with SAEMIX, approximate partly understood processes, so their vector fields can be structurally misspecified. \method{} can be extended to a semi-mechanistic model, \methodm{}, augmenting a mechanistic mean with learned GP corrections (Appendix~\ref{app:vdp_residual_engineered}). We test whether GP residual fields can correct a misspecified Van der Pol model. The true population field is $\vf_0^\star(\vx)=\vf_{\mathrm{VDP}}(\vx;\mu^\star)+\boldsymbol h_0(\vx)$, with $\mu^\star=0.8$ and a smooth additive residual $\boldsymbol h_0$. Each subject also has a smooth residual field $\vg_i^\star$.

Figure~\ref{fig:population_method_rows} compares \method{}, \methodm{}, and SAEMIX, with fields and trajectories shown for one seed. Despite the short observation window, \method{} captures the oscillations reasonably well. SAEMIX forecasts more accurately but produces overly sharp peaks in $x_2$. \methodm{} bridges the gap, closely reproducing both trajectories. Across 20 cohorts, \methodm{} achieves the lowest mean population forecast RMSE (Table~\ref{tab:vdp-additive-results}).

Analysis of the learned $\mu$ distributions provides a plausible explanation. The parameter panel averages 20 $\mu$ densities: Gaussian approximations for SAEMIX and variational distributions for \methodm{}. The 95\% intervals contain $\mu^\star=0.8$ in 11 of 20 cohorts for \methodm{}, versus one for SAEMIX. Moreover, we profile the objective by fixing $\mu$ on a grid and refitting the trajectory and GP coefficient factors, with fitted hyperparameters fixed. Its minimum near $0.774$, close to $\mu^\star=0.8$, suggests practical identifiability of $\mu$ under these fixed settings (Appendix~\ref{app:vdp_residual_engineered}).

\subsection{Ablations}
\label{sec:ablations}

We assess how the field hierarchy, ODE penalty, trajectory prior, feature maps, and data allocation affect prediction (Figure~\ref{fig:shared-local-ablation}; full results in Figure~\ref{fig:ablation-sensitivity-full}).

\begin{figure}[!t]
    \centering
    \vspace*{-13pt}
    \includegraphics[width=\columnwidth]{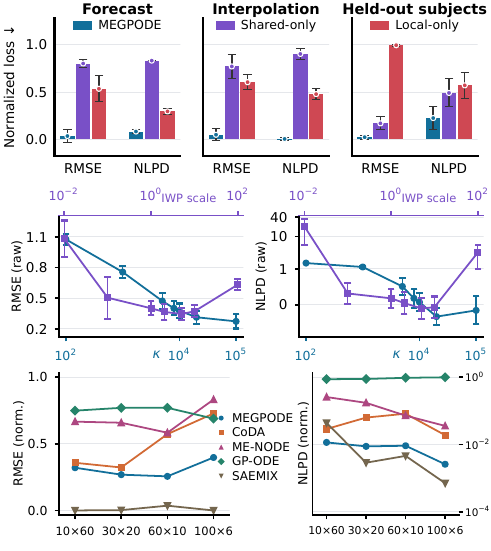}
    \setlength{\abovecaptionskip}{-15pt}
    \caption{\textbf{Shared/local ablations, IWP scale and penalty sensitivity, and sampling budget.} Top: cohort-wise min--max losses across three methods, equally averaged over six systems per seed; mean $\pm$ SD across seeds. Middle: FHN forecast metrics versus $\kappa$ (blue, bottom axis) or multiples of the observation-fitted IWP variance (purple, top axis); mean $\pm$ SD across ten paired cohorts. Bottom: mean normalized forecast losses for allocations of subjects $\times$ observations per subject.}
    \label{fig:shared-local-ablation}
\end{figure}

\textbf{Hierarchical decomposition.} Figure~\ref{fig:shared-local-ablation} compares full \method{} with separately tuned shared-only ($\vg_i\equiv\mathbf0$) and local-only ($\vf_0\equiv\mathbf0$) controls. Full \method{} achieves the lowest aggregate normalized RMSE and NLPD in all three evaluation settings. Unsurprisingly, shared-only performs worst on subject-trajectory forecasting and interpolation, whereas local-only performs worst on held-out subjects, as expected since a new field must be inferred from a short observation prefix without a learned shared drift. Interestingly, the marked gains over local-only fits suggest a practically useful degree of identifiability of the mixed-effect components under the fitted hierarchy.

\textbf{Penalty and trajectory-prior sensitivity.} On FHN, we separately vary the ODE penalty $\kappa$ (Eq.~\eqref{eq:residual-loss}) and the IWP variance $\sigma_{\mathrm{IWP}}^2$, which scales the trajectory-prior process noise (Eq.~\eqref{eq:iwp-path-prior}). Variances of 1--10 times the observation-fitted value perform well across most metrics, suggesting stable tuning within this range. Performance deteriorates at extreme scales. Increasing $\kappa$ strengthens ODE consistency and reduces RMSE, but the largest penalty worsens NLPD and coverage relative to $\kappa=20000$.

\textbf{Data regime.} We vary subject sampling at a fixed budget of 600 observation time points. \method{} retains the lowest aggregate normalized RMSE and NLPD among nonmechanistic methods at all four allocations, whereas CoDA's RMSE rises markedly with only 6--10 observations per subject. These results suggest that \method{}'s shared--local structure enables effective pooling across many sparsely observed subjects.

\textbf{Feature representation.} We compare Nystr\"om and random Fourier features as two ways to approximate GPs at fixed ranks (Appendix~\ref{app:altscalable}). Nystr\"om yields lower mean forecast losses and coverage errors, whereas interpolation scores are similar (Figure~\ref{fig:feature-representation-comparison}, bottom row).

\vspace{-.3cm}
\section{Discussion}
\label{sec:discussion}
\vspace{-.2cm}

We introduced \method{}, a flexible mixed-effect model for heterogeneous continuous-time dynamics, with GP priors over the population field and subject-specific deviations. Inference jointly estimates trajectories and vector fields under one variational objective, without ODE solves during fitting. This construction can be seen as the dynamical analogue of mixed-effect GPs~\citep{leroy2022magma}: the shared and local components act on the ODE drift rather than directly on trajectories or regression outputs. Moreover, we introduced the semi-mechanistic extension \methodm{} (Section~\ref{sec:vdp_residual_engineered}), which extends \emph{grey-box} modelling~\citep{hybridnode, greybox, singh2026variational} with a mixed-effect hierarchy over GP corrections. Shared corrections model systematic discrepancies in population dynamics, while local fields represent subject heterogeneity beyond the mechanistic parameterization. This makes \methodm{} especially promising for real-world applications, where assumptions about both population dynamics and individual variation are often misspecified.

\textbf{Limitations and perspectives.} Our empirical study is limited to synthetic examples. While this enables ground-truth evaluation of fields and population trajectories, validation on real-world longitudinal data remains a key next step. An important direction for future work is component identifiability. In the current model, GP priors regularize the allocation between shared and subject-specific fields, and our ablations suggest that this decomposition is practically useful. This issue also arises in \methodm{}, where variation may be explained either by the mechanistic component or by the GP residuals. Structural constraints, informative priors, and covariate-dependent parameters~\citep{lie2024covariate} could improve identifiability and interpretability. These directions reflect biology-informed ML, combining mechanistic structure, subject-level context, and flexible data-driven components~\citep{martinelli2025position}.

\clearpage
\subsection*{AI use statement}

In this work, we used generative AI tools to assist with the implementation of methods, create and edit software code, assist with drafting and editing parts of the manuscript, and improve its readability and formatting. We have reviewed all AI-assisted work. AI-assisted code was reviewed and tested by the authors, and AI-assisted manuscript text was checked against the underlying results and relevant literature. We take responsibility for the final content of this work, including text, claims, code, and other artifacts produced with the aid of generative AI.

\bibliographystyle{plainnat}
\bibliography{neurips_mixedeffect_ode}

@book{lavielle2014mixedeffects,
  author = {Marc Lavielle},
  title = {Mixed Effects Models for the Population Approach: Models, Tasks, Methods and Tools},
  publisher = {CRC press},
  year = {2014}
}

@article{clairon2023antibody,
  title={Modeling the kinetics of the neutralizing antibody response against SARS-CoV-2 variants after several administrations of Bnt162b2},
  author={Clairon, Quentin and Prague, M{\'e}lanie and Planas, Delphine and Bruel, Timoth{\'e}e and Hocqueloux, Laurent and Prazuck, Thierry and Schwartz, Olivier and Thi{\'e}baut, Rodolphe and Guedj, J{\'e}r{\'e}mie},
  journal={PLoS Computational Biology},
  year={2023},
}

@article{li2026vae,
  title={Variational autoencoder for inference of nonlinear mixed effect models based on ordinary differential equations},
  author={Li, Zhe and Prague, M{\'e}lanie and Thi{\'e}baut, Rodolphe and Clairon, Quentin},
  journal={arXiv preprint arXiv:2601.17400},
  year={2026}
}

@article{lindstrom1990nlme,
  author = {Mary J. Lindstrom and Douglas M. Bates},
  title = {Nonlinear mixed effects models for repeated measures data},
  journal = {Biometrics},
  year = {1990}
}

@article{wang2014mixedode,
  title={Estimating mixed-effects differential equation models},
  author={Wang, Liangliang and Cao, Jiguo and Ramsay, James O and Burger, DM and Laporte, CJL and Rockstroh, J{\"u}rgen K},
  journal={Statistics and Computing},
  year={2014},
}

@inproceedings{heinonen2018npode,
  author = {Markus Heinonen and Cagatay Yildiz and Henrik Mannerstrom and Jukka Intosalmi and Harri Lahdesmaki},
  title = {Learning unknown {ODE} models with Gaussian processes},
  booktitle = {Proceedings of the 35th International Conference on Machine Learning},
  year = {2018}
}

@inproceedings{wenk2020odin,
  title={Odin: Ode-informed regression for parameter and state inference in time-continuous dynamical systems},
  author={Wenk, Philippe and Abbati, Gabriele and Osborne, Michael A and Sch{\"o}lkopf, Bernhard and Krause, Andreas and Bauer, Stefan},
  booktitle={Proceedings of the AAAI Conference on Artificial Intelligence},
  year={2020}
}

@article{yang2021magi,
  author = {Shihao Yang and Samuel W. K. Wong and S. C. Kou},
  title = {Inference of dynamic systems from noisy and sparse data via manifold-constrained Gaussian processes},
  journal = {Proceedings of the National Academy of Sciences},
  year = {2021}
}

@inproceedings{hegde2022vmsgp,
  title={Variational multiple shooting for Bayesian ODEs with Gaussian processes},
  author={Hegde, Pashupati and Y{\i}ld{\i}z, {\c{C}}a{\u{g}}atay and L{\"a}hdesm{\"a}ki, Harri and Kaski, Samuel and Heinonen, Markus},
  booktitle={Uncertainty in Artificial Intelligence},
  year={2022},
}

@inproceedings{long2022autoip,
  title={AutoIP: A united framework to integrate physics into Gaussian processes},
  author={Long, Da and Wang, Zheng and Krishnapriyan, Aditi and Kirby, Robert and Zhe, Shandian and Mahoney, Michael},
  booktitle={International Conference on Machine Learning},
  year={2022},
}

@inproceedings{hamelijnck2024physsgp,
  author = {Oliver Hamelijnck and Arno Solin and Theodoros Damoulas},
  title = {Physics-Informed Variational State-Space {G}aussian Processes},
  booktitle = {Advances in Neural Information Processing Systems},
  year = {2024}
}

@inproceedings{chen2018neuralode,
  author = {Ricky T. Q. Chen and Yulia Rubanova and Jesse Bettencourt and David Duvenaud},
  title = {Neural ordinary differential equations},
  booktitle = {Advances in Neural Information Processing Systems},
  year = {2018}
}

@inproceedings{rubanova2019latentode,
  author = {Yulia Rubanova and Ricky T. Q. Chen and David Duvenaud},
  title = {Latent {ODE}s for irregularly-sampled time series},
  booktitle = {Advances in Neural Information Processing Systems},
  year = {2019}
}

@article{kidger2020neuralcde,
  title={Neural controlled differential equations for irregular time series},
  author={Kidger, Patrick and Morrill, James and Foster, James and Lyons, Terry},
  journal={Advances in neural information processing systems},
  year={2020}
}

@InProceedings{nazarovs2022menode,
  title = 	 {A variational approximation for analyzing the dynamics of panel data},
  author =       {Nazarovs, Jurijs and Chakraborty, Rudrasis and Tasneeyapant, Songwong and Ravi, Sathya and Singh, Vikas},
  booktitle = 	 {Proceedings of the Thirty-Seventh Conference on Uncertainty in Artificial Intelligence},
  year = 	 {2021},
}

@InProceedings{arruda2024amortized,
  title = 	 {An amortized approach to non-linear mixed-effects modeling based on neural posterior estimation},
  author =       {Arruda, Jonas and Sch\"{a}lte, Yannik and Peiter, Clemens and Teplytska, Olga and Jaehde, Ulrich and Hasenauer, Jan},
  booktitle = 	 {Proceedings of the 41st International Conference on Machine Learning},
  year = 	 {2024},
}

@article{bram2025monolix,
  title={Low-dimensional neural ordinary differential equations accounting for inter-individual variability implemented in Monolix and NONMEM},
  author={Br{\"a}m, Dominic Stefan and Steiert, Bernhard and Pfister, Marc and Steffens, Britta and Koch, Gilbert},
  journal={CPT: Pharmacometrics \& Systems Pharmacology},
  year={2025},
}

@inproceedings{
ojedamarin2025aicmet,
title={Amortized In-Context Mixed Effect Transformer Models: A Zero-Shot Approach for Pharmacokinetics},
author={Cesar Ojeda and Niklas Hartung and Wilhelm Huisinga and Ramses J Sanchez},
booktitle={The 29th International Conference on Artificial Intelligence and Statistics},
year={2026},
}

@article{collin2022popkf,
  title={Estimation for dynamical systems using a population-based Kalman filter--Applications in computational biology},
  author={Collin, Annabelle and Prague, M{\'e}lanie and Moireau, Philippe},
  journal={Mathematic{S} In Action},
  year={2022}
}

@article{leroy2022magma,
  author = {Arthur Leroy and Pierre Latouche and Benjamin Guedj and Servane Gey},
  title = {{MAGMA}: inference and prediction using multi-task {G}aussian processes with common mean},
  journal = {Machine Learning},
  year = {2022}
}

@inproceedings{rahimi2008rff,
  author = {Ali Rahimi and Benjamin Recht},
  title = {Random features for large-scale kernel machines},
  booktitle = {Advances in Neural Information Processing Systems},
  year = {2007}
}

@article{solin2020hilbert,
  author = {Arno Solin and Simo S{\"a}rkk{\"a}},
  title = {Hilbert space methods for reduced-rank {G}aussian process regression},
  journal = {Statistics and Computing},
  year = {2020}
}

@book{rasmussen2006gp,
  author = {Carl Edward Rasmussen and Christopher K. I. Williams},
  title = {Gaussian Processes for Machine Learning},
  publisher = {MIT Press},
  year = {2006}
}

@book{sarkka2013bfs,
  author = {Simo S{\"a}rkk{\"a}},
  title = {Bayesian Filtering and Smoothing},
  publisher = {Cambridge University Press},
  year = {2013}
}

@book{sarkka2019sde,
  author = {Simo S{\"a}rkk{\"a} and Arno Solin},
  title = {Applied Stochastic Differential Equations},
  publisher = {Cambridge University Press},
  year = {2019}
}

@article{tronarp2019filtering,
  title={Probabilistic solutions to ordinary differential equations as nonlinear Bayesian filtering: a new perspective},
  author={Tronarp, Filip and Kersting, Hans and S{\"a}rkk{\"a}, Simo and Hennig, Philipp},
  journal={Statistics and Computing},
  year={2019},
}

@article{machado2011modeling,
  title={Modeling formalisms in systems biology},
  author={Machado, Daniel and Costa, Rafael S and Rocha, Miguel and Ferreira, Eug{\'e}nio C and Tidor, Bruce and Rocha, Isabel},
  journal={AMB express},
  year={2011},
}

@article{alexandre2023prediction,
author = {Alexandre, Marie and Prague, Mélanie and McLean, Chelsea and Bockstal, Viki and Douoguih, Macaya and Thiébaut, Rodolphe and Effelterre, Thierry and Solforosi, Laura and Dari, Anna},
year = {2023},
title = {Prediction of long-term humoral response induced by the two-dose heterologous Ad26.ZEBOV, MVA-BN-Filo vaccine against Ebola},
journal = {npj Vaccines},
}

@inproceedings{
martinelli2025position,
title={Position: Biology is the Challenge Physics-Informed {ML} Needs to Evolve},
author={Julien Martinelli},
booktitle={The Thirty-Ninth Annual Conference on Neural Information Processing Systems Position Paper Track},
year={2025},
}

@inproceedings{dondelinger2013ode,
  title={ODE parameter inference using adaptive gradient matching with Gaussian processes},
  author={Dondelinger, Frank and Husmeier, Dirk and Rogers, Simon and Filippone, Maurizio},
  booktitle={Artificial intelligence and statistics},
  year={2013},
}

@InProceedings{lorenzi,
  title = 	 {Constraining the Dynamics of Deep Probabilistic Models},
  author =       {Lorenzi, Marco and Filippone, Maurizio},
  booktitle = 	 {Proceedings of the 35th International Conference on Machine Learning},
  year = 	 {2018},
}

@article{zhang2024model,
  title={Model-based smoothing with integrated Wiener processes and overlapping splines},
  author={Zhang, Ziang and Stringer, Alex and Brown, Patrick and Stafford, Jamie},
  journal={Journal of Computational and Graphical Statistics},
  year={2024},
}

@article{comets2017parameter,
  title={Parameter estimation in nonlinear mixed effect models using saemix, an R implementation of the SAEM algorithm},
  author={Comets, Emmanuelle and Lavenu, Audrey and Lavielle, Marc},
  journal={Journal of Statistical Software},
  year={2017}
}

@inproceedings{kirchmeyer2022generalizing,
  title={Generalizing to new physical systems via context-informed dynamics model},
  author={Kirchmeyer, Matthieu and Yin, Yuan and Don{\`a}, J{\'e}r{\'e}mie and Baskiotis, Nicolas and Rakotomamonjy, Alain and Gallinari, Patrick},
  booktitle={International conference on machine learning},
  year={2022},
}

@inproceedings{
mouli2024metaphysica,
title={MetaPhysiCa: Improving {OOD} Robustness in Physics-informed Machine Learning},
author={S Chandra Mouli and Muhammad Alam and Bruno Ribeiro},
booktitle={The Twelfth International Conference on Learning Representations},
year={2024},
}

@InProceedings{iflgui,
  title = 	 {Discovering Physics Laws of Dynamical Systems via Invariant Function Learning},
  author =       {Gui, Shurui and Li, Xiner and Ji, Shuiwang},
  booktitle = 	 {Proceedings of the 42nd International Conference on Machine Learning},
  year = 	 {2025},
}

@article{kimme, title={Deep Mixed Effect Model Using Gaussian Processes: A Personalized and Reliable Prediction for Healthcare},journal={Proceedings of the AAAI Conference on Artificial Intelligence}, author={Chung, Ingyo and Kim, Saehoon and Lee, Juho and Kim, Kwang Joon and Hwang, Sung Ju and Yang, Eunho}, year={2020}}

@inproceedings{yoon2022doubly,
  title={Doubly mixed-effects gaussian process regression},
  author={Yoon, Jun Ho and Jeong, Daniel P and Kim, Seyoung},
  booktitle={International Conference on Artificial Intelligence and Statistics},
  year={2022},
}

@InProceedings{greybox,
  title = 	 {Deep Grey-Box Modeling With Adaptive Data-Driven Models Toward Trustworthy Estimation of Theory-Driven Models},
  author =       {Takeishi, Naoya and Kalousis, Alexandros},
  booktitle = 	 {Proceedings of The 26th International Conference on Artificial Intelligence and Statistics},
  year = 	 {2023},
}

@InProceedings{hybridnode,
  title = 	 {Hybrid$^2$ Neural {ODE} Causal Modeling and an Application to Glycemic Response},
  author =       {Zou, Bob Junyi and Levine, Matthew E and Zaharieva, Dessi P. and Johari, Ramesh and Fox, Emily},
  booktitle = 	 {Proceedings of the 41st International Conference on Machine Learning},
  year = 	 {2024},
}

@InProceedings{singh2026variational,
  title={Variational Grey-Box Dynamics Matching},
  author={Singh, Gurjeet Sangra and Lavda, Frantzeska and Mercatali, Giangiacomo and Kalousis, Alexandros},
  booktitle = 	 {Proceedings of The 29th International Conference on Artificial Intelligence and Statistics},
  year={2026}
}

@misc{lie2024covariate,
  author={Lie, Han Cheng},
  title={Bayesian inference of covariate-parameter relationships for population modelling},
  year={2024},
  eprint={2407.09640},
  archivePrefix={arXiv},
  primaryClass={math.ST}
}

@article{gneiting2007strictly,
  author={Gneiting, Tilmann and Raftery, Adrian E.},
  title={Strictly Proper Scoring Rules, Prediction, and Estimation},
  journal={Journal of the American Statistical Association},
  volume={102},
  number={477},
  pages={359--378},
  year={2007},
  doi={10.1198/016214506000001437}
}

@article{cheng2019longp,
  author = {Cheng, Lu and Ramchandran, Siddharth and
            Vatanen, Tommi and Lietz{\'e}n, Niina and
            Lahesmaa, Riitta and Vehtari, Aki and
            L{\"a}hdesm{\"a}ki, Harri},
  title = {An additive {Gaussian} process regression model
           for interpretable non-parametric analysis
           of longitudinal data},
  journal = {Nature Communications},
  year = {2019},
  volume = {10},
  number = {1},
  eid = {1798},
  doi = {10.1038/s41467-019-09785-8}
}

@article{timonen2025mixeddomain,
  author = {Timonen, Juho and L{\"a}hdesm{\"a}ki, Harri},
  title = {Scalable mixed-domain {Gaussian} process modeling
           and model reduction for longitudinal data},
  journal = {Bayesian Analysis},
  year = {2025},
  note = {Advance online publication},
  doi = {10.1214/25-BA1539}
}

\clearpage
\onecolumn

\appendix
\makeatletter
\def\section{\@startsection{section}{1}{\z@}{-2.0ex plus
    -0.5ex minus -.2ex}{1.5ex plus 0.3ex minus .2ex}{\large\bfseries\raggedright}}
\makeatother

\part*{Appendix}
\addcontentsline{toc}{part}{Appendix}

\etocsetnexttocdepth{subsection}
\etocsettocstyle{\begingroup\small\parindent0pt\setlength{\parskip}{2pt}}{\par\endgroup}
\localtableofcontents

\setcounter{figure}{0}
\setcounter{table}{0}
\setcounter{equation}{0} %
\renewcommand{\thefigure}{S\arabic{figure}}
\renewcommand{\thetable}{S\arabic{table}}
\renewcommand{\theHfigure}{S\arabic{figure}}
\renewcommand{\theHtable}{S\arabic{table}}
\renewcommand{\theequation}{S\arabic{equation}}
\renewcommand{\theHequation}{S\arabic{equation}}

\newpage
\clearpage

\section{Supplementary Figures}
\label{app:supp-figures}

\begin{figure*}[h]
\centering
\resizebox{\textwidth}{!}{%
\begin{tikzpicture}[
  >=Latex,
  font=\small,
  node distance=10mm and 16mm,
  latent/.style={circle, draw, thick, minimum size=9mm, inner sep=0pt, align=center},
  obs/.style={latent, fill=gray!20},
  det/.style={rectangle, draw, rounded corners=2pt, thick, minimum height=8.5mm, inner sep=3pt, align=center, fill=blue!4},
  const/.style={rectangle, draw, dashed, rounded corners=2pt, thick, minimum height=8mm, inner sep=3pt, align=center},
  plate/.style={draw, rounded corners, thick, inner sep=5pt},
  note/.style={align=left, font=\footnotesize},
  lab/.style={font=\scriptsize}
]

\node[latent] (s0) {$\vs_{i,0}$};
\node[latent, right=18mm of s0] (s1) {$\vs_{i,1}$};
\node[draw=none, right=10mm of s1] (dotsL) {$\cdots$};
\node[latent, right=10mm of dotsL] (sk) {$\vs_{i,k}$};
\node[draw=none, right=10mm of sk] (dotsR) {$\cdots$};
\node[latent, right=10mm of dotsR] (sK) {$\vs_{i,K_i}$};

\node[const, above=9mm of s0] (sprior) {$\bm m_{\mathrm{init}},\;\boldsymbol P_{\mathrm{init}}$};
\draw[->, thick] (sprior) -- (s0);
\draw[->, thick] (s0) -- (s1);
\draw[->, thick] (s1) -- node[above, lab] {\qquad$\vA(h),\,\vQ(h)$} (dotsL);
\draw[->, thick] (dotsL) -- (sk);
\draw[->, thick] (sk) -- node[above, lab] {\qquad$\vA(h),\,\vQ(h)$} (dotsR);
\draw[->, thick] (dotsR) -- (sK);

\node[det, below=13mm of sk] (xk) {$\vx_{i,k}=\Hx\vs_{i,k}$};
\node[det, below=10mm of xk] (xdotk) {$\dot{\vx}_{i,k}=\Hxdot\vs_{i,k}$};
\draw[->, thick] (sk) -- (xk);
\draw[->, thick] (sk.south west) -- ++(-18mm,-7mm) |- (xdotk.west);

\node[obs, below left=22mm and 24mm of xdotk] (yk) {$\vy_{i,k}$};
\node[const, left=13mm of yk] (ysig) {$\boldsymbol{H},\,\boldsymbol{\Sigma}_{y}$};
\draw[->, thick] (xk.west) -- (yk.north);
\draw[->, thick] (ysig) -- (yk);

\node[note, below=3mm of yk] (obsnote) {include only when $\tau_{i,k}\in\mathcal{T}_i$};

\node[det, right=52mm of xk, minimum width=44mm] (fnode) {%
$\vf_i(\vx_{i,k})$\\[-0.5mm]
$=\boldsymbol F_i(\vx_{i,k})\vb$
};
\draw[->, thick] (xk) -- (fnode);

\node[latent, above left=60mm and 8mm of fnode] (b0d) {$\vb_0$};
\node[const, above=10mm of b0d] (prior0) {$\mathbf{0},\,\Id_{Dl_0}$};
\node[const, right=12mm of prior0] (kern0) {$k_0,\,\boldsymbol{Z}_0$};
\draw[->, thick] (prior0) -- (b0d);
\draw[->, thick] (b0d) -- (fnode);
\draw[->, thick] (kern0) -- (fnode);

\node[latent, above right=30mm and 8mm of fnode] (bid) {$\vb_i$};
\node[const, above=10mm of bid] (priori) {$\mathbf{0},\,\lambda_r^{-1}\Id_{Dl_r}$};
\node[const, left=12mm of priori] (kernr) {$k_r,\,\boldsymbol{Z}_r$};
\draw[->, thick] (priori) -- (bid);
\draw[->, thick] (bid) -- (fnode);
\draw[->, thick] (kernr) -- (fnode);
\node[det, right=24mm of xdotk, minimum width=30mm] (rcol) {%
$\bm r_{i,k}$\\[-0.5mm]
$=\dot{\vx}_{i,k}-\vf_i(\vx_{i,k})$
};
\node[det, right=18mm of rcol] (odefactor) {$\exp[-\tfrac12\kappa\omega_{i,k}\|\bm r_{i,k}\|^2]$};
\node[const, below=8mm of odefactor] (odeweight) {$\kappa,\,\omega_{i,k}$};

\draw[->, thick] (xdotk) -- (rcol);
\draw[->, thick] (fnode) -- (rcol);
\draw[->, thick] (rcol) -- (odefactor);
\draw[->, thick] (odeweight) -- (odefactor);

\node[plate,
      fit=(s0)(s1)(dotsL)(sk)(dotsR)(sK)(xk)(xdotk)(fnode)(rcol)(odefactor)(odeweight),
      label={[yshift=-1mm]below:{collocation times $k=0,\dots,K_i$}}] (platek) {};

\node[plate,
      fit=(platek)(sprior)(yk)(ysig)(obsnote)(kernr)(priori)(bid),
      label={[xshift=-2mm,yshift=-1mm]below right:{subjects $i=1,\dots,M$}}] (platei) {};

\node[note, below=11mm of platei.south west, anchor=west] (legend) {%
Grey circles are observations; rounded rectangles show deterministic maps and residual factors.\\
The weighted ODE factors enter the relaxed target in Eq.~\eqref{eq:relaxed-target}.\\
The subject-specific collocation grid satisfies $\mathcal{C}_i\supseteq\mathcal{T}_i$.};

\end{tikzpicture}%
}

\caption{\textbf{Factors defining the relaxed inference target distribution.} The shared coefficients $\vb_0$ enter every subject's field, while $\vb_i$ enters only subject $i$. The IWP prior supplies states and derivatives, and the measurement likelihood acts at observation times. Each collocation residual contributes the weighted factor $\exp[-\kappa\omega_{i,k}\|\dot{\vx}_{i,k}-\vf_i(\vx_{i,k})\|^2/2]$ to Eq.~\eqref{eq:relaxed-target}. ODE forecasts integrate sampled fields from the fitted starting-state distribution.}
\label{fig:appendix-plate-gpode}
\end{figure*}

\clearpage

\begingroup
\setlength{\intextsep}{3pt}
\begin{table}[H]
\centering
\begingroup
\resizebox{\linewidth}{!}{%
\begin{minipage}{1.12\linewidth}
\fontsize{8.2}{9.6}\selectfont
\renewcommand{\arraystretch}{1}
\setlength{\tabcolsep}{1.2pt}
\setlength{\aboverulesep}{.2ex}
\setlength{\belowrulesep}{.35ex}
\dimen0=0pt
\setbox0=\hbox{\begin{tabular}{@{}lcccccc@{}}

\toprule
Method & FHN & VDP & PK & SIR & LV & LV-R \\
\midrule
GP-ODE & $.937{\pm}.45$ & $1.30{\pm}.26$ & $.416{\pm}.093$ & $.141{\pm}.030$ & $.319{\pm}.050$ & $.254{\pm}.059$ \\
ME-NODE & $\underline{.480}{\pm}.062$ & $\underline{.750}{\pm}.23$ & $.122{\pm}.023$ & $.073{\pm}.009$ & $\underline{.152}{\pm}.023$ & $\underline{.194}{\pm}.031$ \\
CoDA & $1.16{\pm}.090$ & $1.05{\pm}.13$ & $\underline{.100}{\pm}.038$ & $\mathbf{.042}{\pm}.005$ & $.203{\pm}.092$ & $.243{\pm}.078$ \\
MEGPODE & $\mathbf{.359}{\pm}.035$ & $\mathbf{.365}{\pm}.061$ & $\mathbf{.037}{\pm}.007$ & $\underline{.044}{\pm}.002$ & $\mathbf{.091}{\pm}.040$ & $\mathbf{.085}{\pm}.009$ \\
\midrule
\rowcolor{gray!12}
\textcolor{black!60}{SAEMIX} & \textcolor{black!60}{$.082{\pm}.031$} & \textcolor{black!60}{$.016{\pm}.002$} & \textcolor{black!60}{$.148{\pm}.023$} & \textcolor{black!60}{$.015{\pm}.001$} & \textcolor{black!60}{$.027{\pm}.002$} & \textcolor{black!60}{$.125{\pm}.017$} \\
\bottomrule
\end{tabular}}
\ifdim\wd0>\dimen0\dimen0=\wd0\fi
\setbox0=\hbox{\begin{tabular}{@{}lcccccc@{}}

\toprule
Method & FHN & VDP & PK & SIR & LV & LV-R \\
\midrule
GP-ODE & $2400{\pm}2300$ & $947{\pm}584$ & $(1.36{\pm}2.5)\times10^{4}$ & $(.620{\pm}1.0)\times10^{4}$ & $(1.41{\pm}2.7)\times10^{5}$ & $(1.81{\pm}3.2)\times10^{4}$ \\
ME-NODE & $37.4{\pm}16$ & $99.0{\pm}37$ & $3.72{\pm}1.6$ & $8.01{\pm}5.2$ & $26.7{\pm}58$ & $13.6{\pm}7.7$ \\
CoDA & $\mathbf{1.75}{\pm}.21$ & $\underline{80.1}{\pm}63$ & $\underline{.378}{\pm}2.4$ & $\mathbf{-1.71}{\pm}.095$ & $\underline{.443}{\pm}.82$ & $\underline{.220}{\pm}.34$ \\
MEGPODE & $\underline{1.83}{\pm}1.00$ & $\mathbf{21.3}{\pm}14$ & $\mathbf{-.622}{\pm}.68$ & $\underline{-1.42}{\pm}.031$ & $\mathbf{-1.18}{\pm}.21$ & $\mathbf{-.630}{\pm}.20$ \\
\midrule
\rowcolor{gray!12}
\textcolor{black!60}{SAEMIX} & \textcolor{black!60}{$-1.75{\pm}.98$} & \textcolor{black!60}{$-2.84{\pm}.090$} & \textcolor{black!60}{$.152{\pm}.28$} & \textcolor{black!60}{$-2.79{\pm}.065$} & \textcolor{black!60}{$-2.26{\pm}.11$} & \textcolor{black!60}{$1.03{\pm}.61$} \\
\bottomrule
\end{tabular}}
\ifdim\wd0>\dimen0\dimen0=\wd0\fi
\setbox0=\hbox{\begin{tabular}{@{}lcccccc@{}}

\toprule
Method & FHN & VDP & PK & SIR & LV & LV-R \\
\midrule
GP-ODE & $.683{\pm}.38$ & $1.08{\pm}.25$ & $.323{\pm}.068$ & $.114{\pm}.025$ & $.258{\pm}.048$ & $.209{\pm}.051$ \\
ME-NODE & $\underline{.287}{\pm}.029$ & $\underline{.548}{\pm}.18$ & $.075{\pm}.013$ & $.048{\pm}.006$ & $\underline{.103}{\pm}.017$ & $\underline{.128}{\pm}.022$ \\
CoDA & $.616{\pm}.056$ & $.690{\pm}.11$ & $\underline{.055}{\pm}.024$ & $\mathbf{.024}{\pm}.002$ & $.119{\pm}.055$ & $.137{\pm}.044$ \\
MEGPODE & $\mathbf{.194}{\pm}.023$ & $\mathbf{.230}{\pm}.045$ & $\mathbf{.022}{\pm}.004$ & $\underline{.028}{\pm}.001$ & $\mathbf{.042}{\pm}.008$ & $\mathbf{.049}{\pm}.005$ \\
\midrule
\rowcolor{gray!12}
\textcolor{black!60}{SAEMIX} & \textcolor{black!60}{$.026{\pm}.011$} & \textcolor{black!60}{$.0086{\pm}.0008$} & \textcolor{black!60}{$.096{\pm}.016$} & \textcolor{black!60}{$.0086{\pm}.0006$} & \textcolor{black!60}{$.015{\pm}.001$} & \textcolor{black!60}{$.076{\pm}.012$} \\
\bottomrule
\end{tabular}}
\ifdim\wd0>\dimen0\dimen0=\wd0\fi
\setbox0=\hbox{\begin{tabular}{@{}lcccccc@{}}

\toprule
Method & FHN & VDP & PK & SIR & LV & LV-R \\
\midrule
GP-ODE & $88.4{\pm}5.8$ & $91.0{\pm}4.2$ & $89.4{\pm}5.1$ & $90.7{\pm}3.2$ & $91.9{\pm}6.3$ & $92.8{\pm}1.5$ \\
ME-NODE & $73.3{\pm}3.2$ & $75.9{\pm}4.3$ & $39.0{\pm}5.7$ & $50.9{\pm}5.1$ & $56.7{\pm}5.3$ & $55.6{\pm}7.4$ \\
CoDA & $\mathbf{14.6}{\pm}2.4$ & $\underline{57.7}{\pm}30$ & $\mathbf{17.3}{\pm}13$ & $\mathbf{3.08}{\pm}1.6$ & $\underline{18.0}{\pm}15$ & $\mathbf{9.85}{\pm}5.0$ \\
MEGPODE & $\underline{30.6}{\pm}5.6$ & $\mathbf{55.3}{\pm}8.8$ & $\underline{23.3}{\pm}6.6$ & $\underline{3.81}{\pm}.63$ & $\mathbf{6.55}{\pm}3.4$ & $\underline{17.0}{\pm}3.1$ \\
\midrule
\rowcolor{gray!12}
\textcolor{black!60}{SAEMIX} & \textcolor{black!60}{$16.9{\pm}11$} & \textcolor{black!60}{$1.70{\pm}1.6$} & \textcolor{black!60}{$30.6{\pm}5.0$} & \textcolor{black!60}{$1.68{\pm}.88$} & \textcolor{black!60}{$2.75{\pm}1.9$} & \textcolor{black!60}{$32.6{\pm}7.3$} \\
\bottomrule
\end{tabular}}
\ifdim\wd0>\dimen0\dimen0=\wd0\fi
\dimen2=.49\linewidth
\ifdim\dimen0<\dimen2\dimen0=\dimen2\fi
\edef\resultstablewidth{\the\dimen0}
\noindent
\begin{minipage}[t]{.49\linewidth}
\centering
{\bfseries RMSE}\par\vspace{2pt}
\newcommand{\resultstablebody}{%
\toprule
Method & FHN & VDP & PK & SIR & LV & LV-R \\
\midrule
GP-ODE & $.937{\pm}.45$ & $1.30{\pm}.26$ & $.416{\pm}.093$ & $.141{\pm}.030$ & $.319{\pm}.050$ & $.254{\pm}.059$ \\
ME-NODE & $\underline{.480}{\pm}.062$ & $\underline{.750}{\pm}.23$ & $.122{\pm}.023$ & $.073{\pm}.009$ & $\underline{.152}{\pm}.023$ & $\underline{.194}{\pm}.031$ \\
CoDA & $1.16{\pm}.090$ & $1.05{\pm}.13$ & $\underline{.100}{\pm}.038$ & $\mathbf{.042}{\pm}.005$ & $.203{\pm}.092$ & $.243{\pm}.078$ \\
MEGPODE & $\mathbf{.359}{\pm}.035$ & $\mathbf{.365}{\pm}.061$ & $\mathbf{.037}{\pm}.007$ & $\underline{.044}{\pm}.002$ & $\mathbf{.091}{\pm}.040$ & $\mathbf{.085}{\pm}.009$ \\
\midrule
\rowcolor{gray!12}
\textcolor{black!60}{SAEMIX} & \textcolor{black!60}{$.082{\pm}.031$} & \textcolor{black!60}{$.016{\pm}.002$} & \textcolor{black!60}{$.148{\pm}.023$} & \textcolor{black!60}{$.015{\pm}.001$} & \textcolor{black!60}{$.027{\pm}.002$} & \textcolor{black!60}{$.125{\pm}.017$} \\
\bottomrule
}
\resizebox{\linewidth}{!}{%
\begin{tabular*}{\resultstablewidth}{@{\extracolsep{\fill}}lcccccc@{}}
\resultstablebody
\end{tabular*}}
\end{minipage}
\hfill
\begin{minipage}[t]{.49\linewidth}
\centering
{\bfseries NLPD}\par\vspace{2pt}
\newcommand{\resultstablebody}{%
\toprule
Method & FHN & VDP & PK & SIR & LV & LV-R \\
\midrule
GP-ODE & $2400{\pm}2300$ & $947{\pm}584$ & $(1.36{\pm}2.5)\times10^{4}$ & $(.620{\pm}1.0)\times10^{4}$ & $(1.41{\pm}2.7)\times10^{5}$ & $(1.81{\pm}3.2)\times10^{4}$ \\
ME-NODE & $37.4{\pm}16$ & $99.0{\pm}37$ & $3.72{\pm}1.6$ & $8.01{\pm}5.2$ & $26.7{\pm}58$ & $13.6{\pm}7.7$ \\
CoDA & $\mathbf{1.75}{\pm}.21$ & $\underline{80.1}{\pm}63$ & $\underline{.378}{\pm}2.4$ & $\mathbf{-1.71}{\pm}.095$ & $\underline{.443}{\pm}.82$ & $\underline{.220}{\pm}.34$ \\
MEGPODE & $\underline{1.83}{\pm}1.00$ & $\mathbf{21.3}{\pm}14$ & $\mathbf{-.622}{\pm}.68$ & $\underline{-1.42}{\pm}.031$ & $\mathbf{-1.18}{\pm}.21$ & $\mathbf{-.630}{\pm}.20$ \\
\midrule
\rowcolor{gray!12}
\textcolor{black!60}{SAEMIX} & \textcolor{black!60}{$-1.75{\pm}.98$} & \textcolor{black!60}{$-2.84{\pm}.090$} & \textcolor{black!60}{$.152{\pm}.28$} & \textcolor{black!60}{$-2.79{\pm}.065$} & \textcolor{black!60}{$-2.26{\pm}.11$} & \textcolor{black!60}{$1.03{\pm}.61$} \\
\bottomrule
}
\resizebox{\linewidth}{!}{%
\begin{tabular*}{\resultstablewidth}{@{\extracolsep{\fill}}lcccccc@{}}
\resultstablebody
\end{tabular*}}
\end{minipage}
\par\vspace{5pt}
\noindent
\begin{minipage}[t]{.49\linewidth}
\centering
{\bfseries CRPS}\par\vspace{2pt}
\newcommand{\resultstablebody}{%
\toprule
Method & FHN & VDP & PK & SIR & LV & LV-R \\
\midrule
GP-ODE & $.683{\pm}.38$ & $1.08{\pm}.25$ & $.323{\pm}.068$ & $.114{\pm}.025$ & $.258{\pm}.048$ & $.209{\pm}.051$ \\
ME-NODE & $\underline{.287}{\pm}.029$ & $\underline{.548}{\pm}.18$ & $.075{\pm}.013$ & $.048{\pm}.006$ & $\underline{.103}{\pm}.017$ & $\underline{.128}{\pm}.022$ \\
CoDA & $.616{\pm}.056$ & $.690{\pm}.11$ & $\underline{.055}{\pm}.024$ & $\mathbf{.024}{\pm}.002$ & $.119{\pm}.055$ & $.137{\pm}.044$ \\
MEGPODE & $\mathbf{.194}{\pm}.023$ & $\mathbf{.230}{\pm}.045$ & $\mathbf{.022}{\pm}.004$ & $\underline{.028}{\pm}.001$ & $\mathbf{.042}{\pm}.008$ & $\mathbf{.049}{\pm}.005$ \\
\midrule
\rowcolor{gray!12}
\textcolor{black!60}{SAEMIX} & \textcolor{black!60}{$.026{\pm}.011$} & \textcolor{black!60}{$.0086{\pm}.0008$} & \textcolor{black!60}{$.096{\pm}.016$} & \textcolor{black!60}{$.0086{\pm}.0006$} & \textcolor{black!60}{$.015{\pm}.001$} & \textcolor{black!60}{$.076{\pm}.012$} \\
\bottomrule
}
\resizebox{\linewidth}{!}{%
\begin{tabular*}{\resultstablewidth}{@{\extracolsep{\fill}}lcccccc@{}}
\resultstablebody
\end{tabular*}}
\end{minipage}
\hfill
\begin{minipage}[t]{.49\linewidth}
\centering
{\bfseries 95\% coverage error (\%)}\par\vspace{2pt}
\newcommand{\resultstablebody}{%
\toprule
Method & FHN & VDP & PK & SIR & LV & LV-R \\
\midrule
GP-ODE & $88.4{\pm}5.8$ & $91.0{\pm}4.2$ & $89.4{\pm}5.1$ & $90.7{\pm}3.2$ & $91.9{\pm}6.3$ & $92.8{\pm}1.5$ \\
ME-NODE & $73.3{\pm}3.2$ & $75.9{\pm}4.3$ & $39.0{\pm}5.7$ & $50.9{\pm}5.1$ & $56.7{\pm}5.3$ & $55.6{\pm}7.4$ \\
CoDA & $\mathbf{14.6}{\pm}2.4$ & $\underline{57.7}{\pm}30$ & $\mathbf{17.3}{\pm}13$ & $\mathbf{3.08}{\pm}1.6$ & $\underline{18.0}{\pm}15$ & $\mathbf{9.85}{\pm}5.0$ \\
MEGPODE & $\underline{30.6}{\pm}5.6$ & $\mathbf{55.3}{\pm}8.8$ & $\underline{23.3}{\pm}6.6$ & $\underline{3.81}{\pm}.63$ & $\mathbf{6.55}{\pm}3.4$ & $\underline{17.0}{\pm}3.1$ \\
\midrule
\rowcolor{gray!12}
\textcolor{black!60}{SAEMIX} & \textcolor{black!60}{$16.9{\pm}11$} & \textcolor{black!60}{$1.70{\pm}1.6$} & \textcolor{black!60}{$30.6{\pm}5.0$} & \textcolor{black!60}{$1.68{\pm}.88$} & \textcolor{black!60}{$2.75{\pm}1.9$} & \textcolor{black!60}{$32.6{\pm}7.3$} \\
\bottomrule
}
\resizebox{\linewidth}{!}{%
\begin{tabular*}{\resultstablewidth}{@{\extracolsep{\fill}}lcccccc@{}}
\resultstablebody
\end{tabular*}}
\end{minipage}
\par
\end{minipage}%
}
\endgroup

\setlength{\abovecaptionskip}{3pt}
\caption{\textbf{Absolute interpolation performance.} Means $\pm$ sample standard deviations across seeds; coverage error is in percentage points. Bold and underline mark the best and second-best nonmechanistic methods. Gray SAEMIX rows are mechanistic references.}
\label{tab:interpolation_results}
\end{table}

\begin{table}[H]
\centering
\begingroup
\resizebox{\linewidth}{!}{%
\begin{minipage}{1.12\linewidth}
\fontsize{8.2}{9.6}\selectfont
\renewcommand{\arraystretch}{1}
\setlength{\tabcolsep}{1.2pt}
\setlength{\aboverulesep}{.2ex}
\setlength{\belowrulesep}{.35ex}
\dimen0=0pt
\setbox0=\hbox{\begin{tabular}{@{}lcccccc@{}}

\toprule
Method & FHN & VDP & PK & SIR & LV & LV-R \\
\midrule
GP-ODE & $\underline{1.04}{\pm}.33$ & $\underline{.614}{\pm}.24$ & $1.27{\pm}.69$ & $.299{\pm}.14$ & $.413{\pm}.19$ & $\underline{.328}{\pm}.16$ \\
ME-NODE & $1.12{\pm}.35$ & $1.69{\pm}.13$ & $.483{\pm}.19$ & $\underline{.119}{\pm}.020$ & $.327{\pm}.043$ & $.669{\pm}.35$ \\
CoDA & $1.05{\pm}.17$ & $\mathbf{.568}{\pm}.29$ & $\mathbf{.244}{\pm}.12$ & $\mathbf{.077}{\pm}.017$ & $\mathbf{.246}{\pm}.43$ & $.349{\pm}.061$ \\
MEGPODE & $\mathbf{.934}{\pm}.22$ & $.780{\pm}.29$ & $\underline{.337}{\pm}.078$ & $.124{\pm}.020$ & $\underline{.274}{\pm}.062$ & $\mathbf{.295}{\pm}.058$ \\
\midrule
\rowcolor{gray!12}
\textcolor{black!60}{SAEMIX} & \textcolor{black!60}{$.414{\pm}.11$} & \textcolor{black!60}{$.034{\pm}.014$} & \textcolor{black!60}{$.267{\pm}.054$} & \textcolor{black!60}{$.025{\pm}.010$} & \textcolor{black!60}{$.029{\pm}.005$} & \textcolor{black!60}{$.282{\pm}.052$} \\
\bottomrule
\end{tabular}}
\ifdim\wd0>\dimen0\dimen0=\wd0\fi
\setbox0=\hbox{\begin{tabular}{@{}lcccccc@{}}

\toprule
Method & FHN & VDP & PK & SIR & LV & LV-R \\
\midrule
GP-ODE & $2580{\pm}6200$ & $129{\pm}147$ & $(.887{\pm}2.2)\times10^{4}$ & $(2.77{\pm}5.2)\times10^{4}$ & $(1.51{\pm}3.1)\times10^{4}$ & $(.790{\pm}1.2)\times10^{4}$ \\
ME-NODE & $36.6{\pm}23$ & $87.9{\pm}41$ & $\underline{4.67}{\pm}2.5$ & $3.19{\pm}1.8$ & $\underline{22.4}{\pm}9.4$ & $45.7{\pm}35$ \\
CoDA & $\mathbf{3.97}{\pm}3.8$ & $\underline{15.9}{\pm}18$ & $16.1{\pm}31$ & $\mathbf{-.651}{\pm}.74$ & $99.5{\pm}312$ & $\underline{8.36}{\pm}8.7$ \\
MEGPODE & $\underline{8.86}{\pm}9.8$ & $\mathbf{1.43}{\pm}.68$ & $\mathbf{2.46}{\pm}1.8$ & $\underline{-.422}{\pm}.093$ & $\mathbf{1.53}{\pm}1.8$ & $\mathbf{1.68}{\pm}2.1$ \\
\midrule
\rowcolor{gray!12}
\textcolor{black!60}{SAEMIX} & \textcolor{black!60}{$.164{\pm}1.0$} & \textcolor{black!60}{$-2.12{\pm}.53$} & \textcolor{black!60}{$1.85{\pm}.65$} & \textcolor{black!60}{$-2.43{\pm}.49$} & \textcolor{black!60}{$-2.10{\pm}.25$} & \textcolor{black!60}{$14.8{\pm}7.2$} \\
\bottomrule
\end{tabular}}
\ifdim\wd0>\dimen0\dimen0=\wd0\fi
\setbox0=\hbox{\begin{tabular}{@{}lcccccc@{}}

\toprule
Method & FHN & VDP & PK & SIR & LV & LV-R \\
\midrule
GP-ODE & $.713{\pm}.29$ & $\underline{.481}{\pm}.20$ & $.986{\pm}.56$ & $.238{\pm}.066$ & $.332{\pm}.15$ & $.276{\pm}.14$ \\
ME-NODE & $.720{\pm}.28$ & $1.32{\pm}.13$ & $.316{\pm}.12$ & $\underline{.079}{\pm}.015$ & $.252{\pm}.035$ & $.543{\pm}.32$ \\
CoDA & $\underline{.631}{\pm}.15$ & $\mathbf{.427}{\pm}.25$ & $\mathbf{.167}{\pm}.083$ & $\mathbf{.050}{\pm}.013$ & $\mathbf{.168}{\pm}.30$ & $\underline{.254}{\pm}.059$ \\
MEGPODE & $\mathbf{.527}{\pm}.14$ & $.487{\pm}.20$ & $\underline{.229}{\pm}.058$ & $.083{\pm}.009$ & $\underline{.169}{\pm}.038$ & $\mathbf{.192}{\pm}.040$ \\
\midrule
\rowcolor{gray!12}
\textcolor{black!60}{SAEMIX} & \textcolor{black!60}{$.189{\pm}.059$} & \textcolor{black!60}{$.021{\pm}.008$} & \textcolor{black!60}{$.177{\pm}.042$} & \textcolor{black!60}{$.016{\pm}.007$} & \textcolor{black!60}{$.017{\pm}.003$} & \textcolor{black!60}{$.215{\pm}.045$} \\
\bottomrule
\end{tabular}}
\ifdim\wd0>\dimen0\dimen0=\wd0\fi
\setbox0=\hbox{\begin{tabular}{@{}lcccccc@{}}

\toprule
Method & FHN & VDP & PK & SIR & LV & LV-R \\
\midrule
GP-ODE & $74.3{\pm}16$ & $77.2{\pm}9.5$ & $81.3{\pm}13$ & $84.9{\pm}13$ & $91.1{\pm}3.8$ & $90.5{\pm}3.5$ \\
ME-NODE & $49.3{\pm}13$ & $78.5{\pm}4.1$ & $52.7{\pm}16$ & $38.0{\pm}9.3$ & $69.8{\pm}4.8$ & $80.6{\pm}5.3$ \\
CoDA & $\mathbf{24.5}{\pm}15$ & $\underline{48.4}{\pm}19$ & $\underline{36.9}{\pm}14$ & $\underline{20.7}{\pm}8.9$ & $\underline{24.6}{\pm}9.0$ & $\underline{43.4}{\pm}14$ \\
MEGPODE & $\underline{27.6}{\pm}11$ & $\mathbf{15.9}{\pm}12$ & $\mathbf{35.7}{\pm}8.4$ & $\mathbf{4.95}{\pm}.17$ & $\mathbf{20.9}{\pm}9.7$ & $\mathbf{21.4}{\pm}14$ \\
\midrule
\rowcolor{gray!12}
\textcolor{black!60}{SAEMIX} & \textcolor{black!60}{$9.04{\pm}5.1$} & \textcolor{black!60}{$12.0{\pm}12$} & \textcolor{black!60}{$62.2{\pm}9.3$} & \textcolor{black!60}{$15.8{\pm}8.0$} & \textcolor{black!60}{$8.82{\pm}5.0$} & \textcolor{black!60}{$59.9{\pm}7.6$} \\
\bottomrule
\end{tabular}}
\ifdim\wd0>\dimen0\dimen0=\wd0\fi
\dimen2=.49\linewidth
\ifdim\dimen0<\dimen2\dimen0=\dimen2\fi
\edef\resultstablewidth{\the\dimen0}
\noindent
\begin{minipage}[t]{.49\linewidth}
\centering
{\bfseries RMSE}\par\vspace{2pt}
\newcommand{\resultstablebody}{%
\toprule
Method & FHN & VDP & PK & SIR & LV & LV-R \\
\midrule
GP-ODE & $\underline{1.04}{\pm}.33$ & $\underline{.614}{\pm}.24$ & $1.27{\pm}.69$ & $.299{\pm}.14$ & $.413{\pm}.19$ & $\underline{.328}{\pm}.16$ \\
ME-NODE & $1.12{\pm}.35$ & $1.69{\pm}.13$ & $.483{\pm}.19$ & $\underline{.119}{\pm}.020$ & $.327{\pm}.043$ & $.669{\pm}.35$ \\
CoDA & $1.05{\pm}.17$ & $\mathbf{.568}{\pm}.29$ & $\mathbf{.244}{\pm}.12$ & $\mathbf{.077}{\pm}.017$ & $\mathbf{.246}{\pm}.43$ & $.349{\pm}.061$ \\
MEGPODE & $\mathbf{.934}{\pm}.22$ & $.780{\pm}.29$ & $\underline{.337}{\pm}.078$ & $.124{\pm}.020$ & $\underline{.274}{\pm}.062$ & $\mathbf{.295}{\pm}.058$ \\
\midrule
\rowcolor{gray!12}
\textcolor{black!60}{SAEMIX} & \textcolor{black!60}{$.414{\pm}.11$} & \textcolor{black!60}{$.034{\pm}.014$} & \textcolor{black!60}{$.267{\pm}.054$} & \textcolor{black!60}{$.025{\pm}.010$} & \textcolor{black!60}{$.029{\pm}.005$} & \textcolor{black!60}{$.282{\pm}.052$} \\
\bottomrule
}
\resizebox{\linewidth}{!}{%
\begin{tabular*}{\resultstablewidth}{@{\extracolsep{\fill}}lcccccc@{}}
\resultstablebody
\end{tabular*}}
\end{minipage}
\hfill
\begin{minipage}[t]{.49\linewidth}
\centering
{\bfseries NLPD}\par\vspace{2pt}
\newcommand{\resultstablebody}{%
\toprule
Method & FHN & VDP & PK & SIR & LV & LV-R \\
\midrule
GP-ODE & $2580{\pm}6200$ & $129{\pm}147$ & $(.887{\pm}2.2)\times10^{4}$ & $(2.77{\pm}5.2)\times10^{4}$ & $(1.51{\pm}3.1)\times10^{4}$ & $(.790{\pm}1.2)\times10^{4}$ \\
ME-NODE & $36.6{\pm}23$ & $87.9{\pm}41$ & $\underline{4.67}{\pm}2.5$ & $3.19{\pm}1.8$ & $\underline{22.4}{\pm}9.4$ & $45.7{\pm}35$ \\
CoDA & $\mathbf{3.97}{\pm}3.8$ & $\underline{15.9}{\pm}18$ & $16.1{\pm}31$ & $\mathbf{-.651}{\pm}.74$ & $99.5{\pm}312$ & $\underline{8.36}{\pm}8.7$ \\
MEGPODE & $\underline{8.86}{\pm}9.8$ & $\mathbf{1.43}{\pm}.68$ & $\mathbf{2.46}{\pm}1.8$ & $\underline{-.422}{\pm}.093$ & $\mathbf{1.53}{\pm}1.8$ & $\mathbf{1.68}{\pm}2.1$ \\
\midrule
\rowcolor{gray!12}
\textcolor{black!60}{SAEMIX} & \textcolor{black!60}{$.164{\pm}1.0$} & \textcolor{black!60}{$-2.12{\pm}.53$} & \textcolor{black!60}{$1.85{\pm}.65$} & \textcolor{black!60}{$-2.43{\pm}.49$} & \textcolor{black!60}{$-2.10{\pm}.25$} & \textcolor{black!60}{$14.8{\pm}7.2$} \\
\bottomrule
}
\resizebox{\linewidth}{!}{%
\begin{tabular*}{\resultstablewidth}{@{\extracolsep{\fill}}lcccccc@{}}
\resultstablebody
\end{tabular*}}
\end{minipage}
\par\vspace{5pt}
\noindent
\begin{minipage}[t]{.49\linewidth}
\centering
{\bfseries CRPS}\par\vspace{2pt}
\newcommand{\resultstablebody}{%
\toprule
Method & FHN & VDP & PK & SIR & LV & LV-R \\
\midrule
GP-ODE & $.713{\pm}.29$ & $\underline{.481}{\pm}.20$ & $.986{\pm}.56$ & $.238{\pm}.066$ & $.332{\pm}.15$ & $.276{\pm}.14$ \\
ME-NODE & $.720{\pm}.28$ & $1.32{\pm}.13$ & $.316{\pm}.12$ & $\underline{.079}{\pm}.015$ & $.252{\pm}.035$ & $.543{\pm}.32$ \\
CoDA & $\underline{.631}{\pm}.15$ & $\mathbf{.427}{\pm}.25$ & $\mathbf{.167}{\pm}.083$ & $\mathbf{.050}{\pm}.013$ & $\mathbf{.168}{\pm}.30$ & $\underline{.254}{\pm}.059$ \\
MEGPODE & $\mathbf{.527}{\pm}.14$ & $.487{\pm}.20$ & $\underline{.229}{\pm}.058$ & $.083{\pm}.009$ & $\underline{.169}{\pm}.038$ & $\mathbf{.192}{\pm}.040$ \\
\midrule
\rowcolor{gray!12}
\textcolor{black!60}{SAEMIX} & \textcolor{black!60}{$.189{\pm}.059$} & \textcolor{black!60}{$.021{\pm}.008$} & \textcolor{black!60}{$.177{\pm}.042$} & \textcolor{black!60}{$.016{\pm}.007$} & \textcolor{black!60}{$.017{\pm}.003$} & \textcolor{black!60}{$.215{\pm}.045$} \\
\bottomrule
}
\resizebox{\linewidth}{!}{%
\begin{tabular*}{\resultstablewidth}{@{\extracolsep{\fill}}lcccccc@{}}
\resultstablebody
\end{tabular*}}
\end{minipage}
\hfill
\begin{minipage}[t]{.49\linewidth}
\centering
{\bfseries 95\% coverage error (\%)}\par\vspace{2pt}
\newcommand{\resultstablebody}{%
\toprule
Method & FHN & VDP & PK & SIR & LV & LV-R \\
\midrule
GP-ODE & $74.3{\pm}16$ & $77.2{\pm}9.5$ & $81.3{\pm}13$ & $84.9{\pm}13$ & $91.1{\pm}3.8$ & $90.5{\pm}3.5$ \\
ME-NODE & $49.3{\pm}13$ & $78.5{\pm}4.1$ & $52.7{\pm}16$ & $38.0{\pm}9.3$ & $69.8{\pm}4.8$ & $80.6{\pm}5.3$ \\
CoDA & $\mathbf{24.5}{\pm}15$ & $\underline{48.4}{\pm}19$ & $\underline{36.9}{\pm}14$ & $\underline{20.7}{\pm}8.9$ & $\underline{24.6}{\pm}9.0$ & $\underline{43.4}{\pm}14$ \\
MEGPODE & $\underline{27.6}{\pm}11$ & $\mathbf{15.9}{\pm}12$ & $\mathbf{35.7}{\pm}8.4$ & $\mathbf{4.95}{\pm}.17$ & $\mathbf{20.9}{\pm}9.7$ & $\mathbf{21.4}{\pm}14$ \\
\midrule
\rowcolor{gray!12}
\textcolor{black!60}{SAEMIX} & \textcolor{black!60}{$9.04{\pm}5.1$} & \textcolor{black!60}{$12.0{\pm}12$} & \textcolor{black!60}{$62.2{\pm}9.3$} & \textcolor{black!60}{$15.8{\pm}8.0$} & \textcolor{black!60}{$8.82{\pm}5.0$} & \textcolor{black!60}{$59.9{\pm}7.6$} \\
\bottomrule
}
\resizebox{\linewidth}{!}{%
\begin{tabular*}{\resultstablewidth}{@{\extracolsep{\fill}}lcccccc@{}}
\resultstablebody
\end{tabular*}}
\end{minipage}
\par
\end{minipage}%
}
\endgroup

\setlength{\abovecaptionskip}{3pt}
\caption{\textbf{Absolute forecast performance for held-out subjects.} Scores average subjects within each seed; entries show means $\pm$ sample standard deviations across ten seeds. Coverage error is in percentage points.}
\label{tab:heldout_results}
\end{table}

\endgroup

\clearpage
\begin{figure}[H]
\centering
\includegraphics[width=\linewidth]{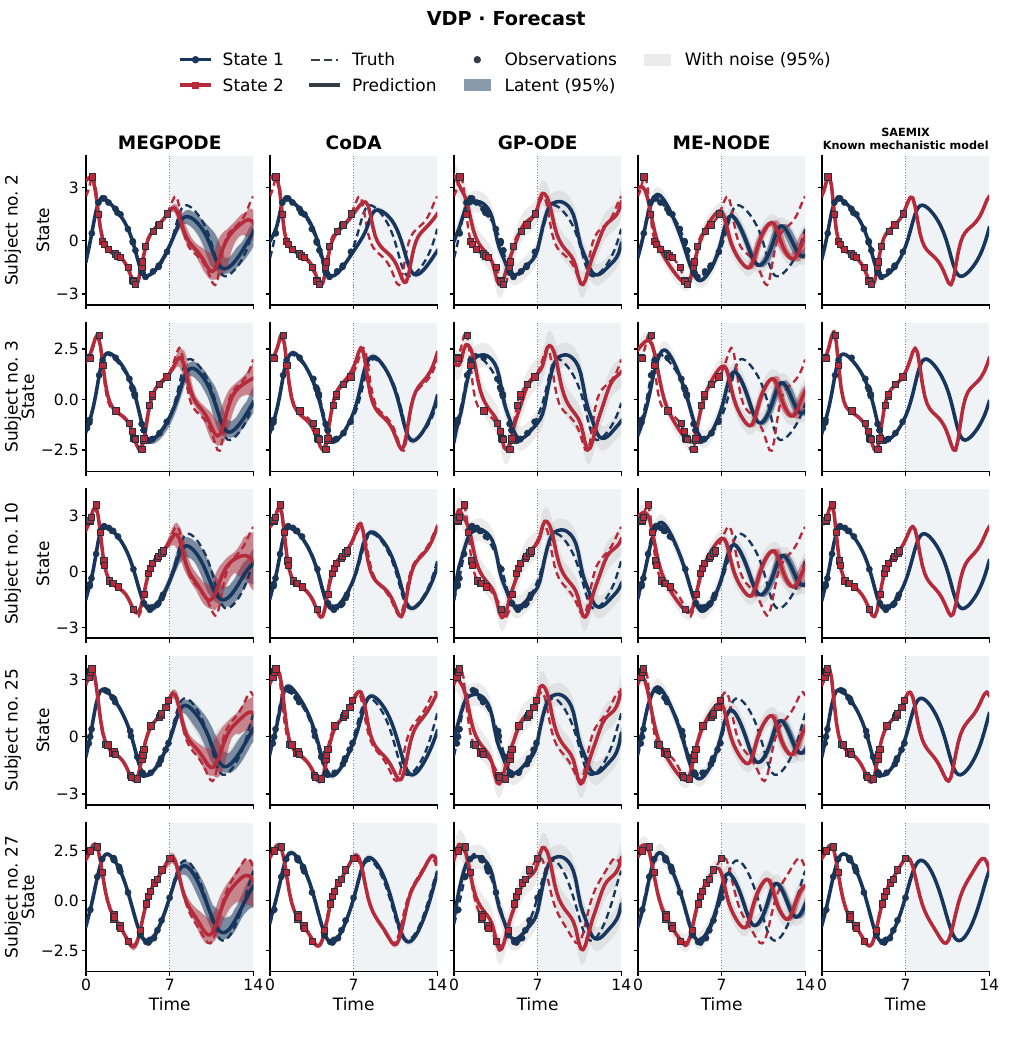}
\caption{\textbf{Van der Pol forecasting.} Five training subjects are displayed. The observation window ends at $t=7$; the pale region is unobserved. Dark blue and red denote the two state dimensions; dashed curves show truth, solid curves predictions, and markers fitting observations. Colored bands show pointwise 95\% posterior intervals (posterior ODE rollouts for MEGPODE and GP-ODE, retained latent trajectory samples for ME-NODE, conditional parameter draws for SAEMIX). Gray bands show 95\% predictive intervals including fitted independent observation noise; CoDA has only this band.}
\label{fig:vdp_fits_forecast}
\end{figure}
\clearpage
\begin{figure}[H]
\centering
\includegraphics[width=\linewidth]{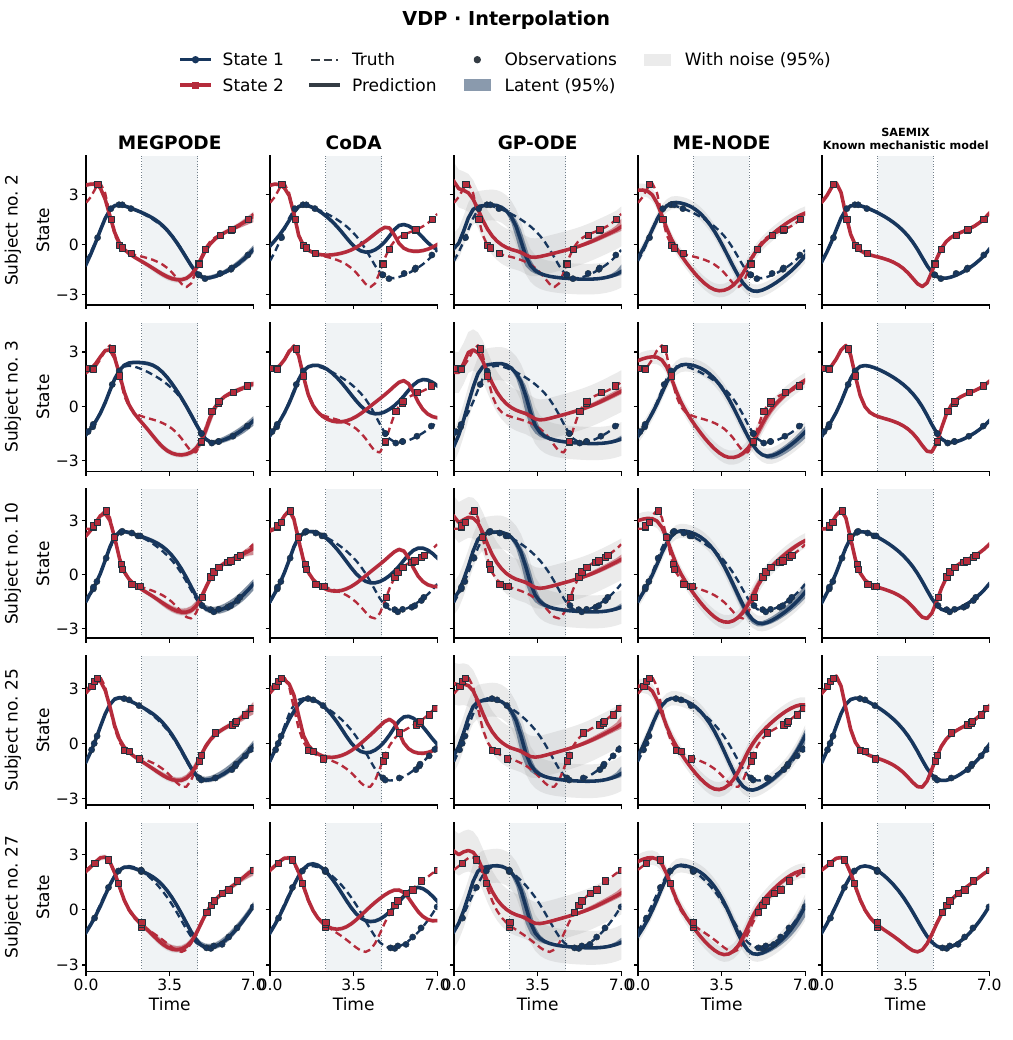}
\caption{\textbf{Van der Pol interpolation.} The pale region is the hidden middle third of the observation window. Dark blue and red denote the two state dimensions; dashed curves show truth, solid curves predictions, and markers fitting observations. Colored bands show pointwise 95\% posterior intervals (posterior ODE rollouts for MEGPODE and GP-ODE, retained latent trajectory samples for ME-NODE, conditional parameter draws for SAEMIX). Gray bands show 95\% predictive intervals including fitted independent observation noise; CoDA has only this band.}
\label{fig:vdp_fits_interpolation}
\end{figure}
\clearpage
\begin{figure}[H]
\centering
\includegraphics[width=\linewidth]{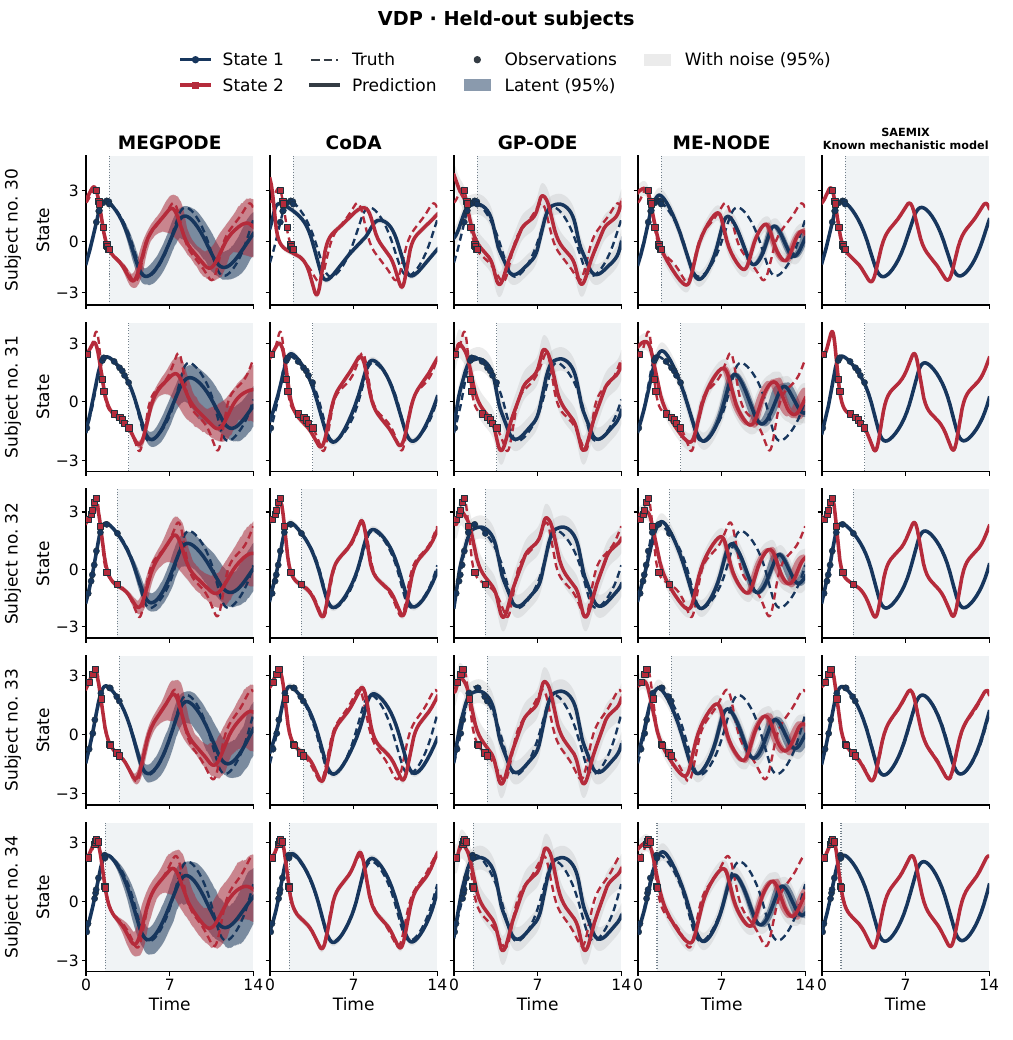}
\caption{\textbf{Van der Pol forecasting for new subjects.} The five displayed subjects are held out from population fitting on 30 training subjects. Forecasts are obtained after adapting each fitted model to a held-out subject. Only the eight prefix observations are available for adaptation; the pale region begins after their last observation. Dark blue and red denote the two state dimensions; dashed curves show truth, solid curves predictions, and markers fitting observations. Colored bands show pointwise 95\% posterior intervals (posterior ODE rollouts for MEGPODE and GP-ODE, retained latent trajectory samples for ME-NODE, conditional parameter draws for SAEMIX). Gray bands show 95\% predictive intervals including fitted independent observation noise; CoDA has only this band.}
\label{fig:vdp_fits_heldout}
\end{figure}
\clearpage

\subsection{Misspecified random effects: LV and LV-R}
\label{app:lv-saemix-comparison}
In our experiments, LV introduces heterogeneity through mechanistic parameters, whereas LV-R adds smooth subject residuals generated by random potentials. SAEMIX's parameter random effects cannot generally represent these residuals. This tests flexible subject-specific fields under a correct population ODE family but misspecified random effects, complementing the shared-field misspecification in Section~\ref{sec:vdp_residual_engineered}.

Figure~\ref{fig:lv-saemix-comparison} compares subject and population forecasts. For LV, \method{} captures distinct subject amplitudes and recovers the population trajectory, although SAEMIX forecasts subjects more accurately. Here subject heterogeneity follows the mechanistic parameterization used by SAEMIX, giving its random effects a suitable representation of the between-subject variation.

SAEMIX's timing errors are most pronounced for subject 18; \method{} follows the oscillations more closely, although its posterior mean attenuates later peaks. \method{} also has lower population forecast RMSE in this cohort, despite phase drift in both methods. Table~\ref{tab:prediction_results} reports aggregate forecast performance.

\begin{figure}[H]
\centering
\includegraphics[width=\linewidth]{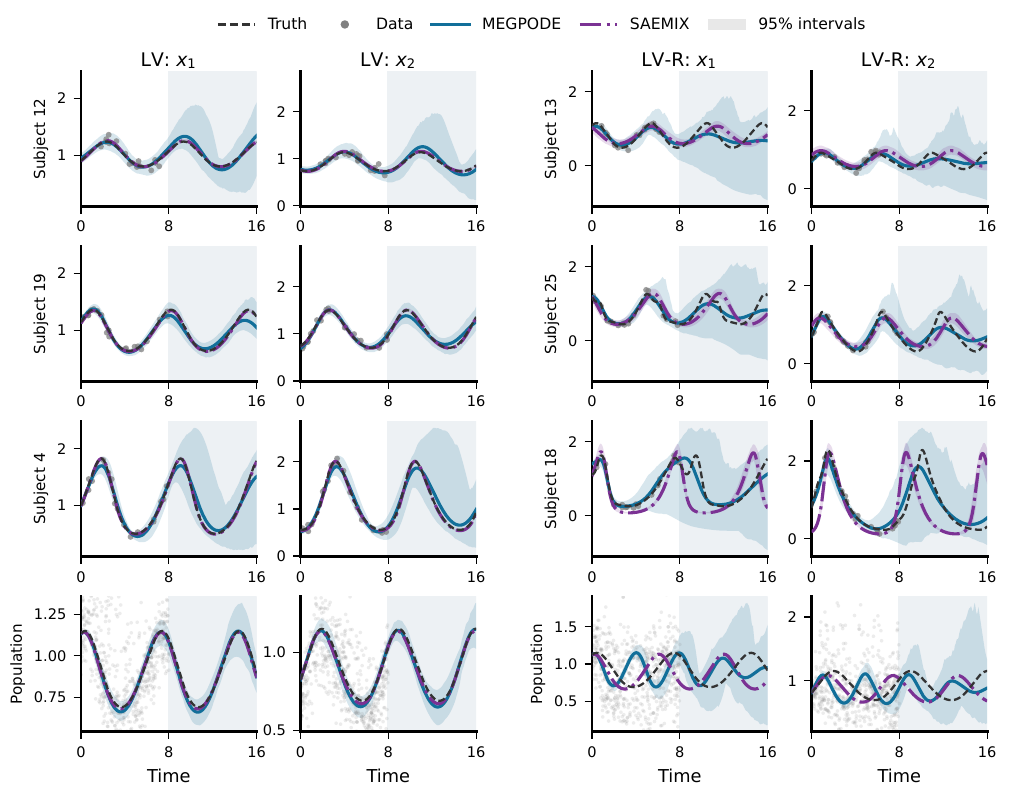}
\setlength{\abovecaptionskip}{2pt}
\caption{\textbf{Subject and population forecasts for LV and LV-R.} The first three rows show performance-selected subjects; the last shows the zero-effect population. Columns show LV (left) and LV-R (right). Dashed dark curves show truth, gray dots fitting observations, blue curves \method{}, and purple dash-dotted curves SAEMIX. Bands show pointwise 95\% posterior intervals; curves show posterior means. SAEMIX's population curve is a fixed-effects point estimate. Shading marks forecasting. Population observations come from heterogeneous subjects.}
\label{fig:lv-saemix-comparison}
\end{figure}
\clearpage

\clearpage
\subsection{Recovery of shared population dynamics}
\label{app:population-recovery}
\begingroup
\setlength{\abovedisplayskip}{2pt plus1pt minus1pt}
\setlength{\belowdisplayskip}{2pt plus1pt minus1pt}
\setlength{\abovedisplayshortskip}{0pt}
\setlength{\belowdisplayshortskip}{2pt}
\setlength{\jot}{1pt}
Accurate subject predictions alone do not establish recovery of the shared population dynamics. We therefore evaluate the learned population field and the trajectories it generates, with subject effects set to zero. These summaries describe population dynamics rather than the average trajectory of the heterogeneous cohort.

The target is the trajectory $\vx^\star(t)$ generated by the true shared field $\vf_0^\star$, with subject effects set to zero. For each fitted cohort, all estimated and true population trajectories start at the same fixed reference state $\vx_0^{\mathrm{ref}}$. We construct this reference coordinatewise from the cohort average of the earliest visible training observations. We score only the forecast window $t>T_{\mathrm{obs}}$, using the forecast fits and their original training observations.

\paragraph{Three recovery metrics.}
Write $\bar{\vf}_0(\vx)=\E_q[\vf_0(\vx)]$ for the posterior-mean population field, with the fitted deterministic field used for point-estimate methods. Its trajectory $\bar{\vx}(t)$ shows the dynamics generated by a single representative population law. The mean trajectory $\vm_x(t)$ instead averages ODE rollouts over the field posterior, giving the expected population evolution after propagating field uncertainty. We report both to distinguish interpretation of the fitted mean field from prediction under uncertainty:
\begin{equation}
\begin{aligned}
\dot{\vx}^\star(t)&=\vf_0^\star(\vx^\star(t)),
&\dot{\bar{\vx}}(t)&=\bar{\vf}_0(\bar{\vx}(t)),\\
\vm_x(t)&=\E_q[\vx(t;\vf_0,\vx_0^{\mathrm{ref}})],
&\vx^\star(0)&=\bar{\vx}(0)=\vx_0^{\mathrm{ref}}.
\end{aligned}
\label{eq:population-recovery-trajectories}
\end{equation}
The field draw is held fixed throughout each solve. We average over field uncertainty while holding the initial state fixed and omitting observation noise and subject effects. In general, $\bar{\vx}(t)\neq\vm_x(t)$ because solving the ODE is a nonlinear map of its vector field.

Let $\mathcal K_F$ contain the forecast indices of a uniform 240-point grid spanning the fitting and forecast windows, and let $D$ denote the state dimension. We report
\begin{equation}
\begin{aligned}
R_f&=\left[\tfrac{1}{D|\mathcal K_F|}\sum\nolimits_{k\in\mathcal K_F}
\left\|\bar{\vf}_0(\vx^\star(t_k))-\vf_0^\star(\vx^\star(t_k))\right\|^2\right]^{1/2},\\
R_{\mathrm{MF}}&=\left[\tfrac{1}{D|\mathcal K_F|}\sum\nolimits_{k\in\mathcal K_F}
\left\|\bar{\vx}(t_k)-\vx^\star(t_k)\right\|^2\right]^{1/2},\\
R_{\mathrm{MT}}&=\left[\tfrac{1}{D|\mathcal K_F|}\sum\nolimits_{k\in\mathcal K_F}
\left\|\E_q[\vx(t_k;\vf_0,\vx_0^{\mathrm{ref}})]-\vx^\star(t_k)\right\|^2\right]^{1/2}.
\end{aligned}
\label{eq:population-recovery-metrics}
\end{equation}
Thus $R_f$ compares fields at identical inputs along the true population forecast trajectory. This is a local field-recovery diagnostic on the nominal forecast path; it does not evaluate the field uniformly over a state-space region. We approximate $\vm_x$ by averaging ODE rollouts from the fixed reference state, using 128 shared-field draws for \method{} and 64 for GP-ODE. Population mean fields are evaluated directly from fitted parameters, including the analytic variational GP mean for GP-ODE.

\paragraph{Population fields across methods.}
For \method{}, the population field is the shared GP component, with all local coefficients set to zero. GP-ODE supplies its single shared GP field. CoDA supplies the empirical average of its fitted training-subject context fields, evaluated at the same state. SAEMIX supplies the mechanistic field at its fitted population parameters, with parameter random effects set to zero. CoDA and SAEMIX use point estimates of their population fields; their two trajectory summaries coincide. The shared--local decomposition and the empirical average of context fields define different population summaries, so this comparison assesses recovery under each method's stated convention. ME-NODE does not expose a shared field in the observed state coordinates with the prescribed fixed initial state and is excluded from these recovery metrics. Averaging its decoded subject predictions would retain subject-specific initial conditions and mixed effects. Appendix~\ref{app:concurrent_baselines} gives further details of the baseline models.

Figure~\ref{fig:population-recovery-combined} compares recovery across systems and illustrates the two \method{} trajectory summaries for LV, LV-R, FHN, VDP, PK and SIR. In this FHN example, $\vm_x(t)$ has lower forecast oscillation amplitude than $\bar{\vx}(t)$, consistent with phase averaging among shared-field draws with differing rollout phases or frequencies. Conversely, for the non-oscillatory PK and SIR systems, the two trajectory summaries remain close in these examples. Gray points reflect heterogeneous fields, initial conditions and observation noise, so their average need not follow the zero-effect population trajectory. The sole recovery target is the trajectory of $\vf_0^\star$.
\endgroup
\clearpage
\begin{figure}[H]
\centering
\includegraphics[width=\linewidth]{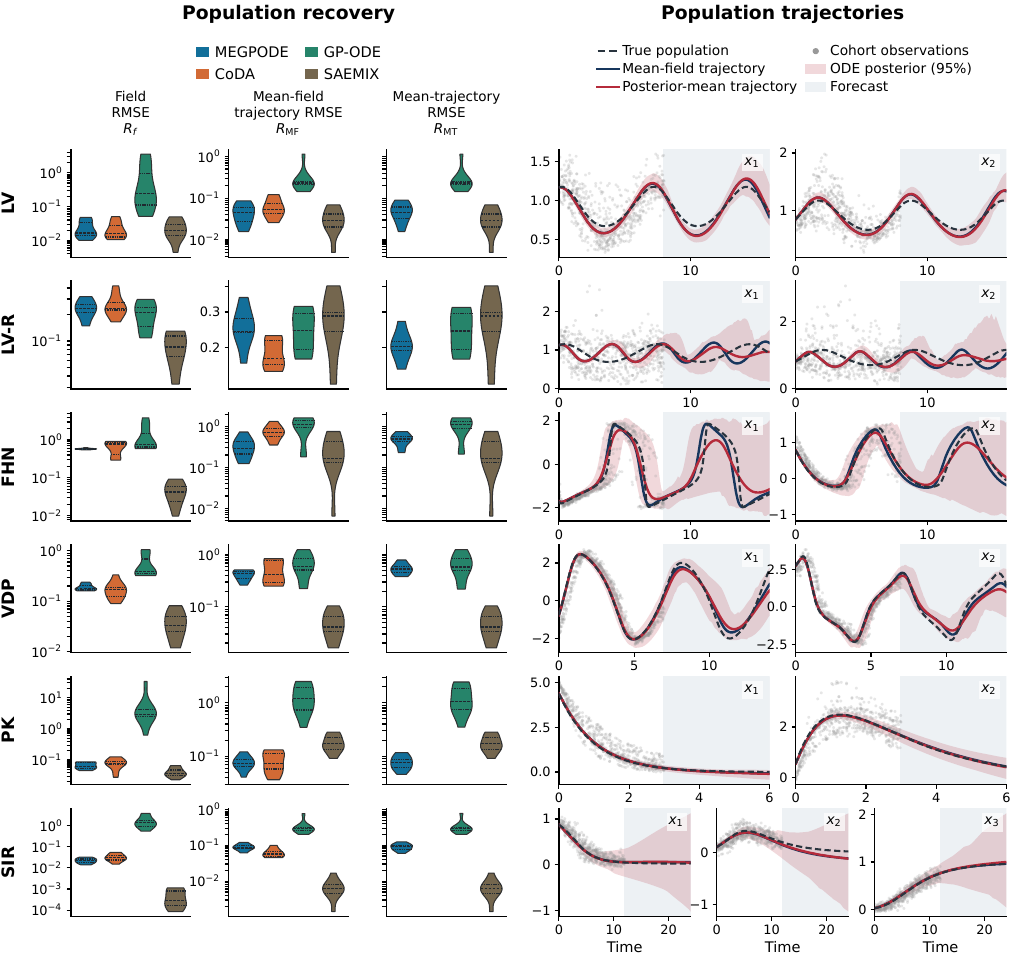}
\setlength{\abovecaptionskip}{2pt}
\setlength{\belowcaptionskip}{-3pt}
\caption{\textbf{Population recovery and dynamics.} Rows show LV, LV-R, FHN, VDP, PK and SIR. Left: forecast field error $R_f$ and trajectory errors $R_{\mathrm{MF}}$ and $R_{\mathrm{MT}}$ (Appendix~\ref{app:population-recovery}); violins summarize scores across seeds with quartiles marked. CoDA's deterministic field gives identical trajectory summaries and is omitted from $R_{\mathrm{MT}}$. Right: true (dashed), posterior-mean-field (blue), and posterior-mean (red) trajectories for \method{}; red bands show pointwise 95\% posterior intervals excluding observation noise. All state coordinates are shown. Within each row, curves start at the same fixed state with local effects set to zero. Gray points show cohort observations; shading marks forecasting.}
\label{fig:population-recovery-combined}
\label{fig:population-recovery-trajectories}
\end{figure}
\clearpage

\clearpage
\subsection{Fitted trajectory smoothing and posterior ODE rollouts}
\label{app:smoothing-rollouts}
Our main object of interest is the hierarchical vector field. During training, the auxiliary trajectory factors estimate subject states and derivatives without numerical ODE integration. The fitted fields can subsequently be integrated to obtain subject trajectory distributions, or population trajectories by setting local effects to zero; these ODE rollouts are the predictions shown in our results. Here we illustrate the distinction between auxiliary smoothing and trajectories generated by the fitted fields, and compare rollouts started at the initial or final smoothed state.

The Gaussian Markov factor $q_i$ in Eq.~\eqref{eq:variational} estimates subject $i$'s states and derivatives while fitting the vector field. Its state marginal at a collocation time is
\begin{equation}
q_i(\vx_{i,k})=\Normal(\vm_{x,i,k},\boldsymbol P_{xx,i,k}),
\label{eq:smoothing-state-marginal}
\end{equation}
where $\vx_{i,k}=\Hx\vs_{i,k}$. These marginals condition on all fitting observations and are available on the fitted grid $\mathcal C_i$. The ODE residual penalty connects them to the field posterior, while numerical integration generates trajectories that satisfy the sampled ODE.

For each coefficient draw, we sample $\vb_0^{(s)}\sim q_b(\vb_0)$ and $\vb_i^{(s)}\sim q_b(\vb_i\mid\vb_0^{(s)})$, preserving shared--local posterior dependence. Write $t_i^{\mathrm{last}}=\max_n t_{i,n}$ and $\vm_{x,i}^{\mathrm{last}}=\E_{q_i}[\vx_i(t_i^{\mathrm{last}})]$ for the smoothed state mean at the last observation time. We integrate the same sampled field from two fixed starting states:
\begin{equation}
\begin{aligned}
\dot{\vx}_i^{(s,a)}(t)&=\boldsymbol F_i(\vx_i^{(s,a)}(t))\vb^{(s)},\qquad a\in\{0,\mathrm{last}\},\\
\vx_i^{(s,0)}(\tau_{i,0})&=\vm_{x,i,0},\qquad
\vx_i^{(s,\mathrm{last})}(t_i^{\mathrm{last}})=\vm_{x,i}^{\mathrm{last}}.
\end{aligned}
\label{eq:fixed-start-rollout-comparison}
\end{equation}
The first rollout starts at the fitted initial-state mean; the second restarts at the smoothed state mean at the last observation time. This terminal mean uses all fitting observations and accounts for measurement noise through the fitted trajectory factor. Both rollouts use identical coefficient draws, kept fixed throughout each integration. Usual subject predictions also sample the fitted initial-state marginal. Here fixed starting means isolate field uncertainty.

We use 256 joint coefficient draws and summarize each ODE trajectory coordinate by its sample mean and empirical 2.5\% and 97.5\% quantiles. The smoothing bands use the exact Gaussian marginal quantiles from Eq.~\eqref{eq:smoothing-state-marginal}. All bands therefore show pointwise 95\% posterior intervals, excluding observation noise. The smoother ends at each subject's last fitting observation, and we do not extrapolate $q_i$.

Figure~\ref{fig:lv_smoothing_rollouts} compares three trajectory summaries: the auxiliary smoother, the ODE rollout from the initial-state mean, and the ODE rollout from the final smoothed mean. The smoothing and ODE means follow similar paths within the fitting window, while their uncertainty bands differ. Restarting from the terminal smoothing mean uses the inferred state to anchor the forecast and shortens the propagation horizon. Some forecast means can still depart from the truth despite close agreement during fitting. The band widths reflect both the starting state and the duration of uncertainty propagation. This comparison illustrates how auxiliary smoothing and ODE rollouts represent trajectory uncertainty under these different conditions.

\clearpage
\begin{figure}[H]
\centering
\includegraphics[width=\linewidth,height=.84\textheight,keepaspectratio]{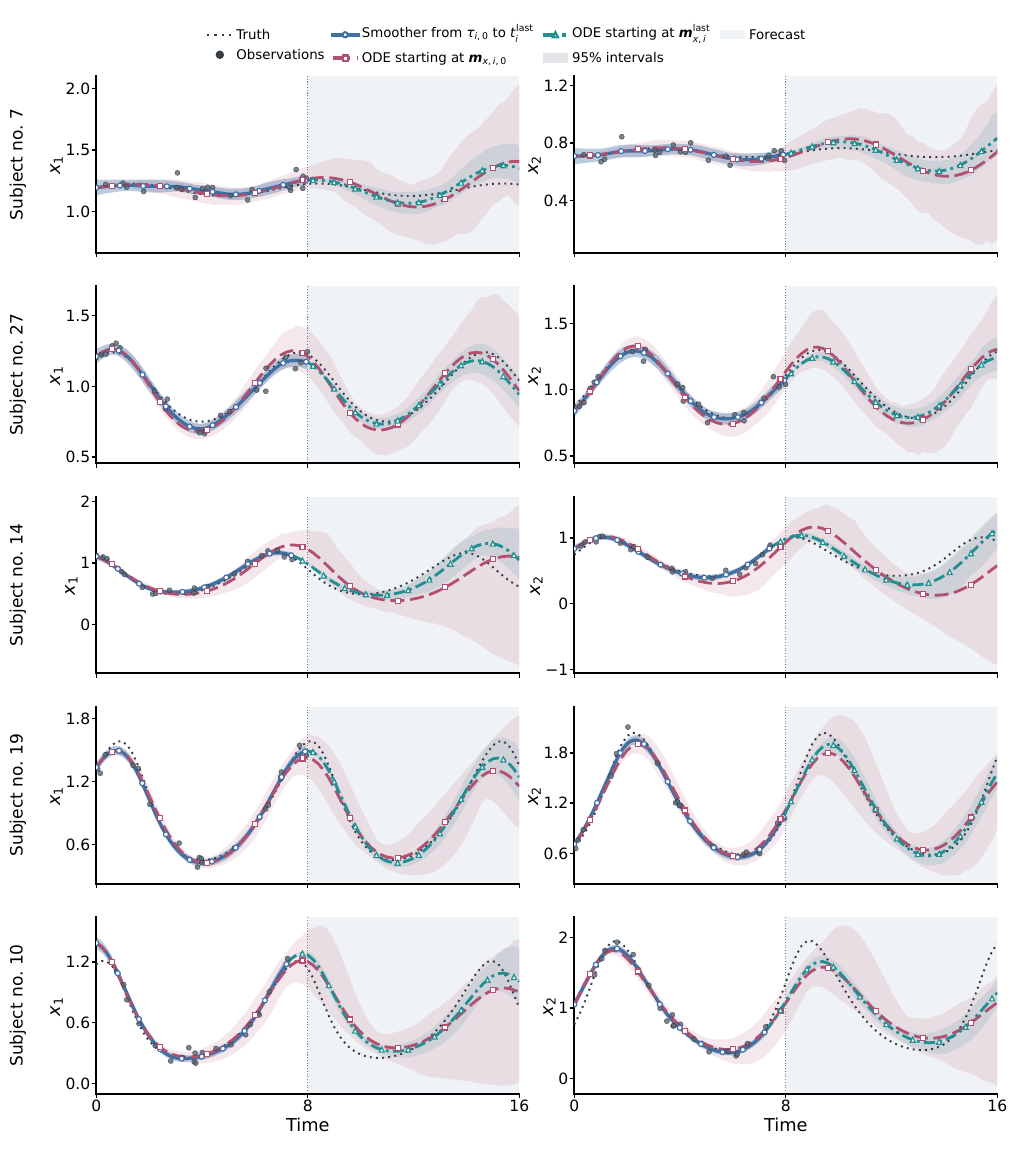}
\caption{\textbf{Auxiliary smoothing and fixed-start ODE rollouts on Lotka--Volterra.} Five subjects; columns show $x_1$ and $x_2$. Solid blue curves with open circles show smoothing from $\tau_{i,0}$ to $t_i^{\mathrm{last}}$; dashed rose curves with open squares show ODE rollouts starting at $\vm_{x,i,0}$; dash-dotted teal curves with open triangles show ODE rollouts starting at $\vm_{x,i}^{\mathrm{last}}$ at time $t_i^{\mathrm{last}}$. Both ODE summaries use identical 256 joint field draws and fixed starts. Bands show pointwise 95\% posterior intervals, using empirical quantiles for ODE rollouts and Gaussian marginal quantiles for smoothing, excluding observation noise. Dotted curves show truth, filled dots fitting measurements, and shading forecasting ($t>8$). Smoothing ends at each subject's last observation.}
\label{fig:lv_smoothing_rollouts}
\end{figure}
\clearpage

\clearpage
\subsection{Local-field contributions to subject dynamics}
\label{app:local-dynamics-contributions}
We illustrate how shared and local fields contribute to subject dynamics, distinguishing their mean dynamical effects from conditional trajectory uncertainty. Figure~\ref{fig:local-dynamics-contributions} shows three subjects per system from separately fitted cohorts.

\paragraph{Population overview.}
Panel (a) shows all observations from the first displayed subject's cohort and shared-field trajectories with zero local effect. The shared mean field is integrated from $\vx_0^{\mathrm{pop}}=M^{-1}\sum_j\vm_{x,j,0}$, with truth from the same origin. The blue band gives pointwise 95\% intervals from 128 shared coefficient draws, holding this origin fixed and excluding observation noise.

\paragraph{Paired trajectories.}
Sample $\vb_0^{(s)}\sim q_b(\vb_0)$ and $\vb_i^{(s)}\sim q_b(\vb_i\mid\vb_0^{(s)})$ using Eq.~\eqref{eq:arrowhead}. The corresponding fields generate paired trajectories:
\begin{equation}
\begin{aligned}
\dot\vx_i^{(s)}(t)&=\vf_0^{(s)}(\vx_i^{(s)}(t))+\vg_i^{(s)}(\vx_i^{(s)}(t)),\\
\dot\vx_{i,\mathrm{sh}}^{(s)}(t)&=\vf_0^{(s)}(\vx_{i,\mathrm{sh}}^{(s)}(t)),\\
\vx_i^{(s)}(0)&=\vx_{i,\mathrm{sh}}^{(s)}(0)=\vm_{x,i,0}.
\end{aligned}
\label{eq:paired-local-rollouts}
\end{equation}
Both use the same shared draw and fixed posterior-mean origin, without refitting. Panel (b) shows sample means and pointwise 95\% intervals. Averaging sampled-field trajectories generally differs from integrating the mean field. Intervals isolate coefficient uncertainty, excluding observation noise and initial-state and hyperparameter uncertainty.

\paragraph{Velocity contributions at common inputs.}
Integrating $\bar\vf_0+\bar\vg_i$, where $\bar\vf_0=\E_{q_b}[\vf_0]$ and $\bar\vg_i=\E_{q_b}[\vg_i]$, from $\vm_{x,i,0}$ gives a reference path $\bar\vx_i$. Panel (c) evaluates both components, their sum, and truth at these common inputs. Gaussian intervals use marginal field covariances, including shared-mediated local uncertainty through $\vSigma_{b,i}=\boldsymbol C_i+\boldsymbol G_i\vSigma_{b,0}\boldsymbol G_i^\top$ in Eq.~\eqref{eq:joint-covariances}, without propagating state uncertainty.

\paragraph{Integrated effect of the local field.}
For each pair in Eq.~\eqref{eq:paired-local-rollouts}, the state shift $\boldsymbol\delta_i^{(s)}=\vx_i^{(s)}-\vx_{i,\mathrm{sh}}^{(s)}$ satisfies
\begin{equation}
\begin{aligned}
\boldsymbol\delta_i^{\mathrm{loc},(s)}(t)
&=\int_0^t\vg_i^{(s)}(\vx_i^{(s)}(u))\,\mathrm du,\\
\boldsymbol\delta_i^{\mathrm{sh},(s)}(t)
&=\int_0^t\left[\vf_0^{(s)}(\vx_i^{(s)}(u))-\vf_0^{(s)}(\vx_{i,\mathrm{sh}}^{(s)}(u))\right]\,\mathrm du,\\
\boldsymbol\delta_i^{(s)}(t)
&=\boldsymbol\delta_i^{\mathrm{loc},(s)}(t)+\boldsymbol\delta_i^{\mathrm{sh},(s)}(t).
\end{aligned}
\label{eq:local-state-shift-decomposition}
\end{equation}
The first term accumulates local velocity; the second records the shared field's response to the altered path. Panel (d) shows their sample means and a 95\% interval for their sum. The identity retains nonlinear changes in shared-field inputs; its terms can have opposite signs.

\paragraph{Relation to conditional uncertainty.}
Panel (e) estimates $\rho^{X,\mathrm{loc}}_{i,d}(t)$ from Eq.~\eqref{eq:trajectory_local_fraction} using Eq.~\eqref{eq:nested-variance-estimators} in Appendix~\ref{app:theory-cov}. Drawing $R=128$ shared vectors and $L=32$ conditional local vectors per shared draw gives 4096 pairs, preserving posterior dependence. With fixed origins, the fraction attributes trajectory variance to local coefficients conditional on shared coefficients, separately from their mean dynamical effect.

We use 481 grid points and RK4 with two substeps per interval, accumulating both integrals at the same stages to preserve their sum up to roundoff. Halving the step on 16 draws per subject changed states by less than $2\times10^{-8}$. Green dotted curves use true fields with the same origins.

\paragraph{Interpretation.}
Across the displayed cases, the shared-field response opposes much of the integrated local contribution. Substantial local velocities can thus produce smaller net shifts by changing where the shared field is evaluated. High trajectory local fractions can coexist with small net shifts, distinguishing uncertainty attribution from the mean dynamical effect.

\clearpage
\begin{figure}[H]
\centering
\includegraphics[width=\linewidth,height=.78\textheight,keepaspectratio]{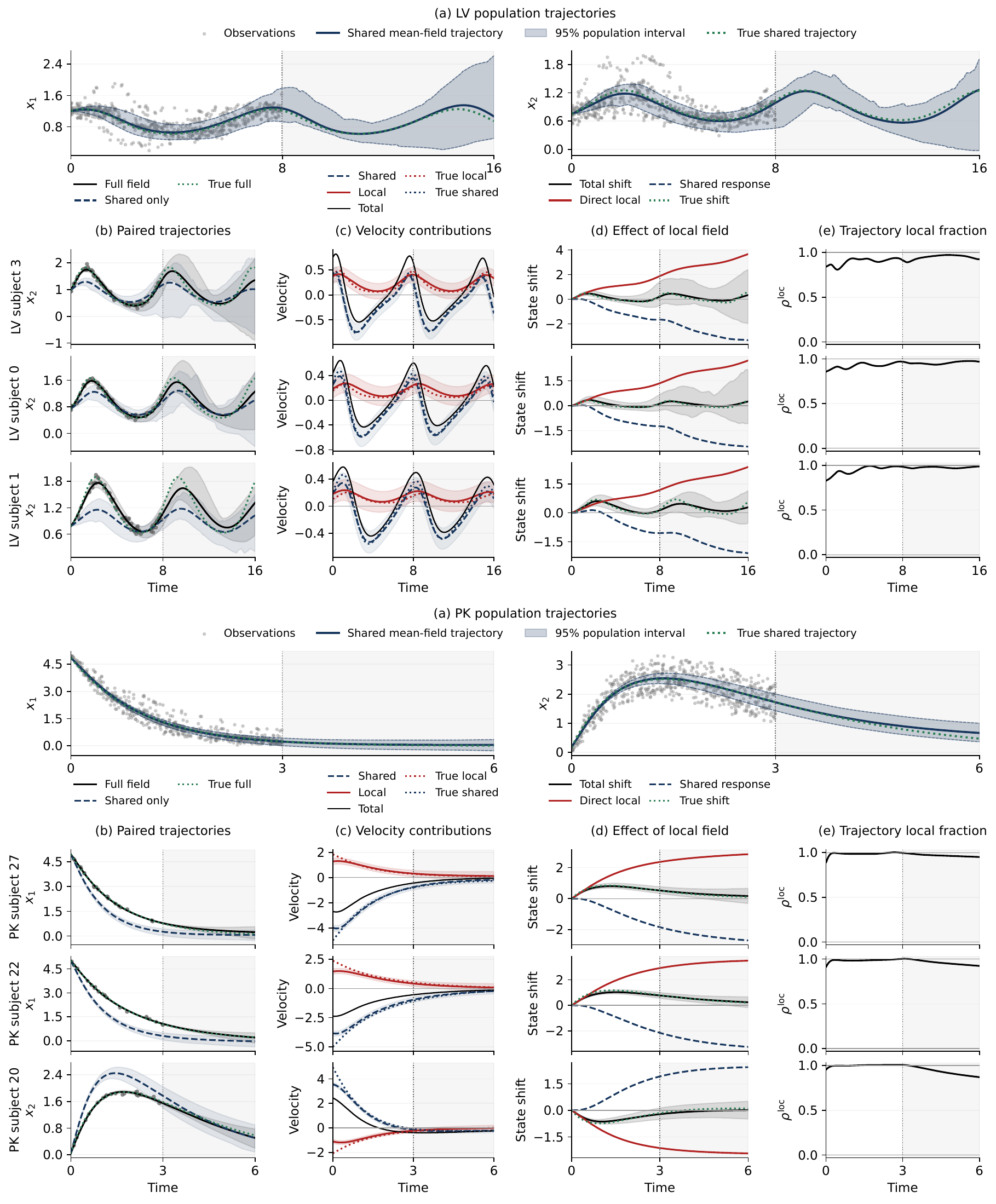}
\caption{\textbf{Shared and local contributions to subject dynamics on LV and PK.} (a) Complete-cohort observations for both coordinates, with shared-field trajectories from a common origin and blue population uncertainty bands. (b) Mean paired full-field and shared-only trajectories for three truth-selected subjects from separate cohorts per system. (c) Shared, local, and total mean velocities at common inputs; dotted blue and red curves show truth. (d) Paired state shifts and their integrated local and shared-response components. (e) Conditional local trajectory variance fractions. Green dotted curves show true trajectories and shifts; gray points are observations and shaded time intervals indicate forecasting. Bands are pointwise 95\% intervals from shared coefficient draws in (a), marginal Gaussian field intervals in (c), and paired trajectory draws in (b) and (d). All origins are fixed.}
\label{fig:local-dynamics-contributions}
\end{figure}
\clearpage

\begingroup
\setlength{\intextsep}{1pt}
\setlength{\belowcaptionskip}{0pt}
\begin{table}[H]
\centering
\small
\setlength{\tabcolsep}{2pt}
\renewcommand{\arraystretch}{.9}
\resizebox{\linewidth}{!}{%
\begin{tabular}{@{}lcccccccc@{}}
\toprule
& \multicolumn{2}{c}{\bfseries Population dynamics} & \multicolumn{6}{c}{\bfseries Subject forecasting} \\
\cmidrule(lr){2-3}\cmidrule(lr){4-9}
Method & \shortstack{Rollout\\RMSE} & \shortstack{Field\\RMSE} & RMSE & MAE & NLPD & CRPS & \shortstack{Cov. err.\\(\%)} & \shortstack{Cal. err.\\(\%)} \\
\midrule
\method{} & $1.81\pm1.02$ & $\underline{0.64\pm0.27}$ & $1.84\pm0.95$ & $1.55\pm0.84$ & $\underline{11.71\pm16.96}$ & $1.19\pm0.65$ & $\underline{35.91\pm27.34}$ & $\underline{35.03\pm23.10}$ \\
\methodm{} & $\mathbf{0.79\pm0.42}$ & $\mathbf{0.14\pm0.06}$ & $\mathbf{0.75\pm0.35}$ & $\mathbf{0.57\pm0.28}$ & $\mathbf{1.20\pm1.72}$ & $\mathbf{0.42\pm0.21}$ & $\mathbf{18.45\pm17.53}$ & $\mathbf{23.77\pm16.12}$ \\
SAEMIX & $\underline{1.54\pm0.99}$ & $1.88\pm4.18$ & $\underline{1.54\pm0.87}$ & $\underline{1.23\pm0.70}$ & $(6.95\pm11.78)\times10^3$ & $\underline{1.16\pm0.65}$ & $80.43\pm13.77$ & $67.73\pm11.42$ \\
\bottomrule
\end{tabular}}
\setlength{\abovecaptionskip}{2pt}
\setlength{\belowcaptionskip}{0pt}
\caption{\textbf{Additive-residual Van der Pol results.} Means $\pm$ sample standard deviations across 20 matched cohorts; lower is better. Bold and underline mark the best and second-best means. Population: population-field errors along the true forecast trajectory and rollout errors from a common initial state with zero local effects. Subjects: dense latent-state forecast scores for training subjects. All scores use $t\in(5,24]$. Cov. err. and Cal. err. denote 95\% coverage error and calibration error in percentage points, respectively (Appendix~\ref{app:evaluation-metrics}).}
\label{tab:vdp-additive-results}
\end{table}

\begin{figure}[H]
    \centering
    \includegraphics[width=\linewidth]{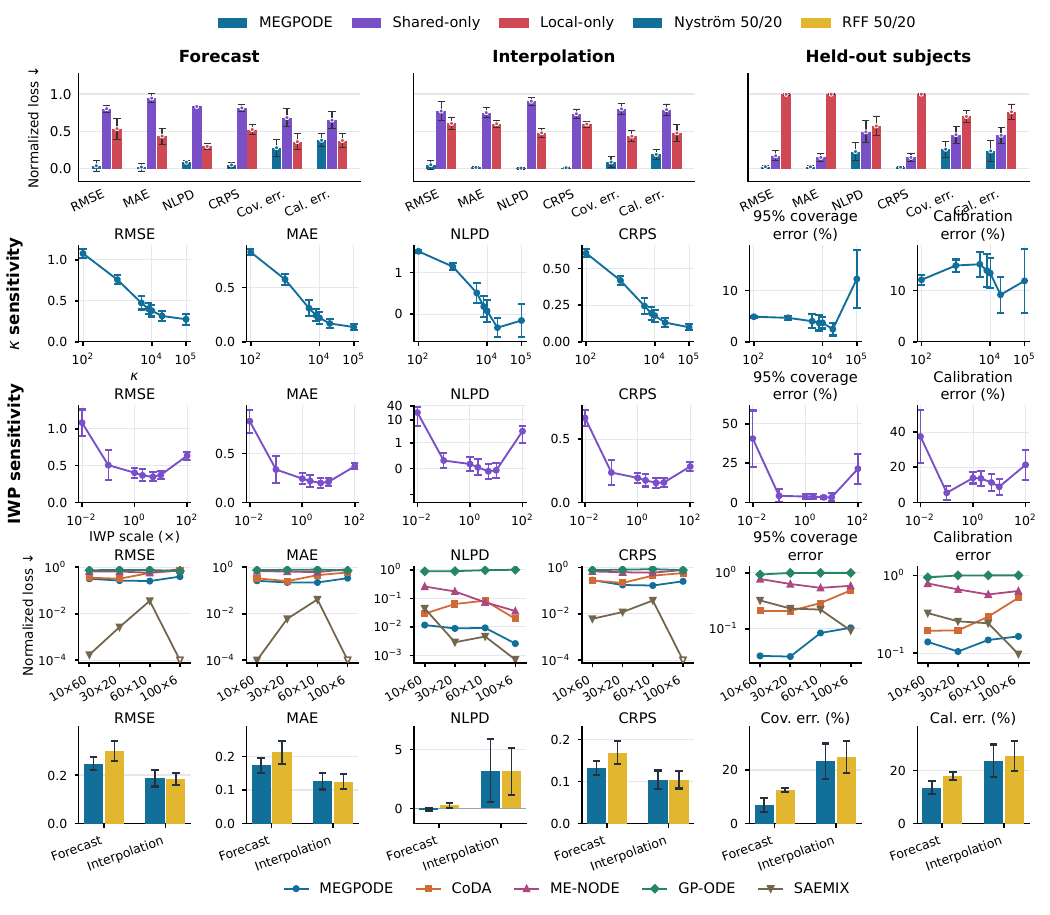}
    \setlength{\abovecaptionskip}{2pt}
    \caption{\textbf{Complete ablation results.} Row 1: shared/local normalized losses, averaged equally over six systems per seed; mean $\pm$ SD across seeds. Rows 2--3: FHN forecast sensitivity to $\kappa$ and multiples of the observation-fitted IWP variance; mean $\pm$ SD across ten paired cohorts. Each sweep fixes the other quantity at its reference ($\kappa=8000$, IWP multiplier one), whose fit is shared. Row 4: sampling-budget forecast losses, normalized across five methods evaluated on the same data, averaged within systems, then equally across systems. Row 5: Nystr\"om versus RFF forecasting and interpolation, both with 50 shared and 20 local features; mean $\pm$ SD across seeds, averaging equally over systems per seed. Coverage and calibration errors are in percentage points.}
    \label{fig:ablation-sensitivity-full}
    \label{fig:feature-representation-comparison}
\end{figure}
\endgroup

\subsection{Runtime comparison}
\label{app:runtime-comparison}

Table~\ref{tab:runtime-comparison} reports wall-clock fitting and adaptation times for the two-dimensional Lotka--Volterra benchmark, using the ten seeds from the prediction comparisons. Baseline optimization settings are given in Appendix~\ref{sec:baselines}.

Forecasting and interpolation timings cover cohort fitting. Held-out timings cover additional adaptation from the fitted forecast model, using eight context observations for each of five new subjects; we average these five times within each seed before computing the mean and sample standard deviation across seeds.

\begin{table}[H]
\centering
\begin{tabular}{lrrr}
\toprule
Method & Forecast fit (min) & Interpolation fit (min) & Held-out adaptation (s/subject) \\
\midrule
\method{} & $23.3 \pm 1.0$ & $24.2 \pm 0.9$ & $60.6 \pm 8.3$ \\
CoDA & $4.9 \pm 1.8$ & $28.7 \pm 26.4$ & $32.6 \pm 7.5$ \\
GP-ODE & $63.0 \pm 50.4$ & $167.5 \pm 134.3$ & $59.6 \pm 39.7$ \\
ME-NODE & $98.1 \pm 4.9$ & $98.7 \pm 3.4$ & $30.2 \pm 6.3$ \\
SAEMIX & $24.9 \pm 1.5$ & $24.0 \pm 1.2$ & $54.2 \pm 4.8$ \\
\bottomrule
\end{tabular}

\caption{\textbf{Lotka--Volterra runtimes.} Means $\pm$ sample standard deviations across ten seeds. Cohort fitting is measured in minutes; held-out adaptation is measured in seconds per subject and excludes the preceding cohort fit.}
\label{tab:runtime-comparison}
\end{table}

The experiments ran on a compute cluster equipped with NVIDIA V100 GPUs.

\clearpage

\section{Baselines and hyperparameters}
\label{sec:baselines}

\subsection{Concurrent baselines}\label{app:concurrent_baselines}

\paragraph{CoDA context-conditioned neural ODE.}

We adapt the authors' official CoDA implementation~\citep{kirchmeyer2022generalizing}\footnote{\url{https://github.com/yuan-yin/CoDA}} to irregular, noisy and partially missing observations. In its original setting, different environments share a neural dynamical model but are allowed to induce environment-specific vector fields through learned context variables. We adapt this view to our population setting by treating each training subject as one environment. Each subject has a learned context code \(\vc_i\in\mathbb{R}^{d_c}\) and a learned initial condition \(\vx_{i,0}\), and evolves according to
\[
    \dot{\vx}_i(t)
    =
    f_{\Theta,\psi}(\vx_i(t);\vc_i),
    \qquad
    \vx_i(t_{i,0})=\vx_{i,0}.
\]
The vector field is represented by a shared base MLP whose weights, biases and learned Swish slope parameters are additively perturbed by a hypernetwork applied to the context code:
\[
    f_{\Theta,\psi}(\vx;\vc)
    =
    \mathrm{MLP}_{\Theta + \Delta_\psi(\vc)}(\vx),
    \qquad
    \Delta_\psi(\vc)=A_\psi \vc.
\]
Thus \(\Theta\) defines the shared base neural field, while \(\Delta_\psi(\vc_i)\) defines a subject-specific deformation. The population field used for recovery is defined below by averaging the fitted subject fields. Observations are modeled as
\[
    \vy_{i,n}
    =
    \vH_i \vx_i(t_{i,n})
    +
    \boldsymbol{\varepsilon}_{i,n},
    \qquad
    \boldsymbol{\varepsilon}_{i,n}\sim\Normal(0,\sigma_y^2\Id).
\]

Training uses the released implementation's row-wise hypernetwork regularizer and summed context penalty. Let $\mathcal O=\{(i,n,d): y_{i,n,d}\text{ is observed}\}$ index the recorded measurements and let $\widetilde{\vx}_i$ denote the trajectory produced by scheduled teacher forcing. The minimized objective is
\[
\mathcal J_{\mathrm{CoDA}}
=\frac{1}{|\mathcal O|}\sum_{(i,n,d)\in\mathcal O}
\bigl(y_{i,n,d}-[\vH_i\widetilde{\vx}_i(t_{i,n})]_d\bigr)^2
+\lambda_c\sum_i\|\vc_i\|_2^2
+\lambda_h\sum_j\|(A_\psi)_{j,:}\|_2
+\frac{\lambda_x}{MD}\sum_i\|\vx_{i,0}-\bar{\vx}_0\|_2^2.
\]
Here $\bar{\vx}_0$ is the empirical mean of the learned initial conditions. Initial states are optimized jointly with context codes rather than fixed to noisy initial measurements. The last term penalizes deviations from their cohort mean. The hidden activation is $\phi_b(x)=\frac{x}{1.1}\sigma\bigl(x\log(1+e^b)\bigr)$, where $\sigma(u)=(1+e^{-u})^{-1}$ is the logistic sigmoid. Each hidden layer has one learned scalar $b$, shared across its units, to which the hypernetwork adds a subject-specific offset from $A_\psi\vc_i$.

For partially observed states, teacher-forcing resets replace only measured coordinates, retaining predictions for missing coordinates. Integration uses fourth-order Runge--Kutta with the 3/8 rule and one step between consecutive observation times for each subject.

Training uses Adam with global gradient clipping at 1, default linear-layer weight initialization, zero biases and zero initial context codes. Each subject supplies one observed trajectory, and training uses the full cohort with a 7,000-epoch cap. Early stopping uses free-running training MSE, with patience 200 after at least 300 epochs. After the best checkpoint is restored, the scalar observation-noise standard deviation is estimated from free-running training RMSE, with a $10^{-3}$ floor.

For held-out-subject adaptation, the shared network and hypernetwork are frozen. The reported held-out runs use 1,000 Adam steps at learning rate $10^{-3}$ to fit free-running trajectories. Only a new context code and initial condition are optimized from the available context observations:
\[
(\hat{\vc}_\star,\hat{\vx}_{\star,0})
=
\arg\min_{\vc,\vx_0}
\frac{1}{|\mathcal O_\star|}\sum_{(n,d)\in\mathcal O_\star}
\left\|
y_{\star,n,d}
-
[\vH_\star \vx_\star(t_{\star,n};\Theta,\psi,\vc,\vx_0)]_d
\right\|_2^2
+
\lambda_c\|\vc\|_2^2
+
\frac{\lambda_x}{D}\|\vx_0-\bar{\vx}_0\|_2^2 .
\]
Hence CoDA can personalize dynamics through a low-dimensional context variable, but it does not explicitly decompose the field into a shared GP component plus a subject-specific random-effect field. For population-level summaries in our evaluation, the CoDA population field is therefore defined as the empirical average of the learned subject-specific context fields,
\[
    f^{\mathrm{pop}}_{\mathrm{CoDA}}(\vx)
    =
    \frac{1}{M}\sum_{i=1}^M
    f_{\Theta,\psi}(\vx;\vc_i).
\]
Since the mapping from context code to vector field acts through weight perturbations, averaging the resulting vector fields is more faithful than evaluating the model at the empirical mean context code.

The population trajectory is obtained by integrating this averaged field from the common training-observation reference initial condition,
\[
    \dot{\vx}^{\mathrm{pop}}(t)
    =
    f^{\mathrm{pop}}_{\mathrm{CoDA}}(\vx^{\mathrm{pop}}(t)),
    \qquad
    \vx^{\mathrm{pop}}(0)
    =
    \vx^{\mathrm{ref}}_0.
\]
The fitted vector fields and trajectories are deterministic. Predictive scores use the fitted scalar Gaussian observation noise, with no posterior uncertainty over vector fields or subject and population trajectories.

\paragraph{Mixed-effect neural ODE.}

The mixed-effect neural ODE baseline is based on~\citep{nazarovs2022menode}\footnote{\url{https://github.com/JurijsNazarovs/panel_me_ode}}. This model represents each subject through a latent trajectory \(\vz_i(t)\in\mathbb{R}^{d_z}\), decoded into observation space by a neural decoder \(D_\omega\):
\[
    \vy_{i,n}
    =
    D_\omega(\vz_i(t_{i,n}))
    +
    \boldsymbol{\varepsilon}_{i,n},
    \qquad
    \boldsymbol{\varepsilon}_{i,n}\sim\Normal(0,\sigma_y^2\Id).
\]
The initial latent state is inferred amortizedly from the subject's observed time series using an encoder,
\[
    q_\phi(\vz_{i,0}\mid \{(t_{i,n},\vy_{i,n})\}_{n\in\mathcal{C}_i}),
\]
implemented either as an ODE-RNN or an RNN encoder. Subject heterogeneity enters the latent vector field through mixed-effect weights. Concretely, the latent drift is written as
\[
    \dot{\vz}_i(t)
    =
    G_\theta(\vz_i(t))\,\vw_i,
\]
where \(G_\theta\) is a shared neural network and the subject-specific mixed-effect weights are sampled from a diagonal Gaussian distribution
\[
    \vw_i
    =
    \boldsymbol{\mu}_w
    +
    \boldsymbol{\sigma}_w\odot\boldsymbol{\epsilon}_i,
    \qquad
    \boldsymbol{\epsilon}_i\sim\Normal(\mathbf{0},\Id).
\]
Thus the model has a neural population component through \(G_\theta\) and \(\boldsymbol{\mu}_w\), together with random effects through \(\boldsymbol{\sigma}_w\).

Training uses the released implementation's selected-sample objective. For subject $i$, draw candidates $(\vz_{i,0}^{(s)},\vw_i^{(s)})$ and select $s_i^*$ by the smallest squared error on observed coordinates. Let $r_{z,i}^{(s)}$ and $r_{w,i}^{(s)}$ be the coordinate-averaged log-density ratios $\log q-\log p$ for the initial state and mixed effects, respectively. Let $\ell_i^{(s)}$ be the Gaussian log likelihood averaged first over visible times within each coordinate and then over coordinates; an entirely missing coordinate contributes zero. The implemented loss is
\[
\mathcal J_{\mathrm{ME\text{-}NODE}}
=\frac{1}{M}\sum_i
\left\{\beta\bigl(r_{z,i}^{(s_i^*)}+r_{w,i}^{(s_i^*)}\bigr)
-\ell_i^{(s_i^*)}\right\},
\]
with the upstream KL-annealing schedule for $\beta$. Observation masks are applied both when selecting a candidate and when evaluating its likelihood.

The inferred initial state is anchored at the first visible observation $t_{i,1}$. Training and prediction use the same normalized time
\[
\tau_i(t)=\frac{t-t_{i,1}}{s_i},\qquad
s_i=\max\{T_{\mathrm{obs}}-t_{i,1},\ t_{i,N_i}-t_{i,1},\ 10^{-6}\}.
\]
Here $T_{\mathrm{obs}}$ is the end of the benchmark observation window, and $N_i$ is the number of visible observation times for subject $i$. The latent drift above is integrated in these model-time units. The prediction horizon does not alter $s_i$. Queries before the first observation use backward transport from this same inferred state. Encoding retains each subject's native Euler schedule. Decoding batches equal-length subject grids by interval index using their own step lengths and the upstream 3/8 RK4 tableau; no union-grid observations are introduced. Training makes one optimizer update per cohort-averaged objective.

For held-out subjects, there is no separate optimization of new parameters. Instead, the learned encoder and mixed-effect distribution are reused: the context observations of the new subject are passed through the encoder, candidate mixed-effect samples are drawn, and the context-compatible samples determine the predictive distribution. Thus adaptation is amortized rather than obtained by a subject-specific optimization loop.

At evaluation time, candidate samples are drawn for each subject \(i\):
\[
    \vz_{i,0}^{(s)} \sim q_\phi(\vz_{i,0}\mid \vy_i),
    \qquad
    \vw_i^{(s)} \sim q(\vw_i),
    \qquad s=1,\dots,S_{\mathrm{cand}}.
\]
The candidates are ranked by masked reconstruction error on the available context observations, and \(S\) context-compatible samples are retained. Each retained sample is propagated through the latent ODE,
\[
    \dot{\vz}_i^{(s)}(t)
    =
    G_\theta(\vz_i^{(s)}(t))\vw_i^{(s)},
    \qquad
    \hat{\vx}_i^{(s)}(t)
    =
    D_\omega(\vz_i^{(s)}(t)).
\]
Evaluation uses one initial-state draw per subject and conditioning call, shared across the mixed-effect candidates. The retained trajectories form a reconstruction-selected ensemble rather than unweighted posterior samples. Its empirical mean and coordinatewise variance are
\[
    \hat{\boldsymbol{\mu}}_i(t)
    =
    \frac{1}{S}\sum_{s=1}^S \hat{\vx}_i^{(s)}(t),
    \qquad
    \hat{\boldsymbol{\sigma}}_{i,\mathrm{post}}^2(t)
    =
    \frac{1}{S}\sum_{s=1}^S
    \left(\hat{\vx}_i^{(s)}(t)-\hat{\boldsymbol{\mu}}_i(t)\right)^2.
\]
Latent trajectory scores use this ensemble variance. Observation-prediction bands additionally use a scalar observation-noise estimate from training reconstruction RMSE (with a $10^{-3}$ floor); the training likelihood itself uses fixed standard deviation $0.1$,
\[
    \hat{\boldsymbol{\sigma}}_{i,\mathrm{pred}}^2(t)
    =
    \hat{\boldsymbol{\sigma}}_{i,\mathrm{post}}^2(t)
    +
    \hat{\sigma}_y^2.
\]

This baseline does not expose a shared vector field in the observation coordinates and is excluded from the shared-field recovery metric. The shared neural parameters remain point estimates; the predictive ensemble varies the selected latent initial states and mixed-effect weights.

\paragraph{GPODE variational multiple shooting with shared field and subject-specific initial conditions.}
As a non-mixed-effect GP-ODE comparator, we adapt the variational multiple-shooting model of Hegde et al.~\citep{hegde2022vmsgp}\footnote{\url{https://github.com/hegdepashupati/gaussian-process-odes}} to our heterogeneous observation setting. The model retains its single shared latent vector field, subject-specific initial conditions and subject-specific auxiliary shooting states; the model supports multiple trajectories under a shared field. Concretely, the dynamics follow
\[
\dot{\vx}_i(t)=\vf(\vx_i(t)),
\qquad
\vx_i(t_{i,0})=\vx_{i,0},
\qquad i=1,\dots,M,
\]
under one common GP prior and one shared variational posterior for \(\vf\), with subject-specific variational factors for \(\vx_{i,0}\) and for the shooting states. Hence this baseline can exploit multiple initial conditions and long trajectories, but it has no explicit mixed-effect decomposition: all inter-subject variation must be absorbed through the subject-specific initial conditions, the subject-specific shooting states, and posterior uncertainty in the single shared field.

Let $t_{i,0}$ denote the initial-state time and $t_{i,1},\dots,t_{i,N_i}$ the retained observation times. The model introduces auxiliary shooting states at the first $N_i-1$ observation times:
\[
\vS_i := \bigl(\vs_{i,1},\dots,\vs_{i,N_i-1}\bigr),
\qquad
\vs_{i,0} := \vx_{i,0},
\qquad i=1,\dots,M.
\]
All subjects share the same latent vector field \(\vf\) and the same inducing variables \(\boldsymbol U\).
The shooting transition is
\[
p\!\left(\vs_{i,n}\mid \vs_{i,n-1},\vf\right)
=
\Normal\!\left(
\vs_{i,n}\;\middle|\;
\vx\!\left(t_{i,n};\,\vs_{i,n-1},\vf\right),
\;\sigma_\xi^2 \Id_D
\right),
\qquad n=1,\dots,N_i-1,
\]
where \(\vx(t_{i,n};\,\vs_{i,n-1},\vf)\) denotes the ODE solution at time \(t_{i,n}\)
obtained by integrating \(\dot{\vx}(t)=\vf(\vx(t))\) from initial state \(\vs_{i,n-1}\) over
\([t_{i,n-1},t_{i,n}]\). The observation likelihood is written similarly as
\[
p\!\left(\vy_{i,n}\mid \vs_{i,n-1},\vf\right)
=
p\!\left(\vy_{i,n}\mid \vx\!\left(t_{i,n};\,\vs_{i,n-1},\vf\right)\right).
\]

The variational family is
\[
q\!\left(\{\vS_i\}_{i=1}^M,\{\vx_{i,0}\}_{i=1}^M,\vf,\boldsymbol U\right)
=
p(\vf\mid \boldsymbol U)\,q(\boldsymbol U)\,
\prod_{i=1}^M
\left[
q(\vx_{i,0})
\prod_{n=1}^{N_i-1} q(\vs_{i,n})
\right].
\]
The state factors are Gaussian, independent across subjects and shooting times, with full covariance within each state. Each observation likelihood includes only visible coordinates. The resulting population ELBO is
\begin{align*}
\mathcal{L}_{\mathrm{GPODE\text{-}pop}}
&=
\sum_{i=1}^M
\Bigg[
\sum_{n=1}^{N_i}
\mathbb{E}_{q(\vs_{i,n-1})\,q(\vf)}
\Bigl[
\log p\!\left(\vy_{i,n}\mid \vs_{i,n-1},\vf\right)
\Bigr]
\notag\\
&\qquad\quad
+
\sum_{n=1}^{N_i-1}
\mathbb{E}_{q(\vs_{i,n-1})\,q(\vs_{i,n})\,q(\vf)}
\Bigl[
\log p\!\left(\vs_{i,n}\mid \vs_{i,n-1},\vf\right)
\Bigr]
\notag\\
&\qquad\quad
-
\sum_{n=1}^{N_i-1}
\mathbb{E}_{q(\vs_{i,n})}
\bigl[\log q(\vs_{i,n})\bigr]
-
\mathrm{KL}\!\left(q(\vx_{i,0})\,\|\,p(\vx_{i,0})\right)
\Bigg]
-
\mathrm{KL}\!\left(q(\boldsymbol U)\,\|\,p(\boldsymbol U)\right).
\label{eq:gpode-pop-elbo}
\end{align*}

The training objective sums the subject-wise likelihood and shooting terms under this common field, together with a single shared inducing-variable KL term. In our implementation, the whole objective is divided by the number of visible scalar observations.

Our adapter retains each subject's native observation times, masks missing coordinates, and pads unequal sequence lengths only for batching. Padded entries contribute no likelihood, shooting-continuity or entropy terms. Each segment has duration $\Delta_{i,n}=t_{i,n}-t_{i,n-1}$ and is integrated on normalized time $u\in[0,1]$ through $d\vx/du=\Delta_{i,n}\vf(\vx)$. Thus the first segment maps the initial state to the first observation, and subsequent segments start from the corresponding auxiliary shooting states. The adapter requires identity observations but permits coordinate-wise missingness.

Inducing values are initialized from finite differences over actual within-subject intervals. Missing coordinates are interpolated only for initialization of latent states and inducing locations; these filled values never enter the observation likelihood. The initial-state mean starts from the first visible state estimate and remains trainable. Trained-subject predictions use 64 draws from both the learned initial-state posterior and the GP field posterior. For held-out subjects in our evaluation, the fitted field posterior is fixed and a point initial state is optimized using 300 Adam steps at learning rate $5\times10^{-3}$. The adaptation loss is masked context MSE plus $10^{-3}\|\vx_0-\bar{\vx}_0\|^2/D$, where $\bar{\vx}_0$ averages the fitted training initial-state means; subsequent prediction samples the field posterior with this initial state fixed.

For shared-field recovery in our evaluation, 64 posterior field draws are propagated from the prescribed fixed reference initial state $\vx^{\mathrm{ref}}_0$ (Appendix~\ref{app:evaluation-metrics}) and their trajectory mean is scored. This retains field uncertainty while holding the initial condition fixed; the mean of these trajectories need not equal the trajectory of the posterior-mean field.

\paragraph{Mechanistic NLME-ODE.}

For benchmarks whose mechanistic family is known, the classical nonlinear mixed-effects ODE baseline~\citep{lavielle2014mixedeffects, comets2017parameter} assumes that each subject follows a parametric ODE
\[
\dot{\vx}_i(t)=f\!\left(\vx_i(t);\boldsymbol{\theta}_i\right),
\qquad
\boldsymbol{\theta}_i=\boldsymbol{\theta}+\boldsymbol{\eta}_i,
\qquad
\boldsymbol{\eta}_i\sim\mathcal{N}(\boldsymbol{0},\boldsymbol{\Omega}),
\]
with subject-specific initial condition
\[
\vx_i(t_{i,0})=\vx_{i,0},
\qquad
\vx_{i,0}=\vx_0+\boldsymbol{\xi}_i,
\qquad
\boldsymbol{\xi}_i\sim\mathcal{N}(\boldsymbol{0},\boldsymbol{\Omega}_{x_0}),
\qquad i=1,\dots,M,
\]
where \(\boldsymbol{\theta}\) collects shared population parameters, \(\boldsymbol{\eta}_i\) denotes subject-specific random effects on the mechanistic parameters, and \(\vx_0\) is the population initial condition. In practice, random effects act only on a subset of mechanistic parameters, with the remaining components fixed across subjects. Observations are modeled as
\[
\vy_{i,n} = \vH_i \vx_i(t_{i,n}) + \boldsymbol{\varepsilon}_{i,n},
\qquad
\boldsymbol{\varepsilon}_{i,n}\sim\mathcal{N}(0,\sigma_y^2 \Id).
\]

On correctly specified mechanistic systems, the random effects are placed on the same parameters used to generate heterogeneity in the benchmark.

This model is fitted using the R package \texttt{saemix}~\citep{comets2017parameter}. Its stochastic approximation EM procedure samples subject effects conditionally on the observations, updates sufficient statistics by stochastic approximation, and estimates population parameters, random-effect covariance, and residual noise. Observation noise is initialized from training observations using the local-linear residual heuristic defined in Appendix~\ref{subsec:hyperparameter_learning}, also used by MEGPODE and GP-ODE. Subject-level uncertainty is computed from conditional parameter draws returned by \texttt{saemix}; held-out adaptation keeps the fitted population model fixed and conditions on the new subject's available observations.

This baseline provides a standard parametric NLME comparator when the mechanistic family is known, but it is structurally misspecified on benchmarks where heterogeneity is injected through residual subject-specific vector fields rather than through parameter random effects.

For comparability with our population-level summaries, the corresponding population trajectory is defined by integrating the learned mechanistic population dynamics from the common fixed reference initial condition:
\[
\dot{\vx}^{\mathrm{pop}}(t)
=
f\!\left(\vx^{\mathrm{pop}}(t);\hat{\boldsymbol{\theta}}\right),
\qquad
\vx^{\mathrm{pop}}(0)=\vx^{\mathrm{ref}}_0.
\]

\subsection{Baselines derived from \method}
\label{app:baselines}
The shared-only and local-only baselines are introduced in Section~\ref{sec:ablations} and compared with full \method{} in Figure~\ref{fig:shared-local-ablation}. These ablations restrict the field decomposition in Eq.~\eqref{eq:global-field} while minimizing the same objective in Eq.~\eqref{eq:objective}, with the same trajectory prior, temporal weights, and observation likelihood. In the updates below, $\boldsymbol F_i(\vx)$ is restricted to the columns for $\vb_0$ in the shared-only variant and for $\vb_i$ in the local-only variant.
Full \method{} uses 50 shared and 20 local features per output. To retain the same total feature rank per subject, each separately tuned baseline uses 70 features per output in its retained component.

\paragraph{Shared field only.}
Set $\vg_i=0$ and retain only $\vb_0$. Eq.~\eqref{eq:coefficient-update} reduces to
\[
\begin{aligned}
\vLambda_0&=\Id_{Dl_0}+\sum_{i,k}\kappa\omega_{i,k}\E_{q_i}[\boldsymbol F_i(\vx)^\top\boldsymbol F_i(\vx)],\\
\veta_0&=\sum_{i,k}\kappa\omega_{i,k}\E_{q_i}[\boldsymbol F_i(\vx)^\top\dot{\vx}],
\end{aligned}
\]
with $q_b(\vb_0)=\Normal(\vLambda_0^{-1}\veta_0,\vLambda_0^{-1})$. All subjects share one dynamical law, with individually inferred initial states.

\paragraph{Local fields only.}
Set $\vf_0=0$ and retain only the local coefficients. Each subject has an independent Gaussian coefficient optimum:
\[
\begin{aligned}
\vLambda_i&=\lambda_r\Id_{Dl_r}+\sum_k\kappa\omega_{i,k}\E_{q_i}[\boldsymbol F_i(\vx)^\top\boldsymbol F_i(\vx)],\\
\veta_i&=\sum_k\kappa\omega_{i,k}\E_{q_i}[\boldsymbol F_i(\vx)^\top\dot{\vx}],
\end{aligned}
\]
with $q_b(\vb_i)=\Normal(\vLambda_i^{-1}\veta_i,\vLambda_i^{-1})$. Each subject has its own dynamical law, with no common-field information transfer.

Both variants use the trajectory updates of Section~\ref{sec:trajectory-optimization}, with field means and variances computed from the retained coefficients, and the posterior ODE prediction procedure of Appendix~\ref{app:seen-subject-extrapolation}.

\subsection{Hyperparameters}\label{sec:hyperparameter-details}

\begin{longtable}{@{}p{0.34\linewidth}p{0.60\linewidth}@{}}
\caption{Hyperparameter settings for CODA, ME-NODE, and GP-ODE}
\label{tab:baseline-hyperparameters}\\
\toprule
\textbf{Hyperparameter} & \textbf{Value} \\
\midrule
\endfirsthead
\toprule
\textbf{Hyperparameter} & \textbf{Value} \\
\midrule
\endhead
\bottomrule
\endfoot
\multicolumn{2}{@{}l}{\textbf{CODA}} \\ \midrule
State dimension & benchmark latent dimension $D$ \\
Base MLP widths & $D \rightarrow 64 \rightarrow 64 \rightarrow 64 \rightarrow D$ \\
Hidden dimension & $64$ \\
Hidden activation & learned, context-modulated Swish divided by $1.1$ \\
Output scale factor & $1.0$ \\
Context code dimension & $2$ \\
Context codes & one learned code per training patient \\
Hypernetwork & linear map from context code to additive offsets for all base MLP weights, biases and Swish slope parameters \\
Base weight initialization & default linear-layer uniform initialization \\
Base bias initialization & zeros \\
Optimizer & Adam \\
Learning rate & $10^{-3}$ \\
Max training epochs & $7000$ \\
Early stopping & free-running training MSE; patience 200, minimum 300 epochs \\
Initial-condition initialization & first observed value per subject and dimension\\
Context-code regularization & $10^{-4}$ \\
Hypernetwork row L1/L2 regularization & $10^{-6}$ \\
Initial-condition regularization & $10^{-3}$ \\
ODE solver & RK4\\
RK4 substeps per interval & $1$ \\
Held-out adaptation variables & new context code and new initial condition \\
Held-out adaptation epochs & $1000$ \\
Held-out adaptation learning rate & same as training learning rate, $10^{-3}$ \\
Population field & average of context-conditioned fields over training context codes \\

\midrule
\multicolumn{2}{@{}l}{\textbf{ME-NODE}} \\ \midrule
Latent dimension & $13$ \\
Recognition dimension & $20$ \\
Recognition layers & $1$ \\
GRU units & $100$ \\
Initial-state encoder & ODE-RNN \\
Encoder input & observations concatenated with observation masks \\
Encoder ODE solver & Euler \\
Encoder ODE tolerances & rtol $10^{-3}$, atol $10^{-4}$ \\
Generative ODE hidden layers & $2$ \\
Generative ODE hidden units & $100$ \\
Generative ODE activation & Tanh \\
Generative ODE widths & $13 \rightarrow 100 \rightarrow 100 \rightarrow 13$ \\
Mixed-effect dimension & $1$ \\
Mixed-effect layer & learned fixed effect plus diagonal random-effect std. \\
Decoder & linear map $13 \rightarrow D$ \\
Initial-state prior & standard normal \\
Observation std. in training likelihood & $0.1$ \\
Optimizer & Adamax \\
Learning rate & $10^{-3}$ \\
Learning-rate scheduler & exponential decay \\
Scheduler decay & $\gamma=0.999$ \\
Training epochs & $3000$ \\
Gradient clipping & global norm $1.0$ \\
KL annealing epochs & $100$ \\
KL annealing start epoch & $1$ \\
Initial-state samples $n_{z_0}$ & $1$ \\
Mixed-effect samples $n_w$ & $20$ \\
Evaluation samples & $10$ \\
Generative ODE solver & RK4 \\
Generative ODE tolerances & rtol $10^{-3}$, atol $10^{-4}$ \\
Held-out adaptation & none; prediction conditions through encoder context \\
Observation-prediction std. & $\sqrt{\hat\sigma_{i,\mathrm{post}}^2(t)+\hat\sigma_y^2}$, with $\hat\sigma_y=\max(10^{-3},\mathrm{train\ RMSE})$ \\

\midrule
\multicolumn{2}{@{}l}{\textbf{GP-ODE}} \\ \midrule\\
Kernel & RBF \\
Kernel parameterization & output-dimension-wise lengthscales and variances \\
Initial RBF lengthscale & $1.3$ \\
Initial RBF variance & $0.5$ \\
Number of inducing points & $16$ \\
Inducing initialization & $k$-means centers on training states \\
Inducing value initialization & empirical finite-difference gradients \\
Random Fourier features & $256$ \\
Inducing posterior covariance & full covariance\\
Training grid & native subject times; masked padding for unequal lengths \\
Likelihood & Gaussian \\
Likelihood variance & learned by default \\
State posterior & factorized Gaussian shooting-state variational posterior \\
State-continuity constraint & Gaussian \\
Constraint initial scale & $10^{-3}$ \\
Constraint trainable & false \\
Optimizer & Adam \\
Learning rate & $5\times 10^{-3}$ \\
Training iterations & $5000$ \\
Monte Carlo samples in objective & $5$ \\
ODE solver & \texttt{dopri5} via \texttt{torchdiffeq} \\
Evaluation samples & $64$ \\
Held-out adaptation variables & initial condition only \\
Held-out adaptation epochs & $300$ \\
Held-out adaptation learning rate & same as training learning rate, $5\times 10^{-3}$ \\
Held-out initial-condition penalty & $10^{-3}$ toward population initial condition \\
Observation-prediction variance & posterior trajectory variance plus likelihood noise variance \\

\end{longtable}

\clearpage
\begin{longtable}{@{}p{0.35\linewidth}p{0.61\linewidth}@{}}
\caption{Reference configuration for \method{} inference.}\label{tab:megpode-hyperparameters}\\
\toprule Parameter & Value\\\midrule\endfirsthead
\toprule Parameter & Value\\\midrule\endhead\bottomrule\endfoot
Trajectory prior & IWP with fixed proper initial $\Normal(0,100\Id_{2D})$\\
IWP diffusion scale & Estimated from observation-only marginal likelihood before fitting; fixed thereafter\\
Field features & RBF/Nystr\"om, $l_0=50$, $l_r=20$, fixed inducing locations\\
Initial lengthscales & $\ell_{0,d}=0.8s_d$, $\ell_{r,d}=s_d$, $s_d=\operatorname{sd}(y_{\cdot,d})+10^{-3}$\\
Field scales & Initialized from training finite differences\\
Coefficient priors & $\vb_{0,d}\sim\Normal(\mathbf0,\Id_{l_0})$, $\vb_{i,d}\sim\Normal(\mathbf0,\Id_{l_r})$ ($\lambda_r=1$)\\
New-subject local coefficient-prior precision & $\lambda_{r,\mathrm{adapt}}=10$\\
Collocation grid & $\mathcal C_i$: measurement times $t_{i,n}$, endpoints $\tau_{i,0},\tau_{i,K_i}$, and 480 uniform points; duplicate times merged\\
Constraint & $\kappa=8000$, normalized trapezoidal weights\\
Trajectory optimizer & Natural sites, up to 40 damped steps per outer iteration\\
Coefficient optimizer & Exact joint arrowhead solve from expected trajectory statistics\\
Outer budget & 200 iterations\\
Learned hyperparameters & Shared/local kernel lengthscales and field amplitudes; diagonal observation noise\\
Hyperparameter optimization & After 5\% warm-up, up to 10 L-BFGS-B iterations per update\\
Fixed quantities & Nystr\"om locations $\boldsymbol Z_0,\boldsymbol Z_r$, IWP diffusion scale, and $\kappa$\\
Forecast distribution & Joint coefficient ODE pushforward, 64 samples\\
Initial state for prediction & Fitted marginal $q_i(\vx_{i,0})$, sampled independently of coefficients\\
Integration for posterior trajectories & RK4, four substeps per simulation-grid interval\\
\end{longtable}

\section{Additional experimental details}
\label{app:systems}

\subsection{Systems}
\label{app:systems-subsection}

We list the benchmark systems considered in Section~\ref{sec:systems}.
All synthetic benchmarks are autonomous first-order ODEs with full-state observations
\begin{equation*}
    \vy_{i,n}
    =
    \vx_i(t_{i,n})
    +
    \boldsymbol{\epsilon}_{i,n},
    \qquad
    \boldsymbol{\epsilon}_{i,n}\sim \Normal(\mathbf{0},\sigma_y^2 \Id_D),
\end{equation*}
where the observation times are irregular and obtained by sorting i.i.d. draws from
$\mathcal{U}([0,T])$. Dense ground-truth trajectories are generated numerically on a fine grid.
For the six main benchmark systems, we use the same simulation budget:
\[
M=30,\qquad
N_{\mathrm{obs}}=20,\qquad
\sigma_y=0.05.
\]
The common value $\sigma_y=0.05$ specifies equal absolute noise, not equal signal-to-noise ratios across systems or state coordinates.
Each subject has its own randomized initial condition and the full latent state is observed with additive Gaussian noise.
Figure~\ref{fig:examples} displays, for representative cohorts, the population and patient trajectories together with the observation horizon, forecasting horizon, and interpolation window for all synthetic systems.

\paragraph{Lotka--Volterra (2D).}
For prey $x_{1,i}$ and predator $x_{2,i}$,
\begin{align*}
    \dot{x}_{1,i} &= \alpha_0 x_{1,i} - \beta_i x_{1,i}x_{2,i},\\
    \dot{x}_{2,i} &= \delta_i x_{1,i}x_{2,i} - \gamma_0 x_{2,i}.
\end{align*}
The subject-specific random effects act on the interaction coefficients,
\[
    \beta_i = \beta_0 \exp(\beta_{\mathrm{sd}}\xi_i),
    \qquad
    \delta_i = \delta_0 \exp(\delta_{\mathrm{sd}}\zeta_i),
    \qquad
    \xi_i,\zeta_i \sim \Normal(0,1).
\]
Thus the population field is determined by
$(\alpha_0,\beta_0,\delta_0,\gamma_0)$, while subject heterogeneity enters
through $(\beta_i,\delta_i)$.

\paragraph{Van der Pol (2D).}
For $\vx_i=(x_{1,i},x_{2,i})^\top$,
\begin{align*}
    \dot{x}_{1,i} &= x_{2,i},\\
    \dot{x}_{2,i} &= \mu_i(1-x_{1,i}^2)x_{2,i}-x_{1,i}.
\end{align*}
Subject heterogeneity is introduced by Gaussian effects on the parameter,
\[
    \mu_i=\mu_0\bigl(1+\mu_{\mathrm{sd}}\xi_i\bigr),
    \qquad
    \xi_i\sim \Normal(0,1).
\]
or its logarithm, $\mu_i=\mu_0\exp(\mu_{\mathrm{sd}}\xi_i)$. The same subject-specific $\mu_i$ defines the fitting and forecast trajectories, with identical ground truth for every method.

\paragraph{FitzHugh--Nagumo (2D).}
For membrane potential $v_i$ and recovery variable $w_i$,
\begin{align*}
    \dot{v}_i
    &=
    s_0\!\left(v_i - \frac{v_i^3}{3} - c_i w_i + I_0\right),\\
    \dot{w}_i
    &=
    s_0 \frac{d_i v_i + a_0 - b_0 w_i}{\tau_0}.
\label{eq:fhn}
\end{align*}
The subject-specific random effects act on the coupling coefficients,
\[
    c_i = c_0 + c_{\mathrm{sd}}\xi_i,
    \qquad
    d_i = d_0 + d_{\mathrm{sd}}\zeta_i,
    \qquad
    \xi_i,\zeta_i \sim \Normal(0,1).
\]

\paragraph{One-compartment pharmacokinetics (2D).}
With gastrointestinal amount $A_{\mathrm{GI},i}$ and plasma concentration $U_{P,i}$,
\begin{align*}
    \dot{A}_{\mathrm{GI},i} &= -k_{a,i} A_{\mathrm{GI},i},\\
    \dot{U}_{P,i} &= \frac{k_{a,i}}{V_i}A_{\mathrm{GI},i} - k_{e,i}U_{P,i}.
\end{align*}
Here heterogeneity is parametric and log-normal:
\[
    \log k_{a,i} = \log k_{a,0} + \sigma_{k_a}\xi_i,\qquad
    \log k_{e,i} = \log k_{e,0} + \sigma_{k_e}\zeta_i,\qquad
    \log V_i = \log V_0 + \sigma_V \rho_i,
\]
with $\xi_i,\zeta_i,\rho_i\sim \Normal(0,1)$.
The initial condition is fixed to a common oral dose,
$A_{\mathrm{GI},i}(0)=\mathrm{dose}_0$ and $U_{P,i}(0)=0$.

\paragraph{SIR (3D).}
For susceptible, infectious, and removed states $(S_i,I_i,R_i)$,
\begin{align*}
    \dot{S}_i &= -\beta_i \frac{S_i I_i}{N_i},\\
    \dot{I}_i &= \beta_i \frac{S_i I_i}{N_i} - \gamma_i I_i,\\
    \dot{R}_i &= \gamma_i I_i,
\end{align*}
where $N_i=S_i+I_i+R_i$. The subject-specific effects are log-normal:
\[
    \beta_i=\beta_0\exp(\sigma_\beta \xi_i),
    \qquad
    \gamma_i=\gamma_0\exp(\sigma_\gamma \zeta_i),
    \qquad
    \xi_i,\zeta_i\sim \Normal(0,1).
\]

\paragraph{Lotka--Volterra with GP-draw random effects (2D).}
This benchmark keeps the population dynamics exactly Lotka--Volterra and introduces smooth subject-specific residual fields through random Hamiltonian potentials:
\[
    \dot{\vx}_i(t)
    =
    \vf_{\mathrm{LV}}(\vx_i(t);\theta_0)
    +
    \vg_i(\vx_i(t)),
    \qquad
    \theta_0=(\alpha_0,\beta_0,\delta_0,\gamma_0)=(0.9,1,1,0.9).
\]
For each subject, we draw a scalar finite-rank GP potential and obtain the residual field from its gradient:
\[
\begin{aligned}
    h_i(\vx)&=a(\vx)B_r(\vx)\vw_i,
    \qquad \vw_i\sim\Normal(\mathbf{0},\Id),\\
    \vg_i(\vx)&=\sigma_g x_1x_2\boldsymbol J\nabla h_i(\vx),
    \qquad \boldsymbol J=\begin{pmatrix}0&-1\\1&0\end{pmatrix}.
\end{aligned}
\]
We use $\sigma_g=0.4$ and $B_r(\vx)=k_r(\vx,\mathcal Z_r)L^{-\top}$, where $\mathcal Z_r$ is a $7\times7$ grid on $[0,2]^2$, $K_{ZZ}+10^{-6}I=LL^\top$, and $k_r$ is a unit-variance RBF kernel with lengthscales $(0.6,0.45)$. The gate
\[
    a(\vx)
    =
    \exp\!\left[
    -\frac12\left(\frac{x_1-1.0}{0.85}\right)^2
    -\frac12\left(\frac{x_2-0.9}{0.65}\right)^2
    \right]
\]
localizes the potential around the observed state region. The resulting field preserves the subject-specific Hamiltonian $H_i(\vx)=H_0(\vx)+\sigma_g h_i(\vx)$, where $H_0(\vx)=x_1-0.9\log x_1+x_2-0.9\log x_2$. Subject draws change orbit shapes, amplitudes, and periods, while $\mathbb E[\vg_i(\vx)]=\mathbf 0$ at each fixed state, retaining the mechanistic population field. Trajectories are integrated in logarithmic state coordinates from initial states drawn around $(1.15,0.75)^\top$ with standard deviation $0.1$.

\begin{figure}
    \centering
    \includegraphics[width=1\linewidth]{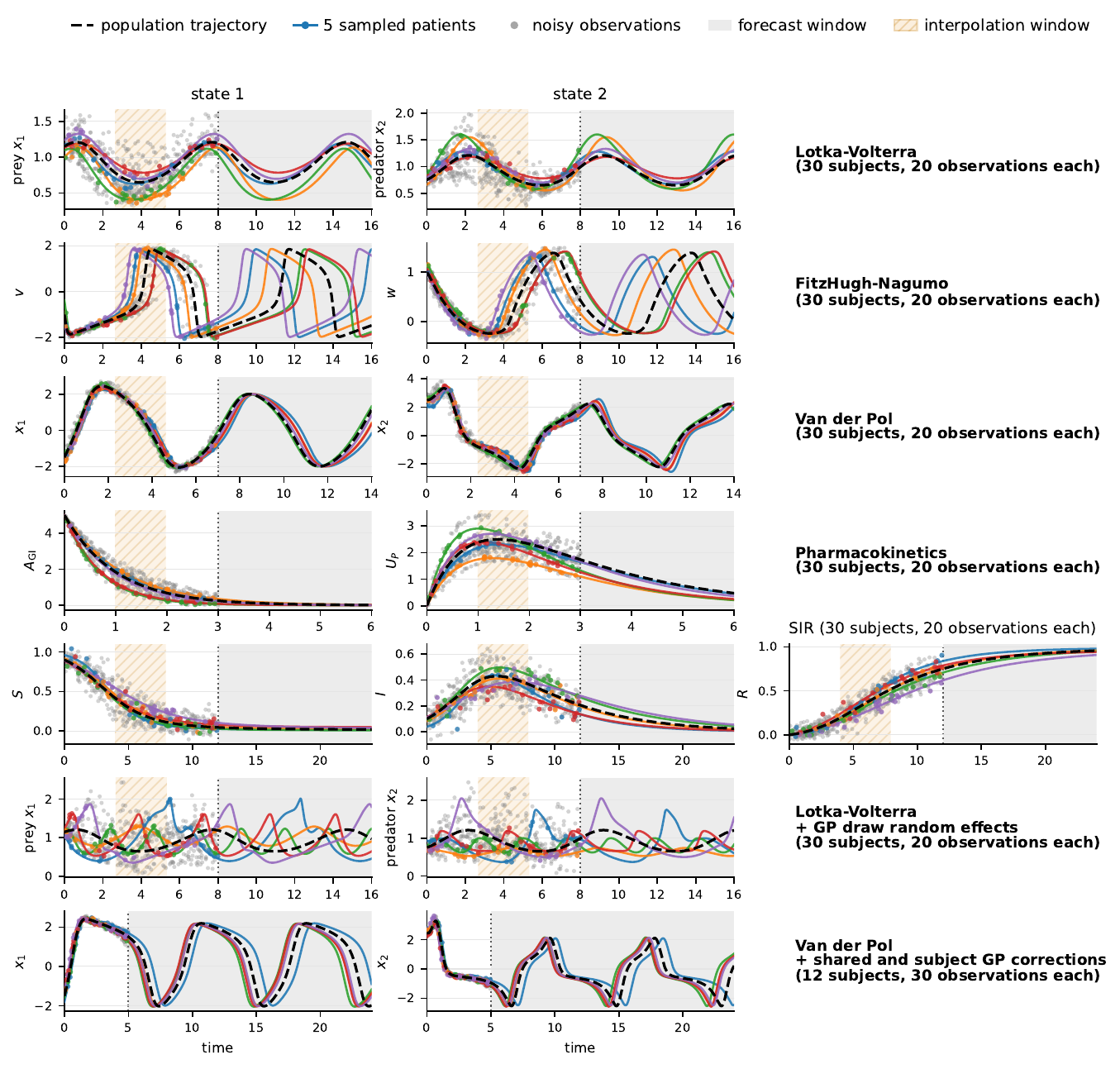}
\caption{
Synthetic benchmark examples. Dashed black curves denote the population trajectory, obtained by integrating the population-only vector field from a single initial condition. Colored curves show five sampled subject trajectories, gray points are noisy observations, and shaded regions indicate interpolation and forecast windows. The main benchmarks use 30 subjects with 20 observations each; the misspecified VDP experiment uses 12 subjects with 30 observations each. Its population field includes the shared GP correction, and its forecast horizon extends to $t=24$.
}
    \label{fig:examples}
\end{figure}

\subsection{Van der Pol additive-residual experiment}
\label{app:vdp_residual_engineered}
For $\vx=(x_1,x_2)^\top$, the mechanistic skeleton is
\[
\vf_{\mathrm{VDP}}(\vx;\mu)
=\bigl(x_2,\,\mu(1-x_1^2)x_2-x_1\bigr)^\top.
\]
The generator uses $\mu^\star=0.8$ and additive fields $\vf_i^\star=\vf_{\mathrm{VDP}}(\cdot;\mu^\star)+\boldsymbol h_0+\vg_i^\star$. The shared correction $\boldsymbol h_0$ is one fixed finite-rank RBF Gaussian draw, and the subject residuals $\vg_i^\star$ are independent draws on a $7\times7$ feature grid. Their scales are $0.36$ and $0.04$, respectively, with kernel lengthscales $(1.1,1.3)$. We generate 12 subjects with initial states perturbed around $(-1.5,2.5)^\top$ by Gaussian noise of standard deviation $0.2$. Each subject has 30 irregular observations over $[0,5]$ with observation-noise standard deviation $0.05$. The truth is propagated to $t=24$ for forecasting.

The mechanistic SAEMIX baseline fits Van der Pol dynamics with random effects on $\mu$ and the initial states, but no GP residual fields. The nonmechanistic Unified \method{} fit uses shared and local GP fields without a parametric mean. The two GP model formulations differ only by the mechanistic mean. The semi-mechanistic fit uses
\begin{equation}
\vf_i(\vx)=\vf_{\mathrm{VDP}}(\vx;\mu)+\boldsymbol r_0(\vx)+\vg_i(\vx),
\label{eq:vdp-semimechanistic-field}
\end{equation}
where $\boldsymbol r_0$ and $\vg_i$ have the shared and local GP priors of Section~\ref{sec:fieldprior}. We use a $\Normal(0,10^2)$ prior on $\mu$ truncated to $[10^{-4},5]$, and initialize its fit at $\mu=1$.

The unified semi-mechanistic fit extends the variational family to
\[
\begin{aligned}
q(\mu,\vb,\vs)&=q_\mu(\mu)q_b(\vb_0\mid\mu)\prod_i\bigl[q_b(\vb_i\mid\vb_0,\mu)q_i(\vs_i)\bigr],\\
z&\sim\Normal(m_z,s_z^2),\qquad
\mu=T(z)=10^{-4}+\frac{5-10^{-4}}{1+e^{-z}}.
\end{aligned}
\]
The transform enforces the bounds on $\mu$ while allowing variational optimization in the unconstrained coordinate $z$.
The residual in Eq.~\eqref{eq:residual-loss} becomes $\dot{\vx}_{i,k}-\vf_{\mathrm{VDP}}(\vx_{i,k};\mu)-\boldsymbol F_i(\vx_{i,k})\vb$. For each $\mu$, the joint shared--local coefficient optimum remains Gaussian. Its conditional mean depends on $\mu$, retaining mechanism--GP posterior dependence. The mechanism factor is fitted against the same profiled variational criterion, including $\KL(q_\mu\|p_\mu)$. Polynomial and RBF state moments are evaluated analytically, and the expectation over $q_\mu$ is approximated by averaging over 64 parameter values.

For the mechanistic parameter profile shown in the bottom-right panel of Figure~\ref{fig:population_method_rows}, we fix physical $\mu$ at 15 grid values in $[0.5,1]$ and reoptimize the Gaussian trajectory factors and joint GP coefficient posterior, retaining all fitted hyperparameters, including observation noise. The fixed-parameter criterion has no variational factor or prior penalty for $\mu$.

\subsection{Evaluation metrics}
\label{app:evaluation-metrics}
Let $\mathcal E$ contain the evaluated subject--time--coordinate entries in a target window, with $n=|\mathcal E|$. For $j\in\mathcal E$, write $x_j^\star$ for the noise-free latent truth and $\mu_j,s_j$ for the predictive mean and standard deviation. Scores use the marginal Gaussian $\Normal(\mu_j,s_j^2)$. For probabilistic methods, $s_j^2$ excludes observation noise. For deterministic CoDA, $s_j=\widehat\sigma_y$, estimated from training residuals, defines a Gaussian scoring approximation. This approximation does not represent posterior latent-state uncertainty.

\paragraph{Trajectory errors.}
With $e_j=x_j^\star-\mu_j$, the mean squared error, root mean squared error, and mean absolute error are
\[
\mathrm{MSE}=\frac1n\sum_{j\in\mathcal E}e_j^2,\qquad
\mathrm{RMSE}=\sqrt{\mathrm{MSE}},\qquad
\mathrm{MAE}=\frac1n\sum_{j\in\mathcal E}|e_j|.
\]
Training reconstruction errors use the same definitions with observed values in place of latent truth.

\paragraph{Probabilistic scores.}
CRPS evaluates the full predictive distribution, accounting for both location and spread, and is commonly used in probabilistic forecasting, including time-series forecasting~\citep{gneiting2007strictly}.
Let $z_j=e_j/s_j$, and let $\phi$ and $\Phi$ denote the standard-normal density and cumulative distribution function. The negative log predictive density and continuous ranked probability score are
\begin{align*}
\mathrm{NLPD}&=\frac1n\sum_{j\in\mathcal E}\left[\frac12\log(2\pi s_j^2)+\frac12z_j^2\right],\\
\mathrm{CRPS}&=\frac1n\sum_{j\in\mathcal E}s_j\left[z_j\bigl(2\Phi(z_j)-1\bigr)+2\phi(z_j)-\frac1{\sqrt\pi}\right].
\end{align*}
Equivalently, each CRPS term is $\int_{\mathbb R}[F_j(u)-\mathbf1\{x_j^\star\leq u\}]^2\,du$, where $F_j$ is the predictive Gaussian cumulative distribution function. Lower values are better for all five error and scoring metrics.

\paragraph{Coverage and calibration.}
For nominal coverage $\alpha$, the central predictive interval is
\[
I_{j,\alpha}=\left[\mu_j-c_\alpha s_j,\mu_j+c_\alpha s_j\right],\qquad
c_\alpha=\Phi^{-1}\!\left(\frac{1+\alpha}{2}\right).
\]
Its empirical coverage and the reported coverage and calibration errors are
\begin{align*}
\widehat C_\alpha&=\frac1n\sum_{j\in\mathcal E}\mathbf1\{x_j^\star\in I_{j,\alpha}\},\\
E_{95}&=|\widehat C_{0.95}-0.95|,\\
E_{\mathrm{cal}}&=\frac1{|\mathcal A|}\sum_{\alpha\in\mathcal A}|\widehat C_\alpha-\alpha|,
\qquad \mathcal A=\{0.50,0.80,0.90,0.95\}.
\end{align*}
Multiplying $E_{95}$ or $E_{\mathrm{cal}}$ by 100 expresses the error in percentage points. Lower errors indicate closer agreement with nominal coverage.

\paragraph{Score aggregation for tables and polygons.}
For trained-subject prediction, $\mathcal E$ pools the evaluated times and state coordinates across subjects within a seed. Held-out scores are computed separately for each subject and then averaged with equal subject weights. For seed scores $L_1,\ldots,L_R$, Table~\ref{tab:prediction_results} and Appendix Tables~\ref{tab:interpolation_results} and~\ref{tab:heldout_results} report
\[
\bar L=\frac1R\sum_{r=1}^R L_r,\qquad
s_L=\left[\frac1{R-1}\sum_{r=1}^R(L_r-\bar L)^2\right]^{1/2}.
\]
For the polygons in Table~\ref{tab:prediction_results}, a loss $L_{m,s,r}$ for method $m$, system $s$, and seed $r$ is transformed to
\[
S_{m,s,r}=\frac{\max_{m'}L_{m',s,r}-L_{m,s,r}}
{\max_{m'}L_{m',s,r}-\min_{m'}L_{m',s,r}}.
\]
This gives 1 to the best method and 0 to the worst within each matched comparison; a complete tie is assigned $1/2$. Scores are averaged over seeds within each system, then equally across systems.

\subsection{Hardware}\label{sec:hardware-details}

All experiments were conducted on a compute cluster equipped with four NVIDIA V100 GPUs (32 GB memory each), as well as on a MacBook Pro with an Apple M4 Pro chip.

\begingroup\color{black}
\section{Additional derivations}
\label{app:add-derivs}

\subsection{Notation reference}
\label{app:field-amplitudes}
\paragraph{Output amplitudes.}
The implementation allows distinct amplitudes for each output of the shared and local fields. For $a\in\{0,r\}$, let $\bar k_a$ be a unit-variance RBF kernel with an ARD lengthscale $\ell_{a,j}>0$ for each input coordinate $j$, shared across outputs. Its whitened feature row and the amplitude-scaled row for output $d$ are
\[
\begin{aligned}
\bar{\boldsymbol\phi}_a(\vx)
&=\bar k_a(\vx,\boldsymbol Z_a)\bar k_a(\boldsymbol Z_a,\boldsymbol Z_a)^{-1/2},\\
\boldsymbol\phi_{a,d}(\vx)&=\sigma_{a,d}\bar{\boldsymbol\phi}_a(\vx),\qquad \sigma_{a,d}>0.
\end{aligned}
\]
The output kernel is $k_{a,d}=\sigma_{a,d}^2\bar k_a$, and the scalar field expansions are $f_{0,d}(\vx)=\boldsymbol\phi_{0,d}(\vx)\vb_{0,d}$ and $g_{i,d}(\vx)=\boldsymbol\phi_{r,d}(\vx)\vb_{i,d}$, with the coefficient priors in Eq.~\eqref{eq:weightprior}. Thus the shared and local pointwise prior variances are $\sigma_{0,d}^2$ and $\sigma_{r,d}^2/\lambda_r$, respectively. When $\sigma_{a,d}=\sigma_a$ for all outputs, $k_a=\sigma_a^2\bar k_a$ and $\boldsymbol\phi_a=\sigma_a\bar{\boldsymbol\phi}_a$ recover the common-output notation of Sections~\ref{sec:fieldprior}--\ref{sec:scalable}. The implementation's prior field covariance generalizes Eq.~\eqref{eq:fieldcov} to
\[
\Cov[\vf_i(\vx),\vf_j(\vx')]
=\diag_{d=1}^D\Bigl[\sigma_{0,d}^2\bar k_0(\vx,\vx')
+\tfrac{\delta_{ij}}{\lambda_r}\sigma_{r,d}^2\bar k_r(\vx,\vx')\Bigr].
\]
The reference fit fixes $\lambda_r=1$ and learns the output variances $\sigma_{a,d}^2$ through these feature rows, with unit coefficient priors. Because the local field prior depends on $\sigma_{r,d}^2/\lambda_r$, learning both local amplitudes and $\lambda_r$ would introduce redundant scale parameters. New-subject adaptation holds the trained amplitudes fixed and may use a different $\lambda_{r,\mathrm{adapt}}$ (Appendix~\ref{app:new-subject}).
\begingroup
\setlength{\LTcapwidth}{\linewidth}
\begin{longtable}{@{}p{0.20\linewidth}p{0.74\linewidth}@{}}
\caption{Notation for the finite-feature variational construction. A dagger (${}^\dagger$) marks notation introduced in the appendix, i.e., not present in the main text.}\label{tab:notation-core}\\
\toprule Symbol & Definition\\\midrule\endfirsthead
\toprule Symbol & Definition\\\midrule\endhead\bottomrule\endfoot
$\vs_{i,k}$ & Stacked state and derivative in $\R^{2D}$ as in the background; $\vx=\Hx\vs$, $\dot{\vx}=\Hxdot\vs$.\\
$\mathcal C_i,T_i$ & Subject grid with $K_i+1$ times; positive interval duration.\\
$\bar k_a,\ell_{a,j}\;{}^\dagger$ & Unit-variance RBF kernel for field $a\in\{0,r\}$ and its lengthscale for input coordinate $j$, shared across outputs.\\
$\sigma_{a,d}\;{}^\dagger$ & Positive field amplitude for output $d$; the shared/local prior variances are $\sigma_{0,d}^2$ and $\sigma_{r,d}^2/\lambda_r$.\\
$\boldsymbol{\phi}_0,\boldsymbol{\phi}_r$ & Whitened feature rows in the main text's common-output kernel notation, of ranks $l_0,l_r$.\\
$\bar{\boldsymbol\phi}_a,\boldsymbol\phi_{a,d}\;{}^\dagger$ & Unit-amplitude base feature row and its amplitude-scaled row for output $d$.\\
$\boldsymbol{\Phi}_0,\boldsymbol{\Phi}_r\;{}^\dagger$ & Output-wise feature designs in $\R^{D\times p_0}$ and $\R^{D\times p_r}$, including output amplitudes.\\
$\boldsymbol F_i,\vb,\boldsymbol P$ & Global subject design, stacked field coefficients, and their prior covariance.\\
$\omega_{i,k},\kappa$ & Normalized temporal weights and residual strength.\\
$p_{T,i},\bm m_{\mathrm{init}},\boldsymbol P_{\mathrm{init}}$ & IWP trajectory density and fixed initial augmented-state mean and covariance in Eq.~\eqref{eq:iwp-path-prior}.\\
$q_b,\vm_b,\vSigma_b$ & Joint Gaussian coefficient posterior, mean, and covariance.\\
$\vm_{b,a},\vSigma_{b,a},\vSigma_{b,ac}\;{}^\dagger$ & Marginal coefficient mean and covariance for block $a$, and cross covariance between blocks $a$ and $c$; $a=0$ denotes the shared field.\\
$q_i,\vm_{s,i},\boldsymbol J_i$ & Gaussian Markov trajectory factor, mean, and path precision.\\
$\boldsymbol\Psi_{T,i}\;{}^\dagger$ & Fixed IWP trajectory-prior precision for subject $i$.\\
$\vLambda_b,\veta_b$ & Coefficient precision and natural mean from expected feature statistics.\\
$\boldsymbol C_i,\boldsymbol G_i,\bm\beta_i\;{}^\dagger$ & Local conditional covariance, gain with sign convention $\bm\beta_i-\boldsymbol G_i\vb_0$, and offset.\\
$\boldsymbol T_{i,k},\bm a_{i,k}\;{}^\dagger$ & Gaussian natural parameters at each grid time for the trajectory optimizer.\\
$\mathcal R_K,\mathcal J_K$ & Finite-grid residual penalty and variational criterion.\\
\end{longtable}
\endgroup

For coefficient calculations, write $p_0=Dl_0$ and $p_r=Dl_r$. For $a\in\{0,r\}$, the block-diagonal feature matrix for all $D$ outputs is
\[
\boldsymbol{\Phi}_a(\vx)=\blkdiag_{d=1}^D\bigl(\boldsymbol\phi_{a,d}(\vx)\bigr)\in\R^{D\times Dl_a}.
\]
These are the two nonzero blocks of $\boldsymbol F_i(\vx)\in\mathbb R^{D\times(p_0+Mp_r)}$, so that $\vf_i(\vx)=\boldsymbol{\Phi}_0(\vx)\vb_0+\boldsymbol{\Phi}_r(\vx)\vb_i$.

For example, with $M=5$, the design for subject $3$ is
\[
\boldsymbol F_3(\vx)=
\begin{bmatrix}
\boldsymbol{\Phi}_0(\vx)&\mathbf0&\mathbf0&\boldsymbol{\Phi}_r(\vx)&\mathbf0&\mathbf0
\end{bmatrix}\in\mathbb R^{D\times(p_0+5p_r)},
\]
where the columns correspond to $(\vb_0,\vb_1,\ldots,\vb_5)$ and each zero block has size $D\times p_r$.

In the parameterization of Eq.~\eqref{eq:weightprior}, the shared and local prior covariances are $\boldsymbol P_0=\Id_{p_0}$ and $\boldsymbol P_r=\lambda_r^{-1}\Id_{p_r}$. The global prior covariance is $\boldsymbol P=\blkdiag(\boldsymbol P_0,\boldsymbol P_r,\ldots,\boldsymbol P_r)$, with $M$ local blocks. Equivalently, absorbing $\lambda_r^{-1/2}$ into $\boldsymbol{\Phi}_r$ and replacing each $\vb_i$ by $\lambda_r^{1/2}\vb_i$ gives unit coefficient priors.

\subsection{Target distribution: ODE constraint relaxation and finite-grid approximation}
\label{app:target}
This section relates the finite-feature, finite-grid target in Eq.~\eqref{eq:relaxed-target} to a continuous-time construction with full GP fields. From this construction, temporal quadrature yields the finite-grid target, and we derive the corresponding variational criterion from its density.

\textbf{Continuous-time target.} Let $P_{f_0}$ and $P_{g_i}$ denote the full GP field priors of Section~\ref{sec:fieldprior}, and let $P_{T,i}$ be the continuous IWP trajectory prior on subject $i$'s time interval, whose grid restriction has the density given in Eq.~\eqref{eq:iwp-path-prior}. Its initial-state marginal is the model prior in Eq.~\eqref{eq:initial-state-prior}. Their product defines a reference measure $\nu_\infty=P_{f_0}\otimes\bigotimes_i(P_{g_i}\otimes P_{T,i})$. With $\vf_i=\vf_0+\vg_i$, the continuous-time residual penalty and the density of the relaxed target measure $\Pi_\infty(\cdot\mid\vy)$ relative to $\nu_\infty$ are
\begin{equation}
\begin{aligned}
\mathcal R_\infty(\vf,\vs)
&=\frac\kappa2\sum_i\frac1{T_i}\int_{\tau_{i,0}}^{\tau_{i,K_i}}
\|\dot{\vx}_i(t)-\vf_i(\vx_i(t))\|^2\,\mathrm dt,\\
\frac{\mathrm d\Pi_\infty(\cdot\mid\vy)}{\mathrm d\nu_\infty}(\vf,\vs)
&=Z_\infty^{-1}L_y(\vs)\exp[-\mathcal R_\infty(\vf,\vs)].
\end{aligned}
\label{eq:fullgp-relaxed-target}
\end{equation}
Here $L_y$ evaluates the measurement densities at their times on the continuous paths. For continuous field and trajectory samples with finite residual integral, the target is proper: the Gaussian observation likelihood is positive and bounded above when observation covariances are fixed and positive definite, and $0<\exp(-\mathcal R_\infty)\leq1$ almost surely. This is a soft-constraint target rather than the posterior induced by exact ODE solutions.

\textbf{Finite-grid target.} Replacing the full fields by the finite features of Section~\ref{sec:scalable}, restricting the IWP paths to the grid, and applying temporal quadrature yields Eq.~\eqref{eq:relaxed-target}, which is therefore an approximation of Eq.~\eqref{eq:fullgp-relaxed-target}. It is \emph{not} its exact finite-dimensional marginal.

For partial observations, each Gaussian density in Eq.~\eqref{eq:physical-likelihood} uses the measured entries of $\vy_{i,n}$, the corresponding rows of $\vH$, and the corresponding covariance submatrix. Measurement values enter the observation likelihood, while their times also serve as collocation points for the ODE residuals. Thus, each recorded measurement contributes one likelihood factor.

Let $h_{i,k}=\tau_{i,k+1}-\tau_{i,k}>0$. The normalized trapezoidal weights are
\begin{equation}
\omega_{i,0}=\frac{h_{i,0}}{2T_i},\quad
\omega_{i,k}=\frac{h_{i,k-1}+h_{i,k}}{2T_i}\ (0<k<K_i),\quad
\omega_{i,K_i}=\frac{h_{i,K_i-1}}{2T_i}.
\label{eq:time-weights}
\end{equation}
They are positive and sum to one. For the finite-feature field $\vf_i(\vx)=\boldsymbol F_i(\vx)\vb$, the resulting quadrature is
\[
\sum_{k=0}^{K_i}\omega_{i,k}\|\dot{\vx}_{i,k}-\vf_i(\vx_{i,k})\|^2
\;\approx\;
\frac1{T_i}\int_{\tau_{i,0}}^{\tau_{i,K_i}}
\|\dot{\vx}_i(t)-\vf_i(\vx_i(t))\|^2\,\mathrm dt.
\]

Duplicate grid times are merged, while genuine repeated measurements retain separate likelihood factors. We set $\kappa$ before optimization and hold it fixed, as described in Appendix~\ref{subsec:hyperparameter_learning}.

\textbf{Variational criterion.} The reference measure is the normalized coefficient prior times the independent IWP path priors. With fixed positive-definite observation covariances, the measurement likelihood is positive and bounded above by a finite constant. Since $0<\exp(-\mathcal R_K)\leq1$ whenever the residual is finite, its normalizing constant satisfies $0<Z_K<\infty$. Taking logarithms of Eq.~\eqref{eq:relaxed-target} and integrating under $q$ gives
\[
\KL(q\|\pi_K)=\KL(q_b\|p_b)+\sum_i\KL(q_i\|p_{T,i})-\E_q\log L_y+\E_q\mathcal R_K+\log Z_K.
\]
Thus the objective is a KL criterion for this fixed relaxed target.

\textbf{Coefficient marginal.} It is worth noticing that, even without measurements, the ODE penalty couples the field coefficients to the IWP trajectory prior, so the coefficient marginal of the relaxed target generally differs from $p_b$. Integrating out the trajectories gives a marginal proportional to
\[
p_b(\vb)Z_T(\vb),\qquad
Z_T(\vb)=\int\prod_i p_{T,i}(\vs_i)\exp[-\mathcal R_K(\vb,\vs)]\,\mathrm d\vs.
\]
This is a property of the stated target. Dividing by $Z_T(\vb)$ would define a different conditional path model and generally remove coefficient conjugacy.

\subsection{ODE loss: expected residual and analytic feature moments}
\label{app:colloc-derivation}
This section derives the expected squared ODE residual in Eqs.~\eqref{eq:trajectory-potential} and~\eqref{eq:expected-residual}, which enters the variational criterion in Eq.~\eqref{eq:objective}, accounting for uncertainty in both the field coefficients and the latent states and derivatives. Averaging over the Gaussian state distribution requires first, second, and state--derivative cross moments of the nonlinear feature maps; we derive these analytically for RBF/Nystr\"om features.

Independence of $q_b$ and $q_i$ gives, conditional on $(\vx,\dot{\vx})$,
\begin{equation}
\E_{q_b}\|\dot{\vx}-\boldsymbol F_i(\vx)\vb\|^2
=\|\dot{\vx}-\boldsymbol F_i(\vx)\vm_b\|^2+
\tr\{\boldsymbol F_i(\vx)\vSigma_b\boldsymbol F_i(\vx)^\top\}.
\label{eq:expected-field-square}
\end{equation}
Conditioning the Gaussian $(\vx,\dot{\vx})$ marginal on $\vx$ yields Eq.~\eqref{eq:conditional-velocity} and hence Eq.~\eqref{eq:expected-residual}. This retains state--derivative covariance. The derivative covariance contributes to the residual expectation without introducing a variance-dependent regression precision.

\paragraph{Analytic RBF moments.}
For $\vx\sim\Normal(\vm_x,\boldsymbol P_{xx})$, a positive diagonal squared-lengthscale matrix $\boldsymbol L$, and
$r_{\boldsymbol L}(\vx,\bm z)=\exp[-(\vx-\bm z)^\top\boldsymbol L^{-1}(\vx-\bm z)/2]$, Gaussian integration gives
\begin{equation}
\E[r_{\boldsymbol L}(\vx,\bm z)]
=\frac{|\boldsymbol L|^{1/2}}{|\boldsymbol L+\boldsymbol P_{xx}|^{1/2}}
\exp\left[-\frac12(\vm_x-\bm z)^\top(\boldsymbol L+\boldsymbol P_{xx})^{-1}(\vm_x-\bm z)\right].
\label{eq:rbf-first-moment}
\end{equation}
Here $|\boldsymbol A|$ denotes the determinant of a matrix. A product of two RBFs yields another RBF, so its expectation follows from Eq.~\eqref{eq:rbf-first-moment}. For diagonal $\boldsymbol L_a,\boldsymbol L_b$, set
\[
\boldsymbol L_c=(\boldsymbol L_a^{-1}+\boldsymbol L_b^{-1})^{-1},\qquad
\bm z_c=\boldsymbol L_c(\boldsymbol L_a^{-1}\bm z_a+\boldsymbol L_b^{-1}\bm z_b).
\]
Expanding the two quadratic forms in the exponents and regrouping their sum into a quadratic form in $\vx-\bm z_c$ plus a term independent of $\vx$ yields
\begin{equation}
r_{\boldsymbol L_a}(\vx,\bm z_a)r_{\boldsymbol L_b}(\vx,\bm z_b)
=\exp\left[-\frac12(\bm z_a-\bm z_b)^\top(\boldsymbol L_a+\boldsymbol L_b)^{-1}(\bm z_a-\bm z_b)\right]
r_{\boldsymbol L_c}(\vx,\bm z_c).
\label{eq:rbf-product}
\end{equation}
Whitening transforms these moments into moments of the unit-amplitude rows $\bar{\boldsymbol\phi}_0$ and $\bar{\boldsymbol\phi}_r$ in Appendix~\ref{app:field-amplitudes}. Output amplitudes multiply first moments by $\sigma_{a,d}$ and feature-product moments by the corresponding product of amplitudes. For a differentiable scalar feature $\phi$ with integrable derivatives,
\begin{equation}
\E[\dot{\vx}\phi(\vx)]=\vm_{\dot{x}}\E[\phi(\vx)]+\boldsymbol P_{\dot{x}x}\E[\nabla\phi(\vx)].
\label{eq:velocity-feature-moment}
\end{equation}
This follows by substituting the conditional mean of $\dot{\vx}$ and applying Gaussian integration by parts to $\vx$. For the RBF, $\E[\nabla\phi]=\nabla_{\vm_x}\E[\phi]$, directly obtained from Eq.~\eqref{eq:rbf-first-moment}. Together with $\E[\dot{\vx}\dot{\vx}^\top]=\boldsymbol P_{\dot{x}\dot{x}}+\vm_{\dot{x}}\vm_{\dot{x}}^\top$, these formulas evaluate every required statistic exactly for fixed RBF/Nystr\"om features.

\subsection{Field optimization: joint coefficient update and cross covariances}
\label{app:fieldupdates}
For fixed trajectory factors, this section derives the joint field-coefficient distribution in Eq.~\eqref{eq:coefficient-update} by expanding the update into shared and local blocks. We exploit the resulting block-arrowhead precision to solve for the shared marginal and local conditional distributions, then recover their marginal and cross covariances.

For fixed trajectory factors and features, positive-definite $\boldsymbol P$, nonnegative penalty weights, and finite required moments imply $\vLambda_b\succeq\boldsymbol P^{-1}\succ0$: the expected design Gram matrices are positive semidefinite. Completing the square gives the coefficient criterion as $\KL(q_b\|q_b^\star)$ plus a constant, where $q_b^\star=\Normal(\vLambda_b^{-1}\veta_b,\vLambda_b^{-1})$. Hence the optimum is unique; this conditional result does not establish component identifiability.

\textbf{Expected blocks.} Partition $\vLambda_b$, $\veta_b$, and $\vm_b$ from Eq.~\eqref{eq:coefficient-update} according to $(\vb_0,\vb_1,\ldots,\vb_M)$, and write $\vLambda_{ab}$, $\veta_a$, and $\vm_{b,a}$ for their blocks. The expectations below are with respect to the joint marginal $q_i(\vx_{i,k},\dot{\vx}_{i,k})$ at each grid time. Define the expected blocks
\begin{align}
\boldsymbol A_i&=\sum_k\kappa\omega_{i,k}\E[\boldsymbol{\Phi}_0(\vx_{i,k})^\top\boldsymbol{\Phi}_0(\vx_{i,k})],&
\boldsymbol E_i&=\sum_k\kappa\omega_{i,k}\E[\boldsymbol{\Phi}_0(\vx_{i,k})^\top\boldsymbol{\Phi}_r(\vx_{i,k})],\\
\boldsymbol D_i&=\sum_k\kappa\omega_{i,k}\E[\boldsymbol{\Phi}_r(\vx_{i,k})^\top\boldsymbol{\Phi}_r(\vx_{i,k})],&
\bm h_{0i}&=\sum_k\kappa\omega_{i,k}\E[\boldsymbol{\Phi}_0(\vx_{i,k})^\top\dot{\vx}_{i,k}],\label{eq:expected-blocks}\\
\bm h_{ri}&=\sum_k\kappa\omega_{i,k}\E[\boldsymbol{\Phi}_r(\vx_{i,k})^\top\dot{\vx}_{i,k}].&&
\end{align}
\textbf{Block-arrowhead precision.} For independent zero-mean coefficient priors with precisions $\boldsymbol P_0^{-1}$ and $\boldsymbol P_r^{-1}$,
\[
\vLambda_{00}=\boldsymbol P_0^{-1}+\sum_i\boldsymbol A_i,\quad
\vLambda_{0i}=\boldsymbol E_i,\quad
\vLambda_{ii}=\boldsymbol P_r^{-1}+\boldsymbol D_i,\quad
\veta_0=\sum_i\bm h_{0i},\quad\veta_i=\bm h_{ri}.
\]
All distinct local--local precision blocks vanish. For $M=5$, with rows and columns ordered as $(\vb_0,\vb_1,\ldots,\vb_5)$, the block-arrowhead precision is
\[
\vLambda_b=
\begin{bmatrix}
\vLambda_{00}&\vLambda_{01}&\vLambda_{02}&\vLambda_{03}&\vLambda_{04}&\vLambda_{05}\\
\vLambda_{10}&\vLambda_{11}&\mathbf0&\mathbf0&\mathbf0&\mathbf0\\
\vLambda_{20}&\mathbf0&\vLambda_{22}&\mathbf0&\mathbf0&\mathbf0\\
\vLambda_{30}&\mathbf0&\mathbf0&\vLambda_{33}&\mathbf0&\mathbf0\\
\vLambda_{40}&\mathbf0&\mathbf0&\mathbf0&\vLambda_{44}&\mathbf0\\
\vLambda_{50}&\mathbf0&\mathbf0&\mathbf0&\mathbf0&\vLambda_{55}
\end{bmatrix}.
\]
The shared block $\vLambda_{00}$ is $p_0\times p_0$, each local diagonal block is $p_r\times p_r$, and $\vLambda_{0i}=\vLambda_{i0}^\top$ is $p_0\times p_r$. Each local block couples to the shared block, producing the first row and column of the arrowhead. Note that the inverse $\vSigma_b=\vLambda_b^{-1}$ generally has nonzero local--local covariance blocks.

\textbf{Shared marginal and local conditionals.} First factor the positive-definite local precision blocks and define
\[
\boldsymbol C_i=\vLambda_{ii}^{-1},\qquad
\boldsymbol G_i=\boldsymbol C_i\vLambda_{i0},\qquad
\bm\beta_i=\boldsymbol C_i\veta_i.
\]
The local equations in $\vLambda_b\vm_b=\veta_b$ give $\vm_{b,i}=\bm\beta_i-\boldsymbol G_i\vm_{b,0}$. Substituting them into the shared block equation gives the reduced system for the shared marginal $q_b(\vb_0)=\Normal(\vm_{b,0},\vSigma_{b,0})$. Gaussian conditioning then gives each local conditional:
\begin{equation}
\begin{aligned}
\vSigma_{b,0}^{-1}&=\vLambda_{00}-\sum_i\vLambda_{0i}\boldsymbol C_i\vLambda_{i0},\\
\vm_{b,0}&=\vSigma_{b,0}\left(\veta_0-\sum_i\vLambda_{0i}\bm\beta_i\right),\\
q_b(\vb_i\mid\vb_0)&=\Normal(\bm\beta_i-\boldsymbol G_i\vb_0,\boldsymbol C_i).
\end{aligned}
\label{eq:arrowhead}
\end{equation}
\textbf{Marginal and cross covariances.} This gives the coefficient-optimal factorization in Eq.~\eqref{eq:coefficient-factorization}. Write $\vm_{b,i}$ and $\vSigma_{b,i}$ for the marginal moments of $\vb_i$, and $\vSigma_{b,ij}=\Cov_{q_b}(\vb_i,\vb_j)$. The implied marginal moments are
\begin{equation}
\begin{aligned}
\vm_{b,i}&=\bm\beta_i-\boldsymbol G_i\vm_{b,0},&
\vSigma_{b,i}&=\boldsymbol C_i+\boldsymbol G_i\vSigma_{b,0}\boldsymbol G_i^\top,\\
\vSigma_{b,0i}&=-\vSigma_{b,0}\boldsymbol G_i^\top,&
\vSigma_{b,ij}&=\boldsymbol G_i\vSigma_{b,0}\boldsymbol G_j^\top\quad(i\ne j).
\end{aligned}
\label{eq:joint-covariances}
\end{equation}
Local coefficients are independent conditional on the shared coefficients and can be correlated marginally through shared uncertainty. These identities follow from Gaussian conditioning of the joint precision in Eq.~\eqref{eq:coefficient-update}. Cholesky solves and elimination of the local blocks compute the result without constructing a dense matrix over all subjects. For fixed $p_0,p_r$, the block solve is linear in subject count and costs on the order of $Mp_r^3+Mp_0p_r(p_0+p_r)+p_0^3$.

\subsection{Trajectory optimization: profiled objective and Gaussian Markov updates}
\label{app:profiled}
\label{app:kalman}
The previous section derived the optimal joint field-coefficient factor $q_b$ for fixed trajectory factors. In this section, we now substitute this optimum into Eq.~\eqref{eq:objective} to derive the profiled objective, then describe how the trajectory factors are updated within the Gaussian Markov family of Eq.~\eqref{eq:variational}.

\textbf{Profiled objective.} Let
\[
c=\frac12\sum_{i,k}\kappa\omega_{i,k}\E[\dot{\vx}_{i,k}^\top\dot{\vx}_{i,k}],\qquad
q_s(\vs)=\prod_i q_i(\vs_i).
\]
We denote by $\mathcal J_{K,b}(q_b;q_s)$ the contribution to $\mathcal J_K$ (Eq.~\eqref{eq:objective}) involving the field coefficients $\vb$, including the residual constant $c$:
\begin{equation}
\mathcal J_{K,b}(q_b;q_s)=\KL(q_b\|p_b)+\frac12\E_{q_b}[\vb^\top(\vLambda_b-\boldsymbol P^{-1})\vb]-\veta_b^\top\vm_b+c.
\label{eq:coefficient-objective}
\end{equation}
For fixed $q_s$, minimizing Eq.~\eqref{eq:coefficient-objective} over $q_b$ gives the Gaussian whose parameters are specified in Eq.~\eqref{eq:coefficient-update}. Substituting that optimum into Eq.~\eqref{eq:coefficient-objective} gives
\[
\min_{q_b}\mathcal J_{K,b}(q_b;q_s)=c+\frac12\left(\log|\boldsymbol P|+\log|\vLambda_b|-\veta_b^\top\vLambda_b^{-1}\veta_b\right).
\]
Adding back the trajectory KL terms and the expected measurement loss from Eq.~\eqref{eq:objective}, which do not depend on $q_b$, yields the full profiled objective:
\begin{equation}
\mathcal J_{K,\mathrm{profiled}}(q_s)=\sum_i\KL(q_i\|p_{T,i})-\E[\log L_y]+c
+\frac12\left(\log|\boldsymbol P|+\log|\vLambda_b|-\veta_b^\top\vLambda_b^{-1}\veta_b\right).
\label{eq:profiled-objective}
\end{equation}
The joint precision determinant $|\vLambda_b|$ is the product of the local precision determinants $|\vLambda_{ii}|$ and $|\vSigma_{b,0}^{-1}|$, the determinant of the shared marginal precision in Eq.~\eqref{eq:arrowhead}.

It remains to minimize Eq.~\eqref{eq:profiled-objective} over the means and covariances of the Gaussian Markov trajectory factors $q_i$. The optimal $q_b$ is recomputed as these factors change; at each exact coefficient optimum, the envelope theorem permits trajectory derivatives to be evaluated with $q_b$ held fixed. We first give the trajectory KL term and its parameterization, then construct the natural-site updates for this optimization.

\textbf{Gaussian Markov KL and parameterization.} Recall that $p_{T,i}$ is the fixed IWP trajectory prior for subject $i$ in Eq.~\eqref{eq:iwp-path-prior}, while $q_i$ is the variational posterior factor for subject $i$\textquotesingle s trajectory in Eq.~\eqref{eq:variational}. Both are Gaussian Markov distributions, meaning that, conditional on the preceding augmented grid state, a state is independent of all earlier states. In practice, this means that their precision matrices $\boldsymbol\Psi_{T,i}$ and $\boldsymbol J_i$ are block tridiagonal. With $n_i=2D(K_i+1)$, both distributions act on the concatenated trajectory $\vs_i=(\vs_{i,0}^\top,\ldots,\vs_{i,K_i}^\top)^\top\in\R^{n_i}$. The factorized prior in Eq.~\eqref{eq:iwp-path-prior} has the equivalent joint Gaussian representation below, alongside the variational factor:
\begin{align*}
p_{T,i}(\vs_i)&=\Normal(\vs_i\mid\vm_{T,i},\boldsymbol\Psi_{T,i}^{-1}),\\
q_i(\vs_i)&=\Normal(\vs_i\mid\vm_{s,i},\boldsymbol S_{s,i}),\qquad
\boldsymbol S_{s,i}=\boldsymbol J_i^{-1}.
\end{align*}
The prior precision $\boldsymbol\Psi_{T,i}$ is fixed, whereas the variational precision $\boldsymbol J_i$ is optimized. The Gaussian KL is
\begin{equation}
\KL(q_i\|p_{T,i})=\frac12\left[
\tr(\boldsymbol\Psi_{T,i}\boldsymbol S_{s,i})+(\vm_{s,i}-\vm_{T,i})^\top\boldsymbol\Psi_{T,i}(\vm_{s,i}-\vm_{T,i})-n_i+
\log|\boldsymbol J_i|-\log|\boldsymbol\Psi_{T,i}|\right].
\label{eq:markov-kl}
\end{equation}
The trace uses only covariances of individual augmented grid states and cross covariances between consecutive states because $\boldsymbol\Psi_{T,i}$ is block tridiagonal. In practice, a lower block-bidiagonal factor $\boldsymbol L_i$ with nonsingular positive-diagonal blocks gives the factorization $\boldsymbol J_i=\boldsymbol L_i\boldsymbol L_i^\top$. This block structure permits recursive computation of the required covariance blocks along the grid; the log determinant of $\boldsymbol J_i$ is twice the sum of the log determinants of the diagonal blocks of $\boldsymbol L_i$. These trajectory matrix operations take linear time in $K_i$ for fixed $D$. Computing feature moments and updating field coefficients additionally incurs costs depending on the feature ranks $l_0,l_r$ (Appendices~\ref{app:colloc-derivation} and~\ref{app:fieldupdates}). 

\textbf{Natural-site updates.} Since the expected measurement and ODE losses decompose over grid times, we express $q_i$ as the fixed IWP prior multiplied by factors acting at individual times. Their parameters determine the trajectory precision $\boldsymbol J_i$; its block-bidiagonal factor $\boldsymbol L_i$ is then computed for the covariance and log-determinant calculations in Eq.~\eqref{eq:markov-kl}. We denote by a \emph{site} an unnormalized factor $\exp[-\tfrac12\vs_{i,k}^\top\boldsymbol T_{i,k}\vs_{i,k}+\bm a_{i,k}^\top\vs_{i,k}]$ acting on the augmented state at a single grid time $\tau_{i,k}$. The symmetric matrix $\boldsymbol T_{i,k}\in\R^{2D\times2D}$ contributes to the trajectory precision, while the vector $\bm a_{i,k}\in\R^{2D}$ contributes to its precision-weighted mean. Again, these are variational parameters of $q_i$. Multiplying these factors by the IWP prior and normalizing gives $q_i$, with precision and precision-weighted mean
\begin{equation}
\boldsymbol J_i=\boldsymbol\Psi_{T,i}+\blkdiag(\boldsymbol T_{i,k}),\qquad
\bm\eta_i=\boldsymbol\Psi_{T,i}\vm_{T,i}+\operatorname{stack}_k(\bm a_{i,k}).
\label{eq:natural-sites}
\end{equation}
Here $\blkdiag$ places the matrices along the block diagonal, with zeros elsewhere, and $\operatorname{stack}_k$ concatenates the vectors in increasing time order. The corresponding trajectory mean and covariance are $\vm_{s,i}=\boldsymbol J_i^{-1}\bm\eta_i$ and $\boldsymbol S_{s,i}=\boldsymbol J_i^{-1}$. The sites are initialized from the measurement likelihood. Let $\mathcal L_i$ be the sum of expected measurement negative log likelihoods and expected ODE losses, with the current coefficient optimum held fixed. Write $\boldsymbol P_{i,kk}=[\boldsymbol S_{s,i}]_{kk}=\Cov_{q_i}(\vs_{i,k})\in\R^{2D\times2D}$ for the diagonal covariance block at grid time $\tau_{i,k}$. At that time, write its derivatives as
$\bm g_{m,k}=\partial\mathcal L_i/\partial\vm_{s,i,k}$ and
$\boldsymbol g_{P,k}=\partial\mathcal L_i/\partial\boldsymbol P_{i,kk}$.
Evaluating these derivatives at the current $q_i$ gives the site-parameter targets
\begin{equation}
\boldsymbol T_{i,k}^{\mathrm{target}}=2\operatorname{sym}(\boldsymbol g_{P,k}),\qquad
\bm a_{i,k}^{\mathrm{target}}=\boldsymbol T_{i,k}^{\mathrm{target}}\vm_{s,i,k}-\bm g_{m,k}.
\label{eq:site-target}
\end{equation}
Here $\operatorname{sym}(\boldsymbol A)=(\boldsymbol A+\boldsymbol A^\top)/2$. The superscript $\mathrm{target}$ distinguishes these recomputed update values from the current site parameters $\boldsymbol T_{i,k},\bm a_{i,k}$ that define $q_i$; the damped updates below move the current parameters toward these targets.
To derive Eq.~\eqref{eq:site-target}, differentiate the trajectory KL $\KL(q_i\|p_{T,i})$ in Eq.~\eqref{eq:markov-kl}: its mean derivative is $\boldsymbol\Psi_T(\vm-\vm_T)$, and its covariance derivative is $(\boldsymbol\Psi_T-\boldsymbol J)/2$. Setting the derivatives of the sum of this KL and the expected loss with respect to the full trajectory mean and covariance to zero gives $\boldsymbol J=\boldsymbol\Psi_T+2\blkdiag(\boldsymbol g_P)$ and the mean equation in Eq.~\eqref{eq:site-target}. Any Gaussian with positive-definite covariance satisfying these zero-derivative equations therefore has this precision structure, with corrections confined to the diagonal blocks.

These site updates have the Gaussian natural-gradient form used in PHYSS-GP~\citep{hamelijnck2024physsgp}. Proposals interpolate between current and target site parameters, with damping selected by backtracking over $1,1/2,1/4,\ldots$. Each trial is evaluated against the same profiled objective. We accept finite trials with positive-definite path precision that do not increase the objective. Individual site matrices $\boldsymbol T_{i,k}$ need not be positive semidefinite.

\subsection{Initial conditions}
\label{app:init-conditions}
The IWP prior in Eq.~\eqref{eq:iwp-path-prior} starts with a fixed Gaussian law $\vs_{i,0}\sim\Normal(\bm m_{\mathrm{init}},\boldsymbol P_{\mathrm{init}})$, independent across subjects, with $\boldsymbol P_{\mathrm{init}}\succ0$. Its state marginal matches Eq.~\eqref{eq:initial-state-prior}, with $\bm m_{x,0}=\Hx\bm m_{\mathrm{init}}$ and $\boldsymbol P_{x,0}=\Hx\boldsymbol P_{\mathrm{init}}\Hx^\top$. We use $\bm m_{\mathrm{init}}=0$ and $\boldsymbol P_{\mathrm{init}}=100\Id_{2D}$. Initial state and derivative uncertainty are inferred as part of the Gaussian Markov trajectory. The initial derivative $\dot{\vx}_{i,0}$ is included in the augmented IWP state; the ODE residual penalty encourages it to match the vector field $\vf_i(\vx_{i,0})$.

For posterior ODE rollouts starting at grid time $\tau_{i,k}$, we draw the starting state from the fitted marginal $q_i(\vx_{i,k})=\Normal(\vx_{i,k}\mid\vm_{x,i,k},\boldsymbol P_{xx,i,k})$, with $\vm_{x,i,k}=\Hx\vm_{s,i,k}$ and $\boldsymbol P_{xx,i,k}=\Hx\boldsymbol P_{i,kk}\Hx^\top$. Here $\boldsymbol P_{i,kk}$ is the corresponding diagonal covariance block of $\boldsymbol S_{s,i}$. These fitted moments are specific to subject $i$ and incorporate its fitting observations through $q_i$.

\color{black}
\subsection{Hyperparameters: initialization and optimization}
\label{subsec:hyperparameter_learning}
This section describes the initialization of observation noise and fixed model quantities, then the optimization of learned hyperparameters against Eq.~\eqref{eq:profiled-objective}.

\textbf{Initial observation noise estimate.} We estimate measurement noise from training observations to initialize the likelihood. For consecutive observations at $t_-<t_0<t_+$, set $a=(t_+-t_0)/(t_+-t_-)$ and $r=[y_0-ay_--(1-a)y_+]/\sqrt{1+a^2+(1-a)^2}$. Pooling these residuals across subjects and coordinates gives $\sigma_{\mathrm{init}}=\operatorname{MAD}(r)/\Phi^{-1}(0.75)$, where MAD is the median absolute deviation and $\Phi^{-1}(0.75)\simeq0.67449$ provides the Gaussian calibration.

\textbf{IWP diffusion scale.} The scale $\sigma_{\mathrm{IWP}}^2$ controls changes in the trajectory derivative through the process covariance $\vQ(h)$ in Eq.~\eqref{eq:ssgp}. Larger values allow greater variability of the derivative, while smaller values favor trajectories closer to linear functions of time. We initialize $\sigma_{\mathrm{IWP}}^2$ using the observation-only marginal likelihood defined by the prior in Eq.~\eqref{eq:iwp-path-prior} and the likelihood in Eq.~\eqref{eq:physical-likelihood}, with $\boldsymbol\Sigma_y=\sigma_{\mathrm{init}}^2\Id$:
\[
p(\vy\mid\sigma_{\mathrm{IWP}}^2)=\int L_y(\vs)\prod_i p_{T,i}(\vs_i)\,\mathrm d\vs.
\]
Here each $p_{T,i}$ uses the candidate diffusion scale and the fixed initial prior from Appendix~\ref{app:init-conditions}. We minimize the negative log marginal likelihood over $\log\sigma_{\mathrm{IWP}}^2\in[-10,\log(10^4)]$ and keep the selected scale fixed during subsequent fitting.

\textbf{Feature locations.} The Nystr\"om locations $\boldsymbol Z_0$ and $\boldsymbol Z_r$ are sampled uniformly without replacement from the pooled training observations and held fixed. For partially observed states, missing coordinates are filled with their training-observation means, or zero if entirely unobserved, for this initialization only.

\textbf{Penalty strength.} The parameter $\kappa$ controls the weight of the time-averaged squared ODE residual in Eq.~\eqref{eq:residual-loss}, with larger values enforcing closer agreement between the trajectory derivative and the vector field. We calibrate its scale using the heuristic $\kappa=(\sigma_y^2/N)^{-1}=N/\sigma_y^2$, where $N$ is the nominal number of observations per subject and $\sigma_y$ is the nominal observation-noise standard deviation used for calibration. The resulting $\kappa$ is held fixed throughout fitting, including subsequent observation-noise updates. For the six main benchmarks, $N=20$ and $\sigma_y=0.05$ give $\kappa=8000$. This is a scale-setting heuristic in the benchmark units.

\textbf{Optimization of learned hyperparameters.} Shared and local output variances $\sigma_{a,d}^2$ and ARD lengthscales $\ell_{a,j}$, together with observation-noise parameters, are optimized in log coordinates against Eq.~\eqref{eq:profiled-objective}. At each hyperparameter proposal, the trajectory factors $q_i$ are held fixed and the optimal joint coefficient factor $q_b$ is recomputed using Eq.~\eqref{eq:coefficient-update}. Inducing locations, IWP scale, the initial augmented-state prior (Appendix~\ref{app:init-conditions}), the local coefficient-prior precision $\lambda_r$, and $\kappa$ remain fixed.
The amplitude parameterization and fixed local prior precision are detailed in Appendix~\ref{app:field-amplitudes}.

Lengthscale, amplitude, and observation-noise updates start after 5\% of the fixed iteration budget. Finite proposals are retained only after checking the common criterion in Eq.~\eqref{eq:profiled-objective}. Because $Z_K$ changes with these parameters, this optimization is not exact evidence maximization for $\pi_K$.

\color{black}
\subsection{Inference algorithm}
\label{app:inference-algorithm}
Algorithm~\ref{alg:main} summarizes the default natural-site trajectory optimizer with analytic coefficient profiling.

\begin{algorithm}[htbp]
\caption{Variational inference for \method{}}
\label{alg:main}
\begin{algorithmic}[1]
\Require Fixed feature locations in Eq.~\eqref{eq:features}, subject grids, priors in Eqs.~\eqref{eq:iwp-path-prior} and~\eqref{eq:weightprior}, $\kappa$, and iteration budget.
\State Form normalized weights from Eq.~\eqref{eq:time-weights} and masked likelihood factors from Eq.~\eqref{eq:physical-likelihood}.
\State Initialize $q_i$ by Kalman smoothing with the IWP prior and measurement likelihood in Eqs.~\eqref{eq:iwp-path-prior} and~\eqref{eq:physical-likelihood}.
\State Compute expected feature statistics (Appendix~\ref{app:colloc-derivation}) and compute $q_b$ using Eqs.~\eqref{eq:coefficient-update} and~\eqref{eq:arrowhead}.
\For{each training iteration}
\State Differentiate the expected measurement and ODE loss terms of Eq.~\eqref{eq:objective} to form the site targets in Eq.~\eqref{eq:site-target}.
\State Propose damped updates of the sites in Eq.~\eqref{eq:natural-sites}, recomputing $q_b$ via Eq.~\eqref{eq:coefficient-update} and evaluating Eq.~\eqref{eq:profiled-objective}.
\State Accept finite proposals with positive-definite $\boldsymbol J_i$ in Eq.~\eqref{eq:variational} and nonincreasing profiled objective in Eq.~\eqref{eq:profiled-objective}.
\State Refresh trajectory marginals and the coefficient factor in Eqs.~\eqref{eq:variational} and~\eqref{eq:coefficient-update}.
\If{warm-up is complete}
\State Propose bounded hyperparameter updates against Eq.~\eqref{eq:profiled-objective}, recomputing Eq.~\eqref{eq:coefficient-update} at each proposal.
\State Retain finite proposals that do not increase this criterion (Appendix~\ref{subsec:hyperparameter_learning}).
\EndIf
\EndFor
\end{algorithmic}
\end{algorithm}

\color{black}
\subsection{Prediction for a trained subject}
\label{app:seen-subject-extrapolation}
This section describes how the fitted field and trajectory factors yield posterior ODE predictions for subjects observed during training. We use these predictions for forecasting beyond the fitting observations and interpolation within a withheld time window.

\textbf{Forecasting.} For rollouts from the subject's grid origin $\tau_{i,0}$, sample the shared and local coefficients jointly and the starting state independently:
\begin{equation}
\vb_0^{(s)}\sim q_b(\vb_0),\quad
\vb_i^{(s)}\sim q_b(\vb_i\mid\vb_0^{(s)}),\quad
\vx_i^{(s)}(\tau_{i,0})\sim q_i(\Hx\vs_{i,0}).
\label{eq:predictive-draws}
\end{equation}
The same shared coefficient draw $\vb_0^{(s)}$ is used for every subject in a cohort sample. Integrate
\[
\dot\vx_i^{(s)}(t)=\boldsymbol{\Phi}_0(\vx_i^{(s)}(t))\vb_0^{(s)}+
\boldsymbol{\Phi}_r(\vx_i^{(s)}(t))\vb_i^{(s)}
\]
on the simulation time grid with fixed-step RK4. We use 64 samples and four RK4 substeps per grid interval. Means and variances of those paths define latent trajectory summaries. Predictive scores use the Gaussian marginal approximation of these trajectories, evaluated against noise-free truth (Appendix~\ref{app:evaluation-metrics}). Observation-prediction variance additionally includes the measurement covariance. The rollouts propagate uncertainty in the finite-feature coefficients and initial state. Forecasts evaluate these trajectories beyond the last fitting observation.

\textbf{Interpolation.} We fit the model using observations on both sides of the withheld time window, then apply the same sampling and ODE integration procedure in Eq.~\eqref{eq:predictive-draws} and evaluate the resulting trajectories inside that window. Auxiliary smoothing also returns Gaussian state marginals from $q_i$ on the fitted grid. The ODE pushforward can differ from these marginals because the variational factorization makes origin-state and coefficient draws independent and the residual penalty is finite.

\color{black}
\subsection{New-subject adaptation}
\label{app:new-subject}
For a new subject, we use the trained shared-coefficient marginal as a Gaussian prior and introduce independent local coefficients. Write
\[
q_b^{\mathrm{train}}(\vb_0)=\Normal(\vm_{b,0}^{\mathrm{train}},\vSigma_{b,0}^{\mathrm{train}}),\qquad
p_{r,\star}(\vb_\star)=\Normal(\mathbf0,\boldsymbol P_{r,\star}),\qquad
\boldsymbol P_{r,\star}=\lambda_{r,\mathrm{adapt}}^{-1}\Id_{p_r}.
\]
Under these independent Gaussian coefficient priors, the new-subject field $\vf_\star(\vx)=\boldsymbol\Phi_0(\vx)\vb_0+\boldsymbol\Phi_r(\vx)\vb_\star$ is a finite-rank vector-valued GP, with mean and covariance
\begin{equation}
\begin{aligned}
\bm\mu_{\star,\mathrm{prior}}^f(\vx)
&=\boldsymbol\Phi_0(\vx)\vm_{b,0}^{\mathrm{train}},\\
\boldsymbol K_{\star,\mathrm{prior}}^f(\vx,\vx')
&=\boldsymbol\Phi_0(\vx)\vSigma_{b,0}^{\mathrm{train}}\boldsymbol\Phi_0(\vx')^\top
+\boldsymbol\Phi_r(\vx)\boldsymbol P_{r,\star}\boldsymbol\Phi_r(\vx')^\top.
\end{aligned}
\label{eq:new-subject-commonmean-prior}
\end{equation}

Using only the subject's context observations, we fit a joint Gaussian coefficient factor $q_{b,\star}(\vb_0,\vb_\star)$ and a Gaussian Markov trajectory factor $q_\star(\vs_\star)$ by minimizing
\[
\mathcal J_\star=\KL\!\left(q_{b,\star}\,\middle\|\,q_b^{\mathrm{train}}p_{r,\star}\right)+\KL(q_\star\|p_{T,\star})-\E_{q_\star}[\log L_{y,\star}]+\E_{q_{b,\star}q_\star}[\mathcal R_{K,\star}].
\]
Here $L_{y,\star}$ contains only context measurements, and $\mathcal R_{K,\star}$ uses the normalized temporal weights on the context grid. We apply the coefficient update with the nonzero prior mean specified below and the trajectory updates of Section~\ref{sec:trajectory-optimization}, holding feature locations, kernel parameters, observation noise, the IWP diffusion scale, and $\kappa$ fixed. The subject's observations update both shared and local coefficients in its adapted posterior; the fitted training-population posterior is held fixed, and each new subject is adapted independently from that posterior.

\paragraph{Coefficient update with a nonzero prior mean.}
For a fixed trajectory factor $q_\star$, form $\boldsymbol A_\star,\boldsymbol E_\star,\boldsymbol D_\star,\bm h_{0\star},\bm h_{r\star}$ as in Eq.~\eqref{eq:expected-blocks}. Unlike the zero-mean training prior in Eq.~\eqref{eq:coefficient-update}, the adaptation prior contributes a nonzero natural mean. The joint coefficient optimum has precision and natural mean
\begin{equation}
\begin{aligned}
\vLambda_{b,\star}
&=\begin{bmatrix}
(\vSigma_{b,0}^{\mathrm{train}})^{-1}+\boldsymbol A_\star&\boldsymbol E_\star\\
\boldsymbol E_\star^\top&\boldsymbol P_{r,\star}^{-1}+\boldsymbol D_\star
\end{bmatrix},\\
\veta_{b,\star}
&=\begin{bmatrix}
(\vSigma_{b,0}^{\mathrm{train}})^{-1}\vm_{b,0}^{\mathrm{train}}+\bm h_{0\star}\\
\bm h_{r\star}
\end{bmatrix}.
\end{aligned}
\label{eq:new-subject-coefficient-update}
\end{equation}
Its joint covariance is $\vLambda_{b,\star}^{-1}$ and its mean is $\vLambda_{b,\star}^{-1}\veta_{b,\star}$. To marginalize the new local coefficients, define
\[
\begin{aligned}
\boldsymbol C_\star&=(\boldsymbol P_{r,\star}^{-1}+\boldsymbol D_\star)^{-1},\\
\boldsymbol\Delta_\star&=\boldsymbol A_\star-\boldsymbol E_\star\boldsymbol C_\star\boldsymbol E_\star^\top,\\
\boldsymbol\zeta_\star&=\bm h_{0\star}-\boldsymbol E_\star\boldsymbol C_\star\bm h_{r\star}.
\end{aligned}
\]
The adapted shared marginal $q_{b,\star}(\vb_0)=\Normal(\vm_{\star,0},\vSigma_{\star,0})$ then satisfies
\begin{equation}
\begin{aligned}
\vSigma_{\star,0}^{-1}&=(\vSigma_{b,0}^{\mathrm{train}})^{-1}+\boldsymbol\Delta_\star,\\
\vm_{\star,0}&=\vSigma_{\star,0}\left[(\vSigma_{b,0}^{\mathrm{train}})^{-1}\vm_{b,0}^{\mathrm{train}}+\boldsymbol\zeta_\star\right].
\end{aligned}
\label{eq:new-subject-commonmean-update}
\end{equation}
The first contribution to the natural mean retains the learned population field; the second incorporates the new subject after accounting for local variation. If the new subject's residual statistics vanish, the shared marginal recovers $q_b^{\mathrm{train}}(\vb_0)$.

Profiling $\mathcal J_\star$ uses Eq.~\eqref{eq:profiled-objective} with the new-subject prior covariance $\boldsymbol P_\star=\blkdiag(\vSigma_{b,0}^{\mathrm{train}},\boldsymbol P_{r,\star})$ and the precision and natural mean in Eq.~\eqref{eq:new-subject-coefficient-update}, together with the additional term
\[
\frac12(\vm_{b,0}^{\mathrm{train}})^\top(\vSigma_{b,0}^{\mathrm{train}})^{-1}\vm_{b,0}^{\mathrm{train}}.
\]
This term comes from the nonzero prior mean and is constant during adaptation because the training marginal is fixed.

ODE predictions use 64 joint coefficient and initial-state draws as in Appendix~\ref{app:seen-subject-extrapolation}. Scores compare the Gaussian marginal moments of these trajectories with the latent truth; future observations and truth are used only for evaluation.

\endgroup

\subsection{Runtime and computational complexity}
\label{app:runtime-complexity}
For independent local coefficient priors, let $G=\sum_i(K_i+1)$, $p_0=Dl_0$, and $p_r=Dl_r$. These are per-pass costs; total fitting time depends on iterations, backtracking, and hyperparameter updates. For padded batches, $G$ counts stored nodes including padding.

\textbf{Gaussian calculations.} Smoothing and the required trajectory covariance blocks cost $O(GD^3)$ arithmetic and $O(GD^2)$ storage. The coefficient solve in Eq.~\eqref{eq:arrowhead} costs
\[
O\left(p_0^3+Mp_r^3+Mp_0p_r(p_0+p_r)\right),
\]
with $O(p_0^2+M[p_r^2+p_0p_r])$ storage for the shared marginal and local conditionals. This representation retains shared--local dependence without storing all cross-subject covariances.

\textbf{Feature statistics.} Analytic squared-exponential moments involve $l_0^2+l_0l_r+l_r^2$ feature pairs per node. With shared inducing grids, evaluation, aggregation, whitening, and output-block expansion have arithmetic cost bounded by
\[
O\left(G\bigl[D^3+D^2(l_0+l_r)^2\bigr]
+M\bigl[(l_0+l_r)^3+D^2(l_0+l_r)^2\bigr]\right).
\]
Kernel factorization adds $O(l_0^3+l_r^3)$ when kernel parameters change. Batched feature-product buffers require $O(G(l_0+l_r)^2)$ storage; peak memory also includes state-dependent intermediates, automatic differentiation, and compiler workspaces. At fixed $D,l_0,l_r$, these inference costs are linear in subject count and total grid size.

\textbf{Prediction and adaptation.} At fixed dimension and ranks, ODE rollout work scales with posterior samples, subjects, and integration steps. New-subject adaptation applies the above inference calculations to one subject, using the fitted shared marginal as its prior.

\textbf{Runtime.} Appendix~\ref{app:runtime-comparison} reports measured fitting and held-out adaptation times on Lotka--Volterra, together with the benchmark dimensions, optimization settings, hardware and compilation policy.

\begingroup
\section{Theoretical properties}
\label{app:theory}
We establish the following properties:
\begin{itemize}
    \item \textbf{Conditional uncertainty decomposition.} With finite second moments, conditioning on shared coefficients separates shared-mediated and conditional subject uncertainty in fields and ODE trajectories, retaining shared--local dependence (Section~\ref{app:theory-cov}).
    \item \textbf{Information pooling across subjects.} With priors, features, penalty weights, and existing trajectory factors fixed, adding a subject cannot increase the shared coefficient covariance (Section~\ref{app:theory-info}).
\end{itemize}

\color{black}

\color{black}
\subsection{Joint field uncertainty and conditional decomposition}
\label{app:theory-cov}
Interpreting subject-specific dynamics requires distinguishing the local field's effect on trajectories from its contribution to posterior uncertainty. Shared and local fields are generally dependent under the fitted posterior, so uncertainty attribution must preserve this dependence. Here, we derive the marginal field covariances and conditional trajectory-variance decomposition used in Figure~\ref{fig:local-dynamics-contributions}.

Under the joint Gaussian posterior in Eq.~\eqref{eq:coefficient-factorization}, write
\[
\vb_i=\bm\beta_i-\boldsymbol G_i\vb_0+\bm\epsilon_i,
\qquad \bm\epsilon_i\sim\Normal(\mathbf0,\boldsymbol C_i),
\]
where $\bm\epsilon_i$ is independent of $\vb_0$. The local marginal covariance is therefore $\vSigma_{b,i}=\boldsymbol C_i+\boldsymbol G_i\vSigma_{b,0}\boldsymbol G_i^\top$, with cross-covariance $\vSigma_{b,0i}=-\vSigma_{b,0}\boldsymbol G_i^\top$, as in Eq.~\eqref{eq:joint-covariances}. At fixed $\vx$, the marginal field covariances are
\begin{equation}
\begin{aligned}
\Cov[\vf_0(\vx)]&=\boldsymbol{\Phi}_0\vSigma_{b,0}\boldsymbol{\Phi}_0^\top,\\
\Cov[\vg_i(\vx)]&=\boldsymbol{\Phi}_r\vSigma_{b,i}\boldsymbol{\Phi}_r^\top.
\end{aligned}
\label{eq:marginal-field-variances}
\end{equation}
The feature matrices include the output amplitudes defined in Appendix~\ref{app:field-amplitudes}. These marginal covariances give the Gaussian field intervals in Figure~\ref{fig:local-dynamics-contributions}(c). Evaluating them along a deterministic reference path holds the inputs fixed and does not propagate state uncertainty.

\paragraph{Nonlinear ODE variance.}
Let $\boldsymbol X_i(t)$ denote an ODE trajectory obtained from a draw of the joint coefficient posterior and a specified origin-state distribution. Under Eq.~\eqref{eq:variational}, a sampled origin state is independent of the coefficients. A deterministic origin is also permitted. Whenever second moments are finite, total covariance gives the exact identity
\begin{equation}
\Cov[\boldsymbol X_i(t)]=
\Cov_{\vb_0}\!\left(\E[\boldsymbol X_i(t)\mid\vb_0]\right)+
\E_{\vb_0}\!\left(\Cov[\boldsymbol X_i(t)\mid\vb_0]\right).
\label{eq:trajectory-total-variance}
\end{equation}
The conditional expectation and covariance integrate $q_b(\vb_i\mid\vb_0)$ and the chosen initial-state distribution. For a coordinate $X_{i,d}(t)=\boldsymbol e_d^\top\boldsymbol X_i(t)$, define $v_{i,d}^{X,\mathrm{sh}}(t)=\Var_{\vb_0}(\E[X_{i,d}(t)\mid\vb_0])$ and $v_{i,d}^{X,\mathrm{loc}}(t)=\E_{\vb_0}[\Var(X_{i,d}(t)\mid\vb_0)]$. For positive total variance, the exact local fraction is
\begin{equation}
\rho^{X,\mathrm{loc}}_{i,d}(t)=
\frac{v_{i,d}^{X,\mathrm{loc}}(t)}{v_{i,d}^{X,\mathrm{sh}}(t)+v_{i,d}^{X,\mathrm{loc}}(t)}.
\label{eq:trajectory_local_fraction}
\end{equation}
When the origin state is sampled, the conditional subject contribution includes its uncertainty. Figure~\ref{fig:local-dynamics-contributions} instead fixes each initial state at its posterior mean, isolating coefficient uncertainty. Its trajectory bands and nonlinear fractions use the same ODE draws. These quantities concern the ODE pushforward; the auxiliary Gaussian trajectory factor has its own covariance.

\paragraph{Nested conditional Monte Carlo.}
At fixed $i,d,t$, draw $R\geq2$ independent shared coefficient vectors. For each shared draw $r$, obtain $L\geq2$ independent conditional local draws and integrate the ODE, giving coordinate values $X_{r\ell}$. Let $\bar X_r=L^{-1}\sum_\ell X_{r\ell}$, $\bar X=R^{-1}\sum_r\bar X_r$, and $s_r^2=(L-1)^{-1}\sum_\ell(X_{r\ell}-\bar X_r)^2$. Unbiased estimates of the two variance components are
\begin{equation}
\widehat v_{i,d}^{X,\mathrm{loc}}(t)=\frac1R\sum_r s_r^2,
\qquad
\widehat v_{i,d}^{X,\mathrm{sh}}(t)=
\frac1{R-1}\sum_r(\bar X_r-\bar X)^2-\frac{\widehat v_{i,d}^{X,\mathrm{loc}}(t)}{L}.
\label{eq:nested-variance-estimators}
\end{equation}
Indeed, $\Var(\bar X_r)=v_{i,d}^{X,\mathrm{sh}}(t)+v_{i,d}^{X,\mathrm{loc}}(t)/L$, so the subtraction removes inner-sampling noise from the shared component. Ratios of these estimates are not themselves unbiased. The corrected shared estimate can be slightly negative at finite sample size, giving estimated fractions outside $[0,1]$; we retain these fluctuations without clipping. Independent draws from the marginal $q_b(\vb_i)$ would remove the required shared--local dependence.

\color{black}

\subsection{Conditional information accumulation}
\label{app:theory-info}
We examine how pooling subjects informs the shared dynamics after accounting for local variation. Here, we quantify each subject's information contribution and its effect on shared uncertainty, keeping inference settings and existing trajectory factors fixed.

\begin{proposition}[Shared covariance under fixed-statistic subject addition]
\label{prop:monotone}
Keep the coefficient priors, features, penalty weights, and existing subjects' trajectory factors fixed. For subject $i$, define
\[
\boldsymbol\Delta_i=\boldsymbol A_i-\boldsymbol E_i(\boldsymbol P_r^{-1}+\boldsymbol D_i)^{-1}\boldsymbol E_i^\top.
\]
Then $\boldsymbol\Delta_i\succeq0$ and the marginal shared precision is $\vSigma_{b,0}^{-1}=\boldsymbol P_0^{-1}+\sum_i\boldsymbol\Delta_i$. Adding a subject with finite expected statistics cannot increase the variance of any linear combination of the shared coefficients. Explicitly, if $\vSigma_{b,0}^{\mathrm{old}}$ and $\vSigma_{b,0}^{\mathrm{new}}$ denote the shared covariances before and after adding the subject, then
\[
\bm a^\top\vSigma_{b,0}^{\mathrm{new}}\bm a
\leq \bm a^\top\vSigma_{b,0}^{\mathrm{old}}\bm a
\qquad\text{for every coefficient vector }\bm a.
\]
\end{proposition}
\begin{proof}
The expected joint design Gram matrix
$\left[\begin{smallmatrix}\boldsymbol A_i&\boldsymbol E_i\\\boldsymbol E_i^\top&\boldsymbol D_i\end{smallmatrix}\right]$
is positive semidefinite. Adding $\boldsymbol P_r^{-1}\succ0$ to the local diagonal preserves that property. For any $\bm z$, minimize its quadratic form over the local vector; the minimum is $\bm z^\top\boldsymbol\Delta_i\bm z\geq0$, proving $\boldsymbol\Delta_i\succeq0$. Equation~\eqref{eq:arrowhead} gives the shared precision. Adding $\boldsymbol\Delta_i\succeq0$ to a positive-definite precision reverses order under inversion, giving the covariance conclusion.
\end{proof}
Refitting trajectories, features, noise, or amplitudes changes the fixed quantities in this statement. It does not imply monotonicity under complete retraining or improved predictive calibration.

\paragraph{Common-mean regression interpretation.}
Under the assumptions of Proposition~\ref{prop:monotone} and the independent zero-mean coefficient priors in Eq.~\eqref{eq:weightprior}, let $\boldsymbol C_i=(\boldsymbol P_r^{-1}+\boldsymbol D_i)^{-1}$ and define the subject's marginalized natural-mean contribution
\[
\boldsymbol\zeta_i=\bm h_{0i}-\boldsymbol E_i\boldsymbol C_i\bm h_{ri}.
\]
Equation~\eqref{eq:arrowhead} then gives the shared marginal of the coefficient coordinate optimum:
\begin{equation}
\begin{aligned}
q_b^\star(\vb_0)&=\Normal(\vm_{b,0},\vSigma_{b,0}),\\
\vSigma_{b,0}^{-1}&=\boldsymbol P_0^{-1}+\sum_i\boldsymbol\Delta_i,\\
\vm_{b,0}&=\vSigma_{b,0}\sum_i\boldsymbol\zeta_i.
\end{aligned}
\label{eq:commonmean-natural-update}
\end{equation}
Thus both precision and natural mean aggregate subject-specific contributions after marginalizing local variation.

\begin{proposition}[Information retained after local adjustment]
\label{prop:commonmean-field}
Under the fixed-statistic conditions of Proposition~\ref{prop:monotone}, for every $\vz\in\R^{p_0}$,
\begin{equation}
\vz^\top\boldsymbol\Delta_i\vz
=\min_{\vu\in\R^{p_r}}\Biggl\{\vu^\top\boldsymbol P_r^{-1}\vu
+\sum_k\kappa\omega_{i,k}\E_{q_i}\left\|
\boldsymbol\Phi_0(\vx_{i,k})\vz-\boldsymbol\Phi_r(\vx_{i,k})\vu
\right\|^2\Biggr\}.
\label{eq:retained-shared-information}
\end{equation}
In particular, $\mathbf0\preceq\boldsymbol\Delta_i\preceq\boldsymbol A_i$.

For the deterministic-input specialization of the coefficient subproblem, let $\boldsymbol\Phi_{a,i}\in\R^{D(K_i+1)\times p_a}$ stack $\boldsymbol\Phi_a(\vx_{i,k})$ over $k$, for $a\in\{0,r\}$, and let $\bm v_i\in\R^{D(K_i+1)}$ stack the derivative responses. Assume $\kappa\omega_{i,k}>0$ and define
\[
\boldsymbol R_i=\blkdiag_k\bigl((\kappa\omega_{i,k})^{-1}\Id_D\bigr),\qquad
\boldsymbol V_i=\boldsymbol R_i+\boldsymbol\Phi_{r,i}\boldsymbol P_r\boldsymbol\Phi_{r,i}^\top.
\]
Then
\begin{equation}
\boldsymbol\Delta_i=\boldsymbol\Phi_{0,i}^\top\boldsymbol V_i^{-1}\boldsymbol\Phi_{0,i},\qquad
\boldsymbol\zeta_i=\boldsymbol\Phi_{0,i}^\top\boldsymbol V_i^{-1}\bm v_i.
\label{eq:commonmean-fixed-inputs}
\end{equation}
Hence the shared marginal is exactly the Gaussian posterior associated with the common-mean regression
\[
\bm v_i\mid\vb_0\sim\Normal(\boldsymbol\Phi_{0,i}\vb_0,\boldsymbol V_i),
\]
with independent subjects conditional on $\vb_0$.
\end{proposition}

\begin{proof}
Expanding the expression minimized in Eq.~\eqref{eq:retained-shared-information} gives
\[
\vz^\top\boldsymbol A_i\vz-2\vz^\top\boldsymbol E_i\vu
+\vu^\top(\boldsymbol P_r^{-1}+\boldsymbol D_i)\vu.
\]
Its unique minimizer is $\vu=\boldsymbol C_i\boldsymbol E_i^\top\vz$, and substitution yields $\vz^\top\boldsymbol\Delta_i\vz$. The minimized expression is nonnegative, while choosing $\vu=\mathbf0$ gives the upper bound $\boldsymbol\Delta_i\preceq\boldsymbol A_i$.

For deterministic inputs, the expected blocks reduce to weighted design products:
\[
\begin{aligned}
\boldsymbol A_i&=\boldsymbol\Phi_{0,i}^\top\boldsymbol R_i^{-1}\boldsymbol\Phi_{0,i},&
\boldsymbol E_i&=\boldsymbol\Phi_{0,i}^\top\boldsymbol R_i^{-1}\boldsymbol\Phi_{r,i},\\
\boldsymbol D_i&=\boldsymbol\Phi_{r,i}^\top\boldsymbol R_i^{-1}\boldsymbol\Phi_{r,i},&
\bm h_{ai}&=\boldsymbol\Phi_{a,i}^\top\boldsymbol R_i^{-1}\bm v_i.
\end{aligned}
\]
The Woodbury identity gives
\[
\boldsymbol V_i^{-1}=\boldsymbol R_i^{-1}
-\boldsymbol R_i^{-1}\boldsymbol\Phi_{r,i}
\left(\boldsymbol P_r^{-1}+\boldsymbol\Phi_{r,i}^\top\boldsymbol R_i^{-1}\boldsymbol\Phi_{r,i}\right)^{-1}
\boldsymbol\Phi_{r,i}^\top\boldsymbol R_i^{-1}.
\]
Multiplication by the shared designs, or by the shared design and response, proves Eq.~\eqref{eq:commonmean-fixed-inputs}. Equivalently, integrating the local coefficients out of the linear-Gaussian regression adds their field covariance to $\boldsymbol R_i$.
\end{proof}

Equation~\eqref{eq:retained-shared-information} describes the information about a shared-field direction that remains after allowing a regularized local adjustment. The deterministic-input specialization is the finite-feature common-mean GP structure of MAGMA~\citep{leroy2022magma}, applied to derivative responses at state inputs. The regression covariance $\boldsymbol R_i$ encodes the residual-penalty weights. The smoother's derivative covariance contributes to the expected squared residual as in Appendix~\ref{app:colloc-derivation}, without being added to $\boldsymbol R_i$.

In the general update, the blocks are formed from expectations under $q_i$; evaluating features only at posterior-mean states does not recover those blocks. The fixed-input regression is an algebraic specialization of the coefficient subproblem, rather than a change to the fitted trajectory factors. It does not imply a common-mean GP distribution for the nonlinear ODE trajectories.

\color{black}

\endgroup

\section{Alternative scalable vector field approximations}
\label{app:altscalable}
The Nystr\"om approximation in Section~\ref{sec:scalable} determines the feature maps used to represent the fields. The variational construction also applies to other finite linear-Gaussian representations. We give a general finite-feature construction and specialize it to random Fourier features. Other alternatives include Hilbert basis functions~\citep{solin2020hilbert}.

\subsection{General finite-feature construction}
\label{app:generic-features}
Let $\boldsymbol\phi_{a,d}:\R^D\to\R^{1\times l_a}$ be a feature row for field $a\in\{0,r\}$ and output $d$, including its amplitude as in Appendix~\ref{app:field-amplitudes}. Consider
\begin{equation}
\begin{aligned}
f_{0,d}(\vx)&=\boldsymbol\phi_{0,d}(\vx)\vb_{0,d},\qquad
g_{i,d}(\vx)=\boldsymbol\phi_{r,d}(\vx)\vb_{i,d},\\
\vb_{0,d}&\sim\Normal(\mathbf0,\boldsymbol\Omega_0^{-1}),\qquad
\vb_{i,d}\sim\Normal(\mathbf0,\boldsymbol\Omega_r^{-1}),
\end{aligned}
\label{eq:generic-feature-fields}
\end{equation}
with positive-definite prior precisions $\boldsymbol\Omega_0,\boldsymbol\Omega_r$ and independent coefficient blocks under the prior. The feature rows and prior covariances jointly determine the approximate kernels. The whitened Nystr\"om model uses $\boldsymbol\Omega_0=\Id_{l_0}$ and $\boldsymbol\Omega_r=\lambda_r\Id_{l_r}$. Stacking outputs and subjects gives the same global design $\boldsymbol F_i$ and block-diagonal prior covariance $\boldsymbol P$ as in the main text.

Under the joint Gaussian coefficient factor $q_b$, the mean remains $\bm\mu_i^f(\vx)=\boldsymbol F_i(\vx)\vm_b$. Writing $\boldsymbol\Phi_0,\boldsymbol\Phi_r$ for the output-wise feature designs, the field covariance is
\begin{equation}
\bm V_i^f(\vx)=\boldsymbol\Phi_0(\vx)\vSigma_{b,0}\boldsymbol\Phi_0(\vx)^\top
+\boldsymbol\Phi_r(\vx)\vSigma_{b,i}\boldsymbol\Phi_r(\vx)^\top
+\boldsymbol\Phi_0(\vx)\vSigma_{b,0i}\boldsymbol\Phi_r(\vx)^\top
+\boldsymbol\Phi_r(\vx)\vSigma_{b,0i}^\top\boldsymbol\Phi_0(\vx)^\top.
\label{eq:generic-feature-covariance}
\end{equation}
Thus changing the features preserves the shared--local posterior cross terms. For differentiable features, the posterior-mean Jacobian is also explicit: $\partial\bm\mu_i^f/\partial x_j=(\partial\boldsymbol F_i/\partial x_j)\vm_b$.

The coefficient update in Eq.~\eqref{eq:coefficient-update} and the trajectory loss in Eq.~\eqref{eq:expected-residual} retain the same form, with moments evaluated for the chosen feature maps.

\subsection{Random Fourier features}
\label{app:rff-features}
Random Fourier features approximate stationary kernels through sampled spectral frequencies~\citep{rahimi2008rff}. For the ARD RBF kernels used here, draw independent $\boldsymbol\epsilon_{a,j}\sim\Normal(\mathbf0,\Id_D)$ and $\beta_{a,j}\sim\mathrm{Unif}(0,2\pi)$ for $j=1,\ldots,l_a$, and set
\begin{equation}
\phi_{a,d,j}(\vx)=c_{a,d}\cos(\boldsymbol\omega_{a,j}^\top\vx+\beta_{a,j}),\qquad
c_{a,d}=\sigma_{a,d}\sqrt{\frac{2}{l_a}},\qquad
\omega_{a,j,h}=\frac{\epsilon_{a,j,h}}{\ell_{a,h}}.
\label{eq:rff-feature-map}
\end{equation}
We use the same coefficient priors as for whitened Nystr\"om features. The base draws $\boldsymbol\epsilon_{a,j},\beta_{a,j}$ remain fixed during fitting; learned lengthscales change the frequencies through Eq.~\eqref{eq:rff-feature-map}, and learned output amplitudes change $c_{a,d}$.

For a Gaussian state $\vx\sim\Normal(\vm_x,\boldsymbol P_{xx})$, define
\begin{equation}
C(\boldsymbol\nu,\gamma)
=\exp\!\left(-\tfrac12\boldsymbol\nu^\top\boldsymbol P_{xx}\boldsymbol\nu\right)
\cos(\boldsymbol\nu^\top\vm_x+\gamma).
\label{eq:gaussian-cosine-moment}
\end{equation}
The Gaussian characteristic function gives the first feature and gradient moments:
\begin{equation}
\begin{aligned}
\E[\phi_{a,d,j}(\vx)]&=c_{a,d}C(\boldsymbol\omega_{a,j},\beta_{a,j}),\\
\E[\nabla\phi_{a,d,j}(\vx)]
&=-c_{a,d}\boldsymbol\omega_{a,j}
\exp\!\left(-\tfrac12\boldsymbol\omega_{a,j}^\top\boldsymbol P_{xx}\boldsymbol\omega_{a,j}\right)
\sin(\boldsymbol\omega_{a,j}^\top\vm_x+\beta_{a,j}).
\end{aligned}
\label{eq:rff-first-gradient-moments}
\end{equation}
For any $a,c\in\{0,r\}$, the identity $\cos u\cos v=\tfrac12[\cos(u-v)+\cos(u+v)]$ gives
\begin{equation}
\E[\phi_{a,d,j}(\vx)\phi_{c,e,h}(\vx)]
=\frac{c_{a,d}c_{c,e}}2\bigl[C(\boldsymbol\omega_{a,j}-\boldsymbol\omega_{c,h},\beta_{a,j}-\beta_{c,h})
+C(\boldsymbol\omega_{a,j}+\boldsymbol\omega_{c,h},\beta_{a,j}+\beta_{c,h})\bigr].
\label{eq:rff-pair-moments}
\end{equation}
Together with Eq.~\eqref{eq:velocity-feature-moment}, these identities evaluate the uncertain-input statistics analytically, including state--derivative covariance and shared--local feature products. Feature ranks $l_0,l_r$ control the representation size, while the sampled spectral basis is not positioned using the observed states as the Nystr\"om inducing locations are.

\end{document}